\documentclass{article}
\usepackage[utf8]{inputenc}
\usepackage{a4wide}
\usepackage{xcolor}
\usepackage{amsmath}
\usepackage{lscape}
\usepackage[font=small,labelfont=bf]{caption}
\usepackage{graphicx}
\usepackage{subcaption}
\usepackage{comment}
\usepackage{tabularx}
\usepackage{multirow}
\usepackage{algorithmic}
\usepackage[ruled,vlined]{algorithm2e}
\usepackage{amsmath,amssymb,amsfonts,amsthm} 
{
\newtheorem{theorem}{Theorem}
\newtheorem{lemma}{Lemma}
\newtheorem {problem}{Problem}
\newtheorem{assumption}{Assumption}
\theoremstyle{remark}
\newtheorem {remark}{Remark}

}
\usepackage[utf8]{inputenc}
\usepackage{amssymb}
\usepackage{multicol}
\usepackage{geometry}
\usepackage{hyperref}
\usepackage{authblk}
\usepackage{ulem}
\usepackage{moreverb}
\usepackage{amsmath,bm}
\usepackage{graphicx}
\usepackage{siunitx}
\usepackage{subcaption}

\renewenvironment{abstract}
  {\small\noindent\textbf{Abstract.}\normalsize\ignorespaces}
  {\par\noindent\ignorespacesafterend}

\makeatletter
\renewcommand{\maketitle}{%
    \noindent\fcolorbox{red}{white}{
        \parbox{\textwidth}{%
            \color{red}\copyright2025 Elsevier. Accepted for \textit{Journal of The Franklin Institute}. Personal use of this material is permitted. Permission from Elsevier must be obtained for all other uses, in any current or future media, including reprinting/republishing this material for advertising or promotional purposes, creating new collective works, for resale or redistribution to servers or lists, or reuse of any copyrighted component of this work in other works.%
        }%
    }
    \vspace{1em} 
    \begin{center}%
        {\LARGE \@title \par}
        \vspace{0.5em}
        {\large \@author \par}
        \vspace{0.5em}
        {\large \@date}
    \end{center}%
    \vspace{1em} 
}
\makeatother

\begin{document}

\title{Flatness-based Finite-Horizon Multi-UAV Formation Trajectory Planning and Directionally Aware Collision Avoidance Tracking} 

\author[1]{Hossein B. Jond}
\author[2]{Logan Beaver}
\author[1]{Martin Jiroušek}
\author[3]{Naiemeh Ahmadlou}
\author[1,4]{Veli Bakırcıoğlu}
\author[1]{Martin Saska}
\affil[1]{Department of Cybernetics, Czech Technical University in Prague, Prague, Czechia}
\affil[2]{Mechanical and Aerospace Engineering, Old Dominion University, Norfolk, VA, USA}
\affil[3]{Mechanical Engineering Faculty, Sahand University of Technology, New Sahand Town, Tabriz, Iran}
\affil[4]{Aksaray University, Aksaray, Türkiye}

\date{}

\maketitle

\begin{abstract}
Optimal collision-free formation control of the unmanned aerial vehicle (UAV) is a challenge. The state-of-the-art optimal control approaches often rely on numerical methods sensitive to initial guesses. This paper presents an innovative collision-free finite-time formation control scheme for multiple UAVs leveraging the differential flatness of the UAV dynamics, eliminating the need for numerical methods. We formulate a finite-time optimal control problem to plan a formation trajectory for feasible initial states. This optimal control problem in formation trajectory planning involves a collective performance index to meet the formation requirements to achieve relative positions and velocity consensus. It is solved by applying Pontryagin’s principle. Subsequently, a collision-constrained regulating problem is addressed to ensure collision-free tracking of the planned formation trajectory. The tracking problem incorporates a directionally aware collision avoidance strategy that prioritizes avoiding UAVs in the forward path and relative approach. It assigns lower priority to those on the sides with an oblique relative approach, disregarding UAVs behind and not in the relative approach. The high-fidelity simulation results validate the effectiveness of the proposed control scheme. 
 
\end{abstract}

\section{Introduction}\label{sec1}

Unmanned aerial vehicle (UAV) formation control has gained significant attention due to its broad applications, such as exploring hazardous environments~\cite{9492802}, transporting loads~\cite{app13020822}, conducting critical inspections~\cite{9213967}, and surveying historical structures~\cite{Saska2020}. A fundamental aspect of UAV formation control involves trajectory planning and trajectory tracking.  In formation trajectory planning, UAVs plan optimal or sub-optimal collision-free trajectories between their initial and terminal states~\cite{ALQUDSI2023104532,9359893}. Formation trajectory tracking ensures that UAVs converge to these pre-planned formation trajectories~\cite{4392487,HUANG20204034,10433443}. Formation trajectory planning and tracking present significant challenges due to the nonlinear dynamics of UAVs, inter-UAV collision avoidance, and dynamic environmental constraints. For quadrotor UAVs, their underactuated nature further complicates control and trajectory planning. These challenges become more pronounced in time-sensitive operations, where UAV groups must rapidly form formations, execute tasks, and ensure collision-free motion within finite time horizons~\cite{10171158}. 

A quadrotor UAV, with six degrees of freedom and only four control inputs, is inherently underactuated. Differential flatness resolves underactuation by representing states and inputs as functions of flat outputs (e.g., position and yaw) and their derivatives. This simplifies trajectory planning and control~\cite{Mellinger,9121690,10845852}. The application of differential flatness theory to quadrotor formation control has garnered research attention. The study in~\cite{AI201920} has tackled the leader-follower formation control for quadrotors using differential flatness theory and proposed an observer-based controller to reconstruct the leader's states and reject disturbances. In~\cite{10590881}, a leader-follower formation control strategy is presented for a team of quadrotors, integrating trajectory planning (using the flatness-based trajectory generation algorithm in~\cite{Mellinger}), distributed consensus, and trajectory tracking. In~\cite{9483176}, a leader-follower control strategy is developed for UAV formation in GPS-denied environments, incorporating feedback linearization, differential flatness, and linear-quadratic regulator (LQR) techniques. In~\cite{8429106}, a virtual structure-based approach is introduced to generate formation trajectories in a quadrotor swarm, where quadrotors utilize differential flatness for control. The virtual structure is integrated with multiple layered potential fields to ensure collision avoidance.

Networked UAVs can exchange information with neighbors, enabling consensus algorithms to facilitate collaboration. These algorithms can be integrated with leader-follower~\cite{8214966}, behavioral, or virtual structure methods for trajectory tracking or employed independently for trajectory planning and tracking in team formation control~\cite{WU2020106332}. Finite-horizon formation control unifies consensus-based coordination with an optimal control framework by optimizing a collective performance index of UAV states and control inputs within a constrained time frame~\cite{10171158}. Some research applies discrete-time optimal control with fixed-time performance guarantees, such as the method for the seeker stabilized platform (SSP) in~\cite{10731581}, which uses iterative updates and dynamic programming. In~\cite{10339271}, a framework is proposed that uses fixed-time performance functions and adaptive neural approximations to efficiently handle unknown nonlinearities and directions, avoiding the reliance on Nussbaum-type functions commonly used in existing strategies for prescribed performance control of discrete-time non-affine systems. In~\cite{10747736}, a discrete-time fuzzy-neural intelligent control approach with fixed-time prescribed performance is introduced, addressing the limitations of existing methods that struggle with time-varying sampling intervals. However, these approaches can be computationally intensive due to time-stepping algorithms~\cite{10731581}, adaptive neural approximations~\cite{10339271}, and fuzzy-neural design~\cite{10747736}. In contrast, our method uses continuous-time finite-horizon optimal control and leverages the differential flatness of UAV dynamics to enable efficient real-time coordination via linear-quadratic optimization.

Existing finite-horizon methods exhibit limitations due to their dependence on numerical optimization, leader-follower hierarchies, or pre-planned trajectory interpolations, each of which introduces practical constraints. Numerical methods for trajectory optimization require well-chosen initial guesses to converge to feasible solutions~\cite{8187728}. In the leader-follower approach~\cite{8214966}, the leader serves as a single point of failure, the behavioral approach lacks a formalized control framework, and the virtual structure approach suffers from limited adaptability~\cite{AI201920,10590881,8429106}. Trajectory interpolations using Bézier curves, high-order piecewise polynomials, or recursive constructions~\cite{s22051855, ALQUDSI2023104532, 9359893} require predefined initial and terminal states, and often waypoints. Existing collision-avoidance strategies often consider all neighboring UAVs equally without prioritizing evasive actions based on their relative motion~\cite{HUANG20204034,LEE201765}. This can lead to inefficient maneuvers and unnecessary trajectory deviations. Although infinite-horizon approaches provide stability guarantees~\cite{SHI20187626}, they lack responsiveness to time-sensitive formation control tasks.

This paper addresses the above challenges by leveraging the differential flatness property of quadrotor UAV dynamics to develop a finite-horizon formation trajectory planning and tracking control scheme based solely on the initial states of the UAVs. Specifically, the proposed approach involves the following features:
\begin{enumerate}
    \item Consensus-Based Formation Control: Unlike most prior works that rely on a virtual or actual leader as an external reference for formation control~\cite{tra_gen2023,AI201920,8854340,SHI20187626}, the proposed approach is entirely consensus-based. This improves robustness and scalability by eliminating potential single-point failures associated with leader-based strategies. Formation direction can be adjusted by incorporating the coordinates of the target area as a deviation term in the performance index, allowing the UAV group to maintain formation while approaching the target without relying on a designated leader.  
    
    \item Optimal Trajectory Planning Using Only Initial State Knowledge: The proposed framework optimally plans UAV trajectories using only their initial states by applying the necessary and sufficient conditions from the Pontryagin principle. This significantly reduces computational complexity, as trajectory planning and generation are performed only once. 
    \item Finite Horizon Real-Time Adaptability: Unlike infinite-horizon formulations~\cite{SHI20187626}, our approach employs a finite-horizon strategy, optimizing UAV motion within a limited time window to ensure timely responsiveness, adaptability to dynamic tasks and environments, and guaranteed finite-time convergence. 
    \item Directionally Aware Collision Avoidance: The proposed method prioritizes evasive actions based on relative motion, emphasizing avoidance maneuvers on the forward path and along the approach trajectory, thus reducing unnecessary deviations and ensuring efficient collision-free multi-UAV operations.
\end{enumerate}

The main contribution of this paper is a unified framework for UAV formation control that integrates the features discussed above. It ensures collision-free trajectory planning and tracking while maintaining scalability, computational efficiency, and real-time adaptability. Furthermore, the approach leverages the linear-quadratic paradigm from optimal control, which is known for its analytic tractability and computational efficiency~\cite{Engwerda.oca.2290}. 

The rest of this paper is organized as follows. Section~\ref{sec:pre} provides preliminaries. The collision-free finite-time formation control is defined in Section~\ref{sec:def}. Section~\ref{sec:main} recasts the problem definition as optimal control problems, where their solutions are investigated. The directionally aware collision avoidance strategy is introduced in Section~\ref{sec:collision}. Section~\ref{sec:sim} presents illustrative examples and simulation results for the (re)formation of a four-UAV and seven-UAV team. Section~\ref{sec:con} concludes the paper.

\section{Preliminaries}\label{sec:pre}

\subsection{graph theory}
A directed graph is a pair $\mathcal{G}(\mathcal{V},\mathcal{E})$ where $\mathcal{V}=\{1,\cdots,N\}$ is a finite set of vertices/nodes and $\mathcal{E}\subseteq \{(i,j):i,j \in \mathcal{V}\}$ is a set of directed edges/arcs. Each edge $ (i,j) \in \mathcal{E}$ represents a directional information flow and is assigned a positive scalar weight $\mu_{ij}>0$. A graph \( \mathcal{G} \) is connected if, for every pair of distinct vertices \( (i, j) \in \mathcal{V} \times \mathcal{V} \) with \( i \neq j \), there exists a path of undirected edges in \( \mathcal{E} \) that connects \( i \) to \( j \).

The graph Laplacian matrix is defined as  
\begin{equation*}  
    \mathbf{L} = \mathbf{D} \mathbf{W} \mathbf{D}^\top \in \mathbb{R}^{N},  
\end{equation*}  
where \( \mathbf{D} \in \mathbb{R}^{|\mathcal{V}| \times |\mathcal{E}|} \) is the incidence matrix of \( \mathcal{G} \), with its \( (i,j) \)th element equal to \( 1 \) if node \( i \) is the tail of edge \( j \), \( -1 \) if node \( i \) is the head, and \( 0 \) otherwise. Additionally, \( \mathbf{W} = \mathrm{diag}(\dots, \mu_{ij}, \dots) \) for all \( (i,j) \in \mathcal{E} \), where \( \mathbf{W} \in \mathbb{R}^{|\mathcal{E}|} \). Matrix $\mathbf{L}$ is symmetric ($\mathbf{L}=\mathbf{L}^\top$), positive semidefinite (PSD) ($\mathbf{L}\succeq \mathbf{0}$), and satisfies the sum-of-squares property
\begin{equation}\label{eq:sum-of-square}
    \sum_{(i,j)\in\mathcal{E}}\mu_{ij}\|\mathbf{x}_i-\mathbf{x}_j\|^2=\mathbf{x}^\top (\mathbf{L}\otimes\mathbf{I}_m) \mathbf{x},
\end{equation}
where $\mathbf{x}=[\mathbf{x}^\top_1,\cdots,\mathbf{x}_N^\top]^\top\in \mathit{\mathbb{R}}^{mN}$, $\mathbf{x}_i\in \mathit{\mathbb{R}}^{m}$,  $\mathbf{I}_m\in \mathit{\mathbb{R}}^{m}$ is the identity matrix of dimension $m$, $\otimes$ demonstrates the Kronecker product operator, and $\lVert .\rVert$ is the Euclidean norm. 

Let \( \mathbf{Y} = [y_{ij}] \in \mathbb{R}^{p \times q} \). Then, the Kronecker product of \( \mathbf{Y} \) and the identity matrix \( \mathbf{I}_m \) is given by  

\begin{equation*}  
\mathbf{Y} \otimes \mathbf{I}_m =  
\begin{bmatrix}  
y_{11} \mathbf{I}_m & \cdots & y_{1q} \mathbf{I}_m \\  
\vdots & \ddots & \vdots \\  
y_{p1} \mathbf{I}_m & \cdots & y_{pq} \mathbf{I}_m  
\end{bmatrix}  
\in \mathbb{R}^{mp \times mq}.  
\end{equation*}  

\subsection{UAV model}

The UAV model is a quadrotor consisting of a rigid cross frame with four rotors, as depicted in Fig.~\ref{fig:UAV}. The equations of motion governing the acceleration of the center of mass for a quadrotor UAV indexed by $i$ are given as follows~\cite{Mellinger}
\begin{align}
m\ddot{\mathbf{p}}_i(t) =& -mg\mathbf{z}_\mathcal{W} + u_{i1}(t) \mathbf{z}_\mathcal{B}, \label{eq:dynamics-linear} \\
\dot{\boldsymbol{\omega}}_{\mathcal{BW}} (t)=& \boldsymbol{\mathcal{I}}^{-1} \big\{-\boldsymbol{\omega}_{\mathcal{BW}}(t) \times \boldsymbol{\mathcal{I}}\boldsymbol{\omega}_{\mathcal{BW}}(t) + \begin{bmatrix}
    u_{i2}(t) & u_{i3}(t) & u_{i4}(t)
\end{bmatrix}^\top\big\}, \label{eq:dynamics-angular}
\end{align}
where $\mathbf{p}_i(t)=[x_i(t),y_i(t),z_i(t)]^\top\in\mathbb{R}^{3}$ is the position vector of the center of mass in the world frame $\mathcal{W}$, $\boldsymbol{\omega}_{\mathcal{BW}}(t)$ is the angular velocity of the body frame $\mathcal{B}$ relative to the world frame $\mathcal{W}$, $\boldsymbol{\mathcal{I}}$ is the moment of inertia matrix referenced to the center of mass along the body axes $\mathbf{x}_\mathcal{B}$ ,$\mathbf{y}_\mathcal{B}$, $\mathbf{z}_\mathcal{B}$, $m$ is the mass of the quadrotor, and $g$ is the gravitational acceleration. The control input to the quadrotor is $\mathbf{u}_i(t)\in\mathbb{R}^{4}$ where $u_{i1}(t)$ is the net body force and $u_{i2}(t)$, $u_{i3}(t)$, $u_{i4}(t)$ are the body moments which can be expressed according to the rotor speeds.

\begin{figure}
\centerline{\includegraphics[width=0.5\columnwidth]{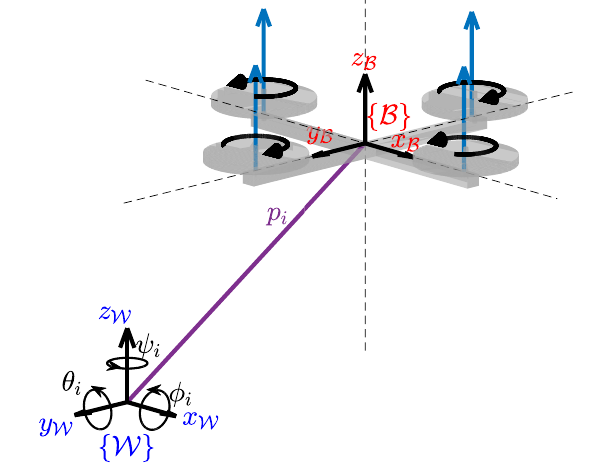}}
\caption{Quadrotor configuration. $\mathcal{W}$ and $\mathcal{B}$ denote the world frame and body frame, respectively. The roll, pitch, and yaw angles are denoted by $\phi_i$, $\theta_i$, and $\psi_i$, respectively.}
\label{fig:UAV}
\end{figure}

\section{Problem definition}\label{sec:def}

Consider a multi-UAV team consisting of \( N \) quadrotors, indexed by \( \mathcal{V} = \{1, \dots, N\} \). The dynamics of each UAV \( i \in \mathcal{V} \) is governed by (\ref{eq:dynamics-linear})–(\ref{eq:dynamics-angular}). Two time-invariant, connected, directed graphs are defined: the formation graph \( \mathcal{G}_f(\mathcal{V}, \mathcal{E}_f) \) and the communication or sensing graph \( \mathcal{G}_c(\mathcal{V}, \mathcal{E}_c) \). The formation graph encodes the desired spatial relationships between UAVs, whereas the communication graph represents all potential sensing and communication links needed for collision avoidance.

\begin{assumption}\label{assump:con}
(Connectivity) The formation graph is directed, with its underlying undirected graph being connected and containing at least one globally reachable node, i.e., a node \(i\) such that there is a directed path from \(i\) to node \(j\) for all \(j\in\mathcal{V}\) with \(j\neq i\)~\cite{LI20141626}.
\end{assumption}

Assumption~\ref{assump:con} implies the existence of a directed spanning tree with a root node \(i\).

\begin{assumption}\label{assump:complete}
    The communication graph is a complete directed graph to account for all inter-UAV collisions (thus, $\mathcal{G}_f\subseteq\mathcal{G}_c$), where every pair of UAVs is connected by directed edges in both directions. 
\end{assumption}

The formation control problem involves steering all UAVs from their initial state to a desired formation. Additionally, collisions between UAVs must be avoided. The formation shape is determined by the offset vectors $\mathbf{d}_{ij}\in\mathbb{R}^3$ for all UAV pairs $(i,j)\in\mathcal{E}_f$. Each UAV $i\in\mathcal{V}$ has an avoidance region radius $r_i$, meaning the minimum safe distance between two UAVs $i$ and $j$ is $r_i + r_j$. Additionally, let $R_i$ denote the radius of the collision reaction region of each UAV $i\in\mathcal{V}$ and $R_i>r_i$ (see Fig.~\ref{fig:radius}). 

Define the UAV team position vector $\mathbf{p}(t)=[\mathbf{p}_1^\top(t),\cdots,\mathbf{p}_N^\top(t)]^\top\in\mathbb{R}^{3N}$. The following regions are defined accordingly
\begin{enumerate}
    \item safety region: $\Psi(\mathbf{p})\triangleq\{\mathbf{p}(t)|\mathbf{p}_i(t)\in\mathbb{R}^3, \|\mathbf{p}_i(t)-\mathbf{p}_j(t)\|\geq R_i+R_j,\forall (i,j)\in\mathcal{E}_c\}$,
     \item collision reaction region: $\Gamma(\mathbf{p})\triangleq\{\mathbf{p}(t)|\mathbf{p}_i(t)\in\mathbb{R}^3, r_i+r_j<\|\mathbf{p}_i(t)-\mathbf{p}_j(t)\|< R_i+R_j,\exists (i,j)\in\mathcal{E}_c\}$, and
    \item collision region: $\Lambda(\mathbf{p})\triangleq\{\mathbf{p}(t)|\mathbf{p}_i(t)\in\mathbb{R}^3, \|\mathbf{p}_i(t)-\mathbf{p}_j(t)\|< r_i+r_j,\exists (i,j)\in\mathcal{E}_c\}$,
    \item singularity: $\Phi(\mathbf{p})\triangleq\{\mathbf{p}(t)|\mathbf{p}_i(t)\in\mathbb{R}^3, \|\mathbf{p}_i(t)-\mathbf{p}_j(t)\|= r_i+r_j,\exists (i,j)\in\mathcal{E}_c\}$.
\end{enumerate}

\begin{figure}[b]
      \centerline{\includegraphics[width=0.5\textwidth]{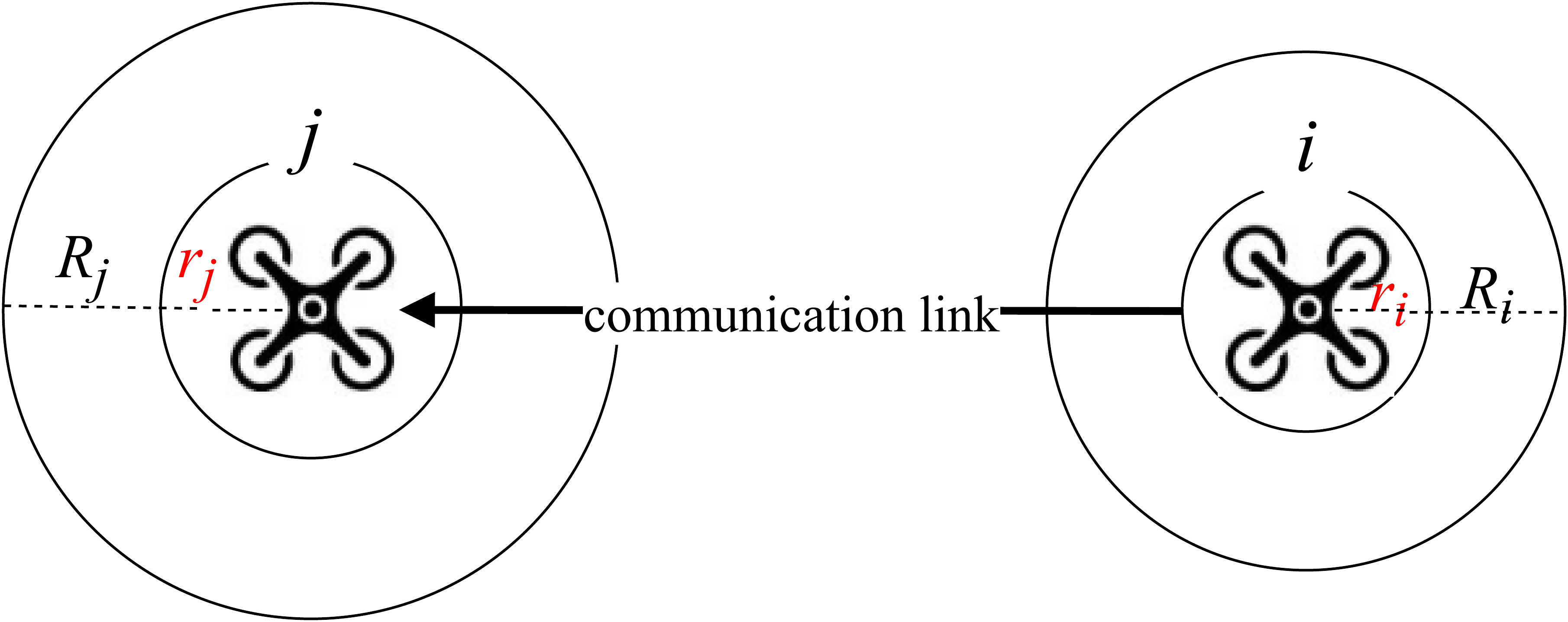}}
\caption{UAV workspace configuration showing the safe distance $r_i$ and collision reaction distance $R_i$.}
\label{fig:radius}
\end{figure}

We make the following assumption.
\begin{assumption}\label{assump:collision}
$\Lambda(\mathbf{p}(0))=\Phi(\mathbf{p}(0))=\Lambda(\mathbf{p}(t_f))=\Phi(\mathbf{p})(t_f)= \emptyset$.
\end{assumption}
Assumption~\ref{assump:collision} ensures no collisions occur at the initial and final UAV positions. This guarantees the feasibility and safety of the formation configuration. By enforcing this condition, collision avoidance is required only during the transition. The multi-UAV formation control problem is defined as follows.

\begin{problem}\label{prob:orginal} (Multi-UAV Formation Control in State Space with Collision Avoidance) (a) Consider a multi-UAV system on the directed connected graph $\mathcal{G}_f(\mathcal{V},\mathcal{E}_f)$ with UAV dynamics (\ref{eq:dynamics-linear})-(\ref{eq:dynamics-angular}). The multi-UAV formation control problem is to determine control strategies that drive all UAVs from their initial state to a desired formation specified by $\mathbf{d}_{ij},\forall(i,j)\in\mathcal{E}_f$ in the state space in a finite time horizon $t_f>0$. (b) All UAVs must avoid mutual collisions, i.e., $\Lambda(\mathbf{p})$ and $\Phi(\mathbf{p})$ must be $\emptyset$ for all $t\in(0,t_f)$.
\end{problem}

The consensus-based formation control objective is to ensure that
\begin{align}\label{eq:cont-obj}
   \sum_{(i,j)\in\mathcal{E}_f}\big\{\|\mathbf{p}_i(t)&-\mathbf{p}_j(t)-\mathbf{d}_{ij}\|^2+\|\dot{\mathbf{p}}_i(t)-\dot{\mathbf{p}}_j(t)\|^2\big\}\rightarrow 0,~ \text{as}~t\rightarrow t_f,
\end{align}
for any initial position $\mathbf{p}_i(0)$ and velocity $\dot{\mathbf{p}}_i(0)$ with $i\in\mathcal{V}$ and while satisfying the collision avoidance constraint
\begin{align}\label{eq:const-col}
&\|\mathbf{p}_i(t)-\mathbf{p}_j(t)\|>r_i+r_j,\quad \forall (i,j)\in\mathcal{E}_c,\forall t\in(0,t_f).
\end{align}

The UAV dynamics, formation control objective, and collision avoidance constraint described by (\ref{eq:dynamics-linear})-(\ref{eq:const-col}) can be recast as a constrained optimal control problem. This allows for investigating optimal formation control strategies for Problem~\ref{prob:orginal}. However, finding a solution is problematic due to the nonlinear UAV dynamics. Numerical methods can be sensitive to initialization~\cite{8187728} and unstable, often relegating them to a lower priority~\cite{bryson1996optimal,betts2010practical}. Moreover, the quadrotor system is inherently underactuated, with six degrees of freedom but only four control inputs. This characteristic poses an additional challenge in finding a solution that adheres to the system's underactuation dynamics.

In the next section, we present a collision-free formation trajectory planning and tracking scheme leveraging the differential flatness property of the UAV dynamics.

\section{Collision-Free Trajectory Planning and Tracking}\label{sec:main}

Leveraging the differential flatness property of the UAV dynamics provides a fully actuated system with four flat outputs~\cite{Mellinger,AI201958,AI201920}. By applying a diffeomorphism (see Definition~1 from~\cite{BEAVER2024111404}), Problem~\ref{prob:orginal} is transformed into a flat space where it becomes easier to handle due to the linear UAV dynamics. The simplified problem, denoted as Problem~\ref{prob:flat}, is recast and addressed as optimal control trajectory planning and tracking problems in flat space. The flat space control strategies are then transformed back to the original coordinate system using the inverse of the diffeomorphism. This approach is summarized in the control scheme diagram in Fig.~\ref{fig:control-scheme}.

\begin{figure}
\centerline{\includegraphics[width=0.7\textwidth]{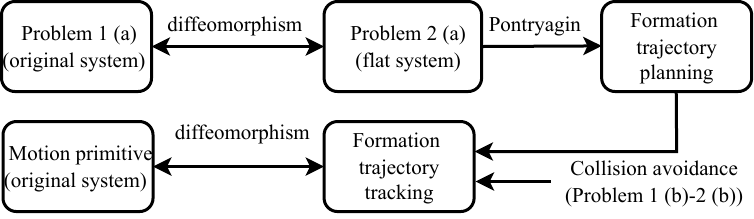}}
\caption{Outline of the control scheme for generating collision-free motion primitives for UAV formation. The formation trajectory planning signal is optimal. Subsequently, this trajectory is tracked within the safety region using an optimal tracking control law. The tracking law becomes sub-optimal in the collision reaction region, ensuring collision-free motion by sacrificing optimality.}
\label{fig:control-scheme}
\end{figure}

\subsection{optimal trajectories in flat coordinates}
The flat output vector is defined as $\mathbf{z}_i(t)=[\mathbf{p}_i^\top(t),\psi_i(t)]^\top\in\mathbb{R}^{4}$ where $\mathbf{p}_i(t)$ is the coordinate of the center of mass of the $i$th UAV in the world coordinate system and $\psi_i(t)$ is its yaw angle. The flat input vector is $\dddot{\mathbf{z}}_i(t)\in\mathbb{R}^{3}$. The collective thrust $u_{i1}(t)$ exerted by the propellers of the $i$th quadrotor is directly determined by the flat outputs and their derivatives. Since angular velocity and acceleration are dependent on these outputs and their derivatives, (\ref{eq:dynamics-angular}) is employed to calculate the inputs $u_{i2}(t)$, $u_{i3}(t)$, and $u_{i4}(t)$ (see~\cite{Mellinger}).

Define $\boldsymbol{\psi}(t)=[\psi_1(t),\cdots,\psi_N(t)]^\top\in\mathbb{R}^{N}$. The UAV team flat state and control input vectors are defined as $\mathbf{r}(t)=[\mathbf{p}^\top(t),\boldsymbol{\psi}^\top(t),\dot{\mathbf{p}}^\top(t),\dot{\boldsymbol{\psi}}^\top(t),\ddot{\mathbf{p}}^\top(t), \ddot{\boldsymbol{\psi}}^\top(t),1]^\top\in\mathbb{R}^{12N+1}$ and $\mathbf{u}_{\mathbf{r}}(t)=[\dddot{\mathbf{p}}(t),\dddot{\boldsymbol{\psi}}(t)]^\top\in\mathbb{R}^{4N}$, respectively. The multi-UAV team dynamics in flat space is given by 
\begin{equation}\label{eq:dynamics}
    \dot{\mathbf{r}}(t)=\mathbf{A}\mathbf{r}(t)+\mathbf{B}\mathbf{u}_{\mathbf{r}}(t), \quad \mathbf{r}(0)=\mathbf{r}_0,
\end{equation}
where $\mathbf{A}=\begin{bmatrix}
       \mathbf{A}_0 & \mathbf{0}\\ \mathbf{0} & 0
    \end{bmatrix}\in\mathbb{R}^{12N+1}$, $\mathbf{A}_0=\begin{bmatrix}
        0 & 1 & 0 \\
        0 & 0 & 1 \\  0 & 0 & 0
    \end{bmatrix}\otimes\mathbf{I}_{4N}$,  $\mathbf{B}=\begin{bmatrix}
        \mathbf{B}_0\\ \mathbf{0}
    \end{bmatrix}\in\mathbb{R}^{(12N+1)\times 4N}$, and $\mathbf{B}_0=\begin{bmatrix}
        0& 0 &  1 
    \end{bmatrix}^\top\otimes\mathbf{I}_{4N}$.

\begin{problem}\label{prob:flat} (Multi-UAV Formation Control in Flat Space with Collision Avoidance) (a) Consider a multi-UAV system on the directed connected graph $\mathcal{G}_f(\mathcal{V},\mathcal{E}_f)$ in flat space with dynamics (\ref{eq:dynamics}). The multi-UAV formation control problem is to determine control strategies for the UAV team that drive all UAVs from their initial state to a desired formation specified by $\mathbf{d}_{ij},\forall(i,j)\in\mathcal{E}_f$ in the flat state space in a finite time horizon $t_f$. (b) Additionally,  $\Lambda(\mathbf{p})$ and $\Phi(\mathbf{p})$ must be $\emptyset$ for all $t\in(0,t_f)$.
\end{problem}

The formation control objective (\ref{eq:cont-obj}) and collision avoidance constraint (\ref{eq:const-col}) in the original coordinates are transferred unchanged to flat coordinates. \textit{Problem}~\ref{prob:flat} (a) pertains to an unconstrained control problem. In the following, we will first address \textit{Problem}~\ref{prob:flat} (a) for formation trajectory planning and then use the resulting trajectory in a collision-free tracking controller to solve \textit{Problem}~\ref{prob:flat} in its entirety (see Fig.~\ref{fig:control-scheme}). 

Using the finite-time linear quadratic paradigm, \textit{Problem}~\ref{prob:flat} (a) is recast as the following optimal control problem 
\begin{align}\label{eq:PI}
\min_{\mathbf{u}_{\mathbf{r}}}~C(\mathbf{r}(t_f))+\int_{0}^{t_f}\big(L(\mathbf{r}(t))+\mathbf{u}_{\mathbf{r}}^\top(t) \mathbf{R} \mathbf{u}_{\mathbf{r}}(t)\big)\,dt,
\end{align}
subject to state dynamics (\ref{eq:dynamics}) where
\begin{align*}
    &C(\mathbf{r}(t_f))=\sum_{(i,j)\in\mathcal{E}_f}\omega_{ij}(\|\mathbf{p}_i(t_f)-\mathbf{p}_j(t_f)-\mathbf{d}_{ij}\|^2+\|\dot{\mathbf{p}}_i(t_f)-\dot{\mathbf{p}}_j(t_f)\|^2),\\
    &L(\mathbf{r}(t))=\sum_{(i,j)\in\mathcal{E}_f}\mu_{ij}(\|\mathbf{p}_i(t)-\mathbf{p}_j(t)-\mathbf{d}_{ij}\|^2+\|\dot{\mathbf{p}}_i(t)-\dot{\mathbf{p}}_j(t)\|^2),
\end{align*}
$\omega_{ij}>0$, $\mathbf{R}=\mathrm{diag}(\gamma_1,\cdots,\gamma_N)\otimes \mathbf{I}_4$, and $\gamma_i>0$. The collective performance index (\ref{eq:PI}) quadratically penalizes the UAV team’s deviations from cooperative formation behavior (i.e., achieving relative distances and velocity consensus) while expending the least control effort.

Using the sum-of-squares property of graph Laplacian matrix in (\ref{eq:sum-of-square}), the terminal and running formation error components in (\ref{eq:PI}), respectively, are compacted as~\cite{jond2019,jond2018} 
\begin{align*} 
   &C(\mathbf{r}(t_f))=\mathbf{r}^\top(t_f) \mathbf{Q}_f \mathbf{r}(t_f), \quad L(\mathbf{z}(t))=\mathbf{r}^\top(t) \mathbf{Q}\mathbf{r}(t),
\end{align*}
where $\mathbf{Q}_f=\begin{bmatrix}
        \hat{\mathbf{L}}_f& \mathbf{0}_{4N} & \mathbf{0}_{4N} &-\boldsymbol{\Theta}_f\\ \mathbf{0}_{4N} & \hat{\mathbf{L}}_f& \mathbf{0}_{4N} &\mathbf{0}_{4N\times 1} \\ \mathbf{0}_{4N} & \mathbf{0}_{4N}& \mathbf{0}_{4N} &\mathbf{0}_{4N\times 1}\\-\boldsymbol{\Theta}_f^\top &\mathbf{0}_{1\times 4N}&\mathbf{0}_{1\times 4N}& \boldsymbol{\Upsilon}_f 
   \end{bmatrix}\in\mathbb{R}^{12N+1}$, $\mathbf{Q}=\begin{bmatrix}
        \hat{\mathbf{L}}& \mathbf{0}_{4N} & \mathbf{0}_{4N} &-\boldsymbol{\Theta} \\ \mathbf{0}_{4N} & \hat{\mathbf{L}}& \mathbf{0}_{4N} &\mathbf{0}_{4N\times 1} \\ \mathbf{0}_{4N} & \mathbf{0}_{4N}& \mathbf{0}_{4N} &\mathbf{0}_{4N\times 1}\\-\boldsymbol{\Theta}^\top &\mathbf{0}_{1\times 4N}&\mathbf{0}_{1\times 4N}& \boldsymbol{\Upsilon} 
   \end{bmatrix}\in\mathbb{R}^{12N+1}$, $\hat{\mathbf{L}}_f=\begin{bmatrix}
       \mathbf{L}_f\otimes\mathbf{I}_3& \mathbf{0}_{3N\times N}\\ \mathbf{0}_{N\times 3N}& \mathbf{0}_{N}
   \end{bmatrix}\in\mathbb{R}^{4N}$, $\hat{\mathbf{L}}=\begin{bmatrix}
       \mathbf{L}\otimes\mathbf{I}_3& \mathbf{0}_{3N\times N}\\ \mathbf{0}_{N\times 3N}& \mathbf{0}_{N}
   \end{bmatrix}\in\mathbb{R}^{4N}$, $\boldsymbol{\Theta}_f=\begin{bmatrix}
      (\mathbf{D}\mathbf{W}_f\otimes\mathbf{I}_3)\mathbf{d} \\ \mathbf{0}_{N\times 1}
   \end{bmatrix}\in\mathbb{R}^{4N\times 1}$, $\boldsymbol{\Theta}=\begin{bmatrix}
      (\mathbf{D}\mathbf{W}\otimes\mathbf{I}_3)\mathbf{d} \\ \mathbf{0}_{N\times 1}
   \end{bmatrix}\in\mathbb{R}^{4N\times 1}$, $\boldsymbol{\Upsilon}_f=\mathbf{d}^\top(\mathbf{W}_f\otimes\mathbf{I}_3)\mathbf{d}\in\mathbb{R}$, $\boldsymbol{\Upsilon}=\mathbf{d}^\top(\mathbf{W}\otimes\mathbf{I}_3)\mathbf{d}\in\mathbb{R}$, $\mathbf{L}_f=\mathbf{D}\mathbf{W}_f\mathbf{D}^\top\in\mathbb{R}^{N}$,  $\mathbf{W}_f=\mathrm{diag}(\cdots,\omega_{ij},\cdots)~\forall (i,j)\in\mathcal{E}\in \mathit{\mathbb{R}}^{|\mathcal{E}|}$, and $\mathbf{d}=[\cdots,\mathbf{d}_{ij}^\top,\cdots]^\top\in \mathit{\mathbb{R}}^{3|\mathcal{E}|}$. The matrices $\mathbf{Q}_f$ and $\mathbf{Q}$ are symmetric and PSD (note that $\mathbf{r}^\top(t_f) \mathbf{Q}_f \mathbf{r}(t_f)\geq 0$, $\mathbf{r}^\top(t) \mathbf{Q} \mathbf{r}(t)\geq 0$).

The objective functional (\ref{eq:PI}) and the system dynamics (\ref{eq:dynamics}) are both convex. Thus, Pontryagin’s principle is both a necessary and sufficient condition for optimality. Applying these conditions, the optimal formation trajectory planning signal for any initial state of the UAV team is presented as follows.

\begin{theorem}\label{theorem:sol-open}
Consider the multi-UAV formation control problem in flat coordinates defined by (\ref{eq:dynamics}) and (\ref{eq:PI}). The unique globally optimal control input and corresponding optimal formation trajectory for any initial state $\mathbf{r}_0$ are given by
\begin{align}
&\mathbf{u}_{\mathbf{r}}(t)=-\mathbf{R}^{-1}\mathbf{B}^\top\mathbf{G}(t_f-t)\mathbf{H}^{-1}(t_f)\mathbf{r}_0,  \label{eq:open-con}\\&\mathbf{r}(t)=\mathbf{H}(t_f-t)\mathbf{H}^{-1}(t_f)\mathbf{r}_0, \label{eq:open-traj} 
 \end{align}
 where 
 \begin{align}
    &\mathbf{H}(t)=\begin{bmatrix}
     \mathbf{I} & \mathbf{0} 
 \end{bmatrix}\exp{(-t\mathbf{M})}\begin{bmatrix}
     \mathbf{I} \\ \mathbf{Q}_f 
 \end{bmatrix},\label{eq:Ht}\\
 &\mathbf{G}(t)=\begin{bmatrix}
     \mathbf{0} & \mathbf{I} 
 \end{bmatrix}\exp{(-t\mathbf{M})}\begin{bmatrix}
     \mathbf{I} \\ \mathbf{Q}_f 
 \end{bmatrix},\label{eq:Gt}\\
 &\mathbf{M}=\begin{bmatrix}
      \mathbf{A} &  -\mathbf{B}\mathbf{R}^{-1} \mathbf{B}^\top  \\
     -\mathbf{Q} &  -\mathbf{A}^\top 
 \end{bmatrix}.\label{eq:M}
 \end{align}
\end{theorem}
\begin{proof}
   See Appendix~\ref{app:proof}.
\end{proof}

\subsection{Formation trajectory tracking}

In the control scheme shown in Fig.~\ref{fig:control-scheme}, the formation trajectory planning signal, as given by (\ref{eq:open-traj}), is followed by the formation trajectory tracker. Meanwhile, the tracker must ensure that each UAV avoids collisions with its neighboring UAVs along its specified trajectory whenever \( \Gamma(\mathbf{p}) \) becomes non-empty.

\subsubsection{safety region control strategy} \label{sec:safety}
When \( \Psi(\mathbf{p}) \neq \emptyset \), the UAVs are within the safety region, and as a result, the inter-UAV collision avoidance constraint is inactive. The control strategy for the UAVs to follow the formation trajectory planning signal within the safety region is derived as follows.

Using the finite-time LQR paradigm, formation trajectory tracking within the safety region is recast as the following optimal control problem
\begin{align}\label{eq:PI-track-comp}
&\min_{\mathbf{u}_{\mathbf{z}}}~(\mathbf{z}(t_f)-\mathbf{r}(t_f))^\top \mathbf{K}_f (\mathbf{z}(t_f)-\mathbf{r}(t_f))+\nonumber\\&\int_{0}^{t_f}\Big\{(\mathbf{z}(t)-\mathbf{r}(t))^\top \mathbf{K} (\mathbf{z}(t)-\mathbf{r}(t))+\mathbf{u}_{\mathbf{z}}^\top(t) \mathbf{R}_{\mathbf{z}}\mathbf{u}_{\mathbf{z}}(t)\Big\}\,dt,
\end{align}
subject to
\begin{equation}\label{eq:dynamics-}
    \dot{\mathbf{z}}(t)=\mathbf{A}\mathbf{z}(t)+\mathbf{B}\mathbf{u}_{\mathbf{z}}(t), \quad \mathbf{z}(0)=\mathbf{r}_0,
\end{equation}
where $\mathbf{K}_f=\mathrm{diag}(\mathbf{K}_{0f},0)$, $\mathbf{K}_{0f}=\mathbf{I}_{3}\otimes\mathrm{diag}(\delta_1,\cdots,\delta_N)\otimes  \mathbf{I}_4$, $\mathbf{K}=\mathrm{diag}(\mathbf{K}_0,0)$, $\mathbf{K}_0=\mathbf{I}_{3}\otimes\mathrm{diag}(\zeta_1,\cdots,\zeta_N)\otimes  \mathbf{I}_4$, $\mathbf{R}_{\mathbf{z}}=\mathrm{diag}(\eta_1,\cdots,\eta_N)\otimes\mathbf{I}_{4}$, $\delta_i>0$, $\zeta_i>0$, and $\eta_i>0$.

The optimal feedback control for (\ref{eq:PI-track-comp}) and (\ref{eq:dynamics-}) is given by
\begin{equation}\label{eq:open-loop2}
   \mathbf{u}_{\mathbf{z}}(t)=-\mathbf{R}_{\mathbf{z}}^{-1}\mathbf{B}^\top\mathbf{P}(t)(\mathbf{z}(t)-\mathbf{r}(t)), 
\end{equation}
where $\mathbf{P}(t)$ is the solution to the Riccati differential equation
\begin{align}\label{eq:riccati}
&\dot{\mathbf{P}}(t)+\mathbf{P}(t)\mathbf{A}-\mathbf{P}(t)\mathbf{B}\mathbf{R}_{\mathbf{z}}^{-1}\mathbf{B}^\top\mathbf{P}(t)+\mathbf{K}+\mathbf{A}^\top\mathbf{P}(t)=\mathbf{0},\quad\mathbf{P}(t_f)=\mathbf{K}_f.
\end{align}
The pair $(\mathbf{A},\mathbf{B})$ is controllable, $\mathbf{R}_{\mathbf{z}}\succ 0$, and $\mathbf{K}\succ 0$, thus, there exists a unique solution $\mathbf{P}(t)=\mathbf{P}^\top(t)\succeq 0$ to (\ref{eq:riccati}).

The optimal state feedback is given by
\begin{equation}\label{eq:dynamics2}
    \dot{\mathbf{z}}(t)=\left(\mathbf{A}-\mathbf{B}\mathbf{R}_{\mathbf{z}}^{-1}\mathbf{B}^\top\mathbf{P}(t)\right) (\mathbf{z}(t)-\mathbf{r}(t)).
\end{equation}
The minimal cost attained at any $t\in[0,t_f]$ and state $\mathbf{z}(t)$ is yielded by the value function
\begin{equation}
    n(\mathbf{z})=(\mathbf{z}(t)-\mathbf{r}(t))^\top \mathbf{P}(t) (\mathbf{z}(t)-\mathbf{r}(t)).
\end{equation}
The Hamilton–Jacobi–Bellman (HJB) equation is given by
\begin{align} \label{eq:A}
    \frac{d n}{d t}+&\min_{\mathbf{u}_{\mathbf{z}}}\big\{(\mathbf{z}(t)-\mathbf{r}(t))^\top \mathbf{K} (\mathbf{z}(t)-\mathbf{r}(t))+\mathbf{u}_{\mathbf{z}}^\top(t) \mathbf{R}_{\mathbf{z}} \mathbf{u}_{\mathbf{z}}(t)+\frac{\partial n}{\partial \mathbf{z}}\big(\mathbf{A}\mathbf{z}(t)+\mathbf{B}\mathbf{u}_{\mathbf{z}}(t) \big) \big\}=0.
\end{align}
The optimal input (\ref{eq:open-loop2}), equivalently is derived from the HJB equation (\ref{eq:A}) as
\begin{align*} 
    \mathbf{u}_{\mathbf{z}}(t)=-\mathbf{R}_{\mathbf{z}}^{-1}\mathbf{B}^\top\frac{\partial n}{\partial \mathbf{z}}.
\end{align*}

\subsubsection{collision avoidance control strategy}\label{sec:collision-avoidance}
When $\Gamma(\mathbf{p})\neq \emptyset$, the formation trajectory tracker applies the collision avoidance control strategy derived below for the collision reaction region. First, we define a collision penalty function~\cite{10.1115/1.2764510,HUANG20204034} as
\begin{align}\label{eq:penalty}
    v_{ij}=&\big(\min\big\{0,\frac{\|\mathbf{p}_j(t)-\mathbf{p}_i(t)\|^2-(R_i+R_j)^2}{\|\mathbf{p}_j(t)-\mathbf{p}_i(t)\|^2-(r_i+r_j)^2}\big\}\big)^2=\big(\min\big\{0,\frac{\|\mathbf{e}_{ij}^\top\mathbf{z}(t)\|^2-(R_i+R_j)^2}{\|\mathbf{e}_{ij}^\top\mathbf{z}(t)\|^2-(r_i+r_j)^2}\big\}\big)^2, 
\end{align}
for all $(i,j)\in\mathcal{E}_c$ where $\mathbf{e}_{ij}=[-\mathbf{D}_{ij}\otimes\mathbf{I}_3,\mathbf{0}]\in\mathbb{R}^{(12N+1)\times 3}$ and $\mathbf{D}_{ij}$ is the column of $\mathbf{D}$ corresponding to the edge $(i,j)$. Note that $v_{ij}$ is nonzero only within the collision reaction region and is either zero or undefined outside of this region.  

The partial derivative of $v_{ij}$ with respect to $\mathbf{z}_i$ is given by
\begin{equation*}
   \frac{\partial v_{ij}}{\partial \mathbf{z}_i}=\left \{
  \begin{aligned}
    &\mathbf{0}, && \text{if}\ \Psi(\mathbf{p})\neq \emptyset \\
    &\mathbf{z}^\top(t)\mathbf{e}_{ij}\phi_{ij},&& \text{if}\ \Gamma(\mathbf{p})\neq \emptyset \\
    &\mathbf{0}, && \text{if}\ \Lambda(\mathbf{p})\neq \emptyset\\ &\text{not defined}, && \text{if}\ \Phi(\mathbf{p})\neq \emptyset 
  \end{aligned} \right.
\end{equation*}
where
\begin{align*}
    \phi_{ij}=&4((R_i+R_j)^2-(r_i+r_j)^2)\frac{\|\mathbf{e}_{ij}^\top\mathbf{z}(t)\|^2-(R_i+R_j)^2}{(\|\mathbf{e}_{ij}^\top\mathbf{z}(t)\|^2-(r_i+r_j)^2)^3}.
\end{align*}
Moreover, the time derivative of $v_{ij}$ at the collision reaction region (i.e., at $\Gamma(\mathbf{p})\neq\emptyset$) is
\begin{equation} \label{eq:derivativee}
   \frac{d v_{ij}}{dt}=\frac{\partial v_{ij}}{\partial \mathbf{z}_i}\mathbf{e}_{ij}^\top\dot{\mathbf{z}}(t).
\end{equation}

Let $V_i=\sum_{j\in\mathcal{N}_i}v_{ij}$ where $\frac{\partial V_i}{\partial \mathbf{z}_i}=\sum_{j\in\mathcal{N}_i}\mathbf{z}^\top(t)\mathbf{e}_{ij}\phi_{ij}$ for all $i\in\mathcal{V}$. Define $\frac{\partial V}{\partial \mathbf{z}}=[\mathbf{0}_{8N},\frac{\partial V_1}{\partial \mathbf{z}_1},\cdots,\\\frac{\partial V_N}{\partial \mathbf{z}_N},\mathbf{0}_{N+1}]\in\mathbb{R}^{12N+1}$. Now, augment the value function of the HJB (\ref{eq:A}) with the collision avoidance penalty \(V\)
\begin{align} \label{eq:C}
    &\big(\frac{d n}{d t}+\frac{d V}{d t}\big)+\min_{\mathbf{u}_{\mathbf{z}}}\big\{(\mathbf{z}(t)-\mathbf{r}(t))^\top \mathbf{K} (\mathbf{z}(t)-\mathbf{r}(t))+\mathbf{u}_{\mathbf{z}}^\top(t) \mathbf{R}_{\mathbf{z}} \mathbf{u}_{\mathbf{z}}(t)+\nonumber\\&\big(\frac{\partial n}{\partial \mathbf{z}}+\frac{\partial V}{\partial \mathbf{z}}\big)\big(\mathbf{A}\mathbf{z}(t)+\mathbf{B}\mathbf{u}_{\mathbf{z}}(t)\big)\big\}=0.
\end{align}
Note that the HJB in (\ref{eq:C}) is constructed for the tracking optimal control problem with the performance index (\ref{eq:PI-track-comp}), augmented with a penalty function to indirectly designate the collision region instead of the inequality constraint, ensuring safety~\cite{LEE201765}. 

The collision avoidance control strategy is derived by minimizing the HJB equation in (\ref{eq:C}), leading to the following feedback control law
\begin{align} \label{eq:feedback}
   \mathbf{u}_{\mathbf{z}}(t) = -\mathbf{R}_{\mathbf{z}}^{-1}\mathbf{B}^\top( \mathbf{P}(t) \mathbf{z}(t) + \frac{\partial V}{\partial \mathbf{z}}^\top ).
\end{align}
The corresponding state dynamics is given by
\begin{equation}\label{eq:dynamics-0}
    \dot{\mathbf{z}}(t) = ( \mathbf{A} - \mathbf{B}\mathbf{R}_{\mathbf{z}}^{-1}\mathbf{B}^\top \mathbf{P}(t) ) \mathbf{z}(t) - \mathbf{B}\mathbf{R}_{\mathbf{z}}^{-1}\mathbf{B}^\top \frac{\partial V}{\partial \mathbf{z}}^\top.
\end{equation}

It is important to note that the collision avoidance control law within the safety region (i.e., when UAVs are not within the collision reaction zone) coincides with the optimal control law for the system.

To demonstrate that (\ref{eq:feedback}) provides a collision-free control strategy, we prove that (\ref{eq:C}) is non-increasing in the region \(\Gamma(\mathbf{p})\), following the approach in~\cite{10.1115/1.2764510}. 

We define \(\hat{V} = \sum_{i=1}^N V_i\), where \(V_i\) represents the sum of collision penalty functions for UAV \(i\) regarding its neighbors. The following lemma provides the basis for showing that UAVs maintain a collision-free formation.

\begin{lemma} \label{lma:ICs}
If the initial conditions of the UAVs satisfy
\begin{equation}
    ||\mathbf{p}_i(0) - \mathbf{p}_j(0)|| \geq r_i + r_j + \epsilon,
\end{equation}
for all \(i,j \in \mathcal{V}\), \(i \neq j\), and some \(\epsilon > 0\), then the UAVs can achieve \(\frac{d\hat{V}}{dt} \geq 0\) using bounded control effort, thereby avoiding collisions.
\end{lemma}

\begin{proof}
Let $\frac{d\hat{V}}{dt}< 0$, since any other condition trivially satisfies Lemma~\ref{lma:ICs}. By the fundamental theorem of calculus,
\begin{equation}
    \hat{V}(t) = \hat{V}(0) + \int_0^{t} \frac{d\hat{V}}{d\tau}\, d\tau.
\end{equation}
The control input $\mathbf{u}_{\mathbf{z}}$ appears in $\frac{d^{n+1} \hat{V}}{dt^{n+1}}$ (i.e., $n=2$ here). We can obtain $\hat{V}(t)$ and $\frac{d\hat{V}}{dt}$ as a function of $\mathbf{u}_z$,
\begin{align}\label{eq:V}
    \hat{V}(t) =& \hat{V}(0) +\sum_{i=1}^n\left. \frac{d^n\hat{V}}{dt^n} \right|_{t=0}\frac{t^n}{n!}+  
    \underbrace{\int_0^t \dots \int_0^t}_{n+1} \frac{d^{n+1}\hat{V}}{d\tau^{n+1}} \underbrace{d\tau \dots d\tau}_{n+1}, 
\end{align}
\begin{align}\label{eq:dV/dt}
    \frac{d\hat{V}}{dt} =& \sum_{i=1}^n\left. \frac{d^n\hat{V}}{dt^n} \right|_{t=0}\frac{t^{n-1}}{(n-1)!}+\underbrace{\int_0^t \dots \int_0^t}_{n} \frac{d^{n+1}\hat{V}}{d\tau^{n+1}} \underbrace{d\tau \dots d\tau}_{n},
\end{align}
where $\left. \frac{d^n\hat{V}}{dt^n} \right|_{t=0}$ is the initial condition of the $n^{th}$ derivative of $\hat{V}(t)$.

To guarantee collision avoidance while 
$\frac{d\hat{V}}{dt}\to 0$ at time $t$, we require
\begin{align}
    \hat{V}(t) &\leq \hat{V}(0) + \delta, \label{eq:lma-des1} \\
    \frac{d\hat{V}}{dt} &\leq 0, \label{eq:lma-des2}
\end{align}
where $\delta>0$ and \eqref{eq:lma-des1} implies collision avoidance, as $\frac{d\hat{V}}{dt} > 0$ implies that $\hat{V}(t) = \max_{\tau\in[0,t]} \hat{V}(\tau)$.
Substituting (\ref{eq:V}) and (\ref{eq:dV/dt}) into (\ref{eq:lma-des1}) and (\ref{eq:lma-des2}), respectively, and rearranging their terms yield
\begin{align}
    &\underbrace{\int_0^t \dots \int_0^t}_{n+1}  \frac{d^{n+1}\hat{V}}{d\tau^{n+1}} \underbrace{d\tau \dots d\tau}_{n+1}  \leq 
    -\big(\delta+\sum_{i=1}^n\left. \frac{d^n\hat{V}}{dt^n} \right|_{t=0}\frac{t^n}{n!}\big), \label{eq:lma-cond1}
    \\
    &\underbrace{\int_0^t\dots \int_0^t}_{n}\frac{d^{n+1}\hat{V}}{d\tau^{n+1}}\underbrace{d\tau\dots d\tau}_{n}  \leq 
    -\sum_{i=1}^n\left. \frac{d^n\hat{V}}{dt^n} \right|_{t=0}\frac{t^{n-1}}{(n-1)!} \label{eq:lma-cond2}.
\end{align}

Both (\ref{eq:lma-cond1}) and (\ref{eq:lma-cond2}) are finite lower bounds on the function $\frac{d^{n+1}\hat{V}}{dt^{n+1}}$. Therefore, we can always construct a finite function $\frac{d^{n+1}\hat{V}}{dt^{n+1}}$ to satisfy \eqref{eq:lma-cond1} and \eqref{eq:lma-cond2} for an arbitrarily small $t > 0$. This implies that \eqref{eq:lma-des1} and \eqref{eq:lma-des2} are satisfied for any $\delta > 0$. Finally, by continuity in $\frac{d\hat{V}}{dt}$, 
we can select $\delta$ sufficiently small to ensure \eqref{eq:lma-des1} and \eqref{eq:lma-des2} for any $\epsilon > 0$. This concludes the proof.
\end{proof}

\begin{theorem} \label{thm:safe}
    Given the conditions of Lemma \ref{lma:ICs}, any trajectory that results in a collision is sub-optimal.
\end{theorem}

\begin{proof}
We use a proof by contradiction. Assume that a trajectory where two UAVs collide is optimal, resulting in a cost of $J^* = \infty$ by definition.
By Lemma \ref{lma:ICs}, there exists a bounded control input $\mathbf{u}_{\mathbf{z}}^\dagger$ such that $\left. \frac{d\hat{V}}{dt} \right|_{t=t^\dagger}=0$ in some finite time $t^\dagger<t_f$ without collisions.
The corresponding cost is
    \begin{align} \label{eq:stop-cost}
        J^\dagger &= \int_0^{t^\dagger}\Big\{(\mathbf{z}(t)-\mathbf{r}(t))^\top \mathbf{K} (\mathbf{z}(t)-\mathbf{r}(t))+(\mathbf{u}_{\mathbf{z}}^\dagger(t))^\top \mathbf{R}_{\mathbf{z}}\mathbf{u}_{\mathbf{z}}^\dagger(t)+ \frac{d\hat{V}}{dt}\Big\}\, dt \notag \\ &+ \int_{t^\dagger}^{t_f}\Big\{(\mathbf{z}(t)-\mathbf{r}(t))^\top \mathbf{K} (\mathbf{z}(t)-\mathbf{r}(t))+(\mathbf{u}_{\mathbf{z}}^\dagger(t)) ^\top\mathbf{R}_{\mathbf{z}}\mathbf{u}_{\mathbf{z}}^\dagger(t)\Big\}\, dt,
    \end{align}
where $V$ remains stationary for $t \geq t^\dagger$.
Every term in \eqref{eq:stop-cost} is finite. Therefore, $J^\dagger < J^*$ and $J^*$ is not optimal. This contradiction implies that our original assumption is incorrect, proving Theorem~\ref{thm:safe}.
\end{proof}

\section{Directionally Aware Collision Avoidance}\label{sec:collision}
In this section, we introduce two collision avoidance strategies that incorporate directional awareness by integrating the UAVs' velocities into the collision penalty function in Subsection~\ref{sec:collision-avoidance}. Recognizing that combining these strategies enhances the effectiveness of collision avoidance, we unify them into a single approach to provide UAVs with a more robust, directionally aware collision avoidance capability.

\subsection{Forward-path aware collision avoidance}\label{sec:forward}
To derive a collision avoidance control strategy that is directionally aware, we adjust the penalty function to prioritize avoiding collisions in the forward path of each UAV. We incorporate a forward-path aware weighting factor \(\alpha_{ij} \) into the collision penalty function~\eqref{eq:penalty} as
  \begin{equation*}
   v^F_{ij} = \alpha_{ij}v_{ij},
  \end{equation*}
where
 \begin{equation} \label{eq:weight}
\alpha_{ij} =
\begin{cases} 
  \max\{0, \cos(\eta_{ij})\}, & \text{if } \|\mathbf{v}_i(t)\| \neq 0, \\
  0, & \text{if } \|\mathbf{v}_i(t)\| = 0,
\end{cases}
\end{equation}
with
\[
\cos(\eta_{ij}) = \frac{(\mathbf{p}_j(t) - \mathbf{p}_i(t))^\top \mathbf{v}_i(t)}{\|\mathbf{p}_j(t) - \mathbf{p}_i(t)\|\|\mathbf{v}_i(t)\|}.
\]

The forward-path aware penalty function \(v^F_{ij}\) accounts for whether UAV \( j \) lies within the forward path of UAV \( i \), and it is zero otherwise. Specifically, when UAV \(j\) is in the forward path, the directionally aware weighting factor \(\alpha_{ij} \) is given by the cosine of the angle between the relative position vector and the velocity vector, as depicted in Fig.~\ref{fig:dir}. The cosine is obtained using the dot product as shown in~\eqref{eq:weight}. If UAV \(j\) is positioned behind UAV \(i\), the weighting factor \(\alpha_{ij} \) becomes zero.
\begin{figure}[htbp]
	\centering
    \includegraphics[width=0.6\linewidth]{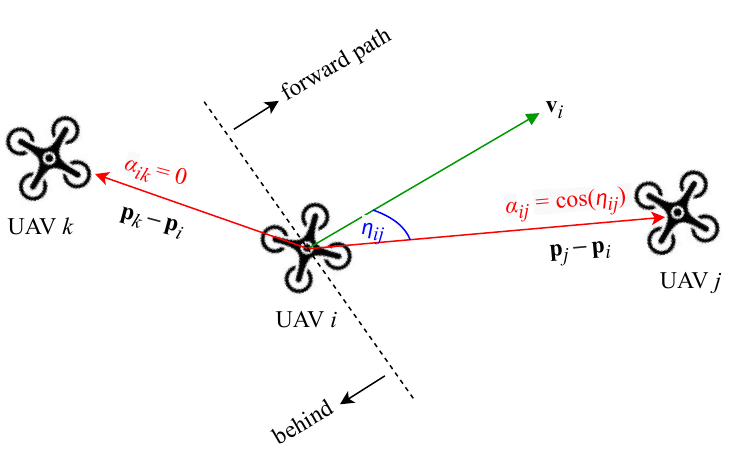}
\caption{The forward-path aware collision avoidance prioritizes avoiding collisions in the forward path. The velocity vector \(\mathbf{v}_i\) indicates the forward motion path of UAV \(i\). Within the collision reaction region, when UAV \(j\) is located in the forward path of UAV \(i\), the directionally aware weighting factor \(\alpha_{ij} \) for UAV \(i\) with respect to UAV \(j\) is defined as the cosine of the angle between the UAV \(i\)'s velocity vector \(\mathbf{v}_i\) and their relative distance vector \(\mathbf{p}_j-\mathbf{p}_i\). The value of \(\alpha_{ij}\) decreases as the UAV \(j\) moves further to the sides. By definition, \(\alpha_{ik}\) is zero when UAV \(k\) is located behind UAV \(i\).}
\label{fig:dir}
\end{figure}

\subsection{Approach-aware collision avoidance}\label{sec:approach}
In approach-aware collision avoidance, the collision penalty function is adjusted to prevent a collision when two UAVs are on an approach trajectory. This approach leverages the relative velocities to assess whether the UAVs are approaching each other. The approach-aware weighting factor \(\beta_{ij} \) is incorporated into~\eqref{eq:penalty} as
  \begin{equation*}
   v^A_{ij} = \beta_{ij}v_{ij},
  \end{equation*}
where
\begin{equation*}
\beta_{ij} =
\begin{cases}
  \max\left\{ 0,  -\cos(\zeta_{ij}) \right\}, & \text{if } \|\mathbf{v}_j(t) - \mathbf{v}_i(t)\| \neq 0, \\
  0, & \text{if } \|\mathbf{v}_j(t) - \mathbf{v}_i(t)\| = 0,
\end{cases}
\end{equation*}
with
\[
\cos(\zeta_{ij}) = \frac{(\mathbf{p}_j(t) - \mathbf{p}_i(t))^\top (\mathbf{v}_j(t) - \mathbf{v}_i(t))}{\|\mathbf{p}_j(t) - \mathbf{p}_i(t)\| \|\mathbf{v}_j(t) - \mathbf{v}_i(t)\|}.
\]
Here, \(\cos(\zeta_{ij})\) captures the alignment between the relative position \(\mathbf{p}_j(t) - \mathbf{p}_i(t)\) and velocity \(\mathbf{v}_j(t) - \mathbf{v}_i(t)\). When \(\cos(\zeta_{ij}) < 0\), the UAVs are approaching, posing a collision risk. If their velocities are identical, i.e., \(\|\mathbf{v}_j(t) - \mathbf{v}_i(t)\| = 0\), they remain at a constant separation. A higher \(\beta_{ij}\) (near 1) indicates a direct, high-speed approach, while a lower \(\beta_{ij}\) (closer to 0) suggests an oblique approach with reduced collision risk. When \(\beta_{ij} = 0\), the UAVs are either parallel, diverging, or stationary relative to each other.

\subsection{Directionally aware unified collision avoidance}

We combine the collision avoidance strategies from Subsections~\ref{sec:forward} and~\ref{sec:approach} into a unified approach that accounts for both UAV \(j\) being within UAV \(i\)'s forward path and the relative approach. This is achieved by defining the weighting factor \(\xi_{ij}\) as
\begin{equation*}
\xi_{ij} = \alpha_{ij} \beta_{ij},
\end{equation*}
or
\[
\xi_{ij} =
\begin{cases}
  \max\{0, \cos(\eta_{ij})\} \max\{0, -\cos(\zeta_{ij})\}, & \text{if } \|\mathbf{v}_i(t)\| \neq 0 \text{ and } \|\mathbf{v}_j(t) - \mathbf{v}_i(t)\| \neq 0, \\
  0, & \text{otherwise}.
\end{cases}
\]
Here, \(\alpha_{ij}\) ensures that UAV \(j\) is in the forward path of UAV \(i\). If UAV \(j\) is outside the forward path, \(\alpha_{ij} = 0\), making \(\xi_{ij} = 0\). \(\beta_{ij}\) ensures that UAV \(j\) is approaching UAV \(i\). If UAV \(j\) is moving away or stationary relative to UAV \(i\), \(\beta_{ij} = 0\), making \(\xi_{ij} = 0\). Thus, \(\xi_{ij}\) is nonzero only when UAV \(j\) lies in the forward path of UAV \(i\) (\(\alpha_{ij} > 0\)) and UAV \(j\) is approaching UAV \(i\) (\(\beta_{ij} > 0\)) and is zero if either condition is not met.

To demonstrate the improved effectiveness of the unified collision avoidance approach, we analyze two scenarios with two UAVs having evolving positions and velocities. In the first scenario, \(\mathbf{p}_1(0) = [1, 1,1]^\top\) and \(\mathbf{p}_2(0) = [7, 0,3]^\top\); in the second, \(\mathbf{p}_2(0)\) is changed to \([7, 3,1]^\top\). Both scenarios feature constant velocities \(\mathbf{v}_1(0) = [2, 1,0.5]^\top\) and \(\mathbf{v}_2(0) = [-2, 1,-0.5]^\top\). The evolution of relative vectors, \(\alpha_{12}\), \(\beta_{12}\), and \(\xi_{12}\) at some time steps are shown in Fig.~\ref{fig:relative_vectors}. Figs.~\ref{fig:relative_vectors}(\subref{vector:a})--(\subref{vector:d}) represent the first scenario, while Figs.~\ref{fig:relative_vectors}(\subref{vector:e})--(\subref{vector:h}) depict the second. Figs.~\ref{fig:relative_vectors}(\subref{vector:d}) and~\ref{fig:relative_vectors}(\subref{vector:h}) compare \(\alpha_{12}\) and \(\beta_{12}\) over 20 time steps. It is evident that \(\xi_{12}\) consistently captures collision risks more effectively and shows to be more reliable overall. 

\begin{figure}[htbp]
    \centering
    \begin{subfigure}{0.24\textwidth}
        \centering
        \includegraphics[width=\linewidth]{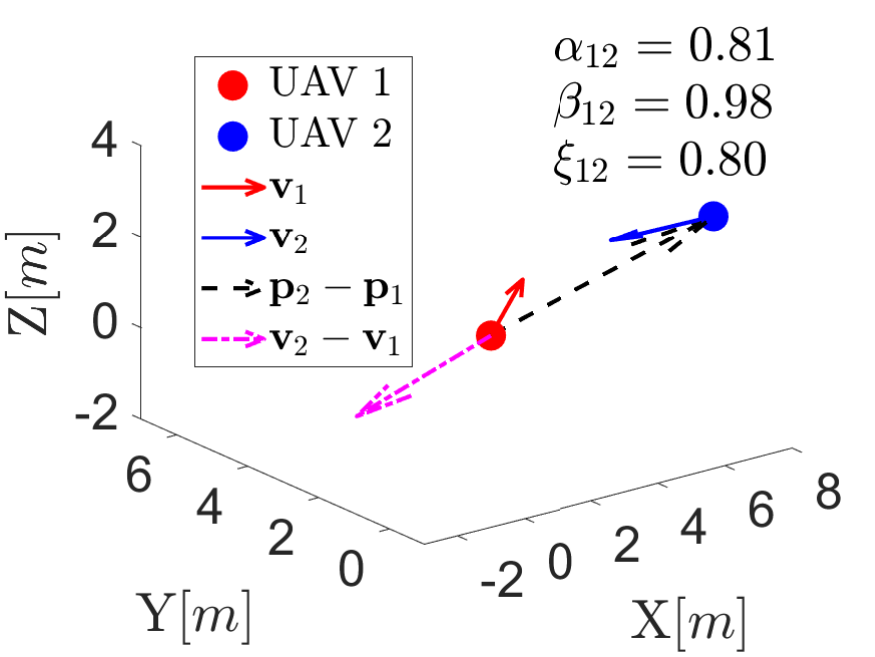}
        \subcaption[]{} \label{vector:a}
    \end{subfigure}
    \hfill
    \begin{subfigure}{0.24\textwidth}
        \centering
        \includegraphics[width=\linewidth]{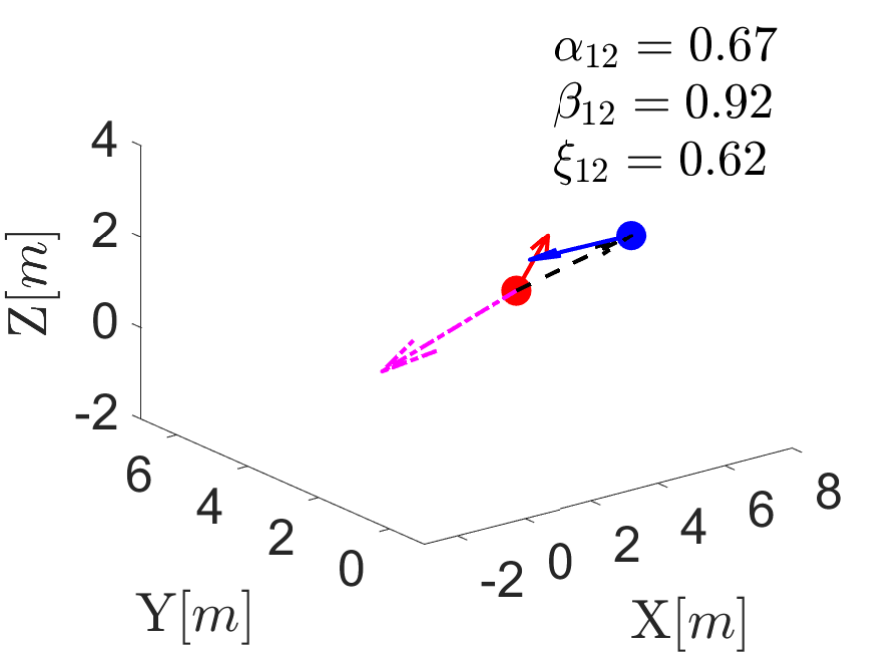}
        \subcaption[]{} \label{vector:b}
    \end{subfigure}
    \hfill
    \begin{subfigure}{0.24\textwidth}
        \centering
        \includegraphics[width=\linewidth]{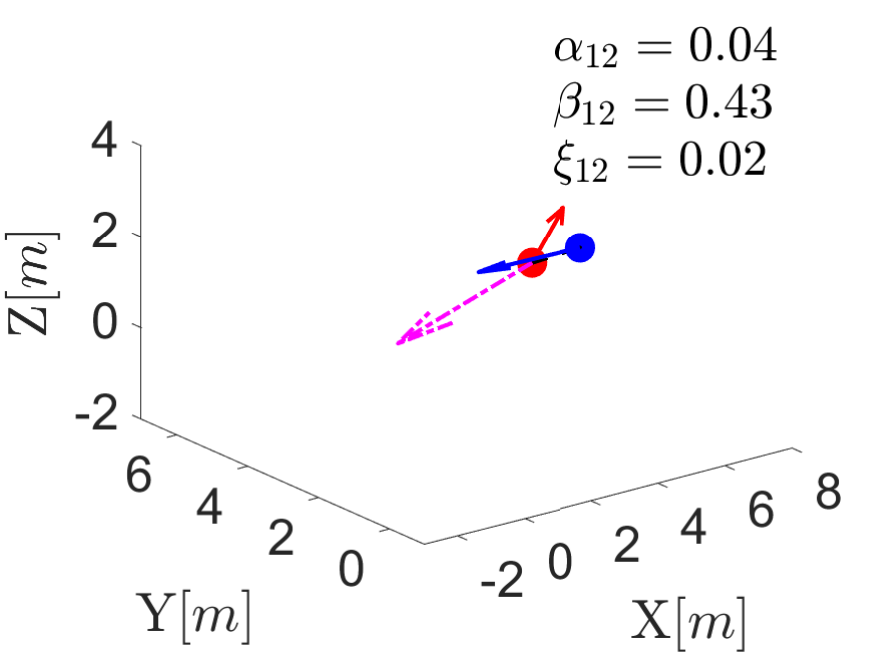}
        \subcaption[]{} \label{vector:c}
    \end{subfigure}
    \hfill
    \begin{subfigure}{0.24\textwidth}
        \centering
        \includegraphics[width=\linewidth]{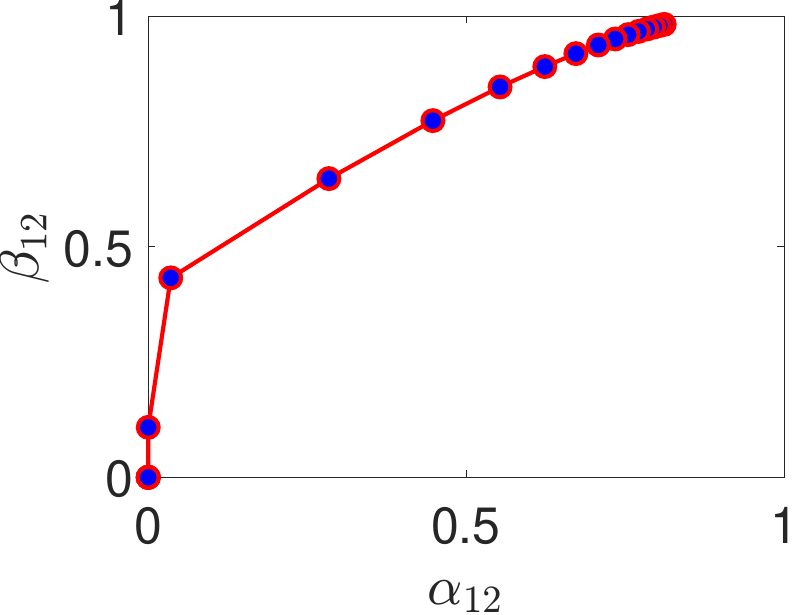}
        \subcaption[]{} \label{vector:d}
    \end{subfigure}    
    \vskip\baselineskip
    \begin{subfigure}{0.24\textwidth}
        \centering
        \includegraphics[width=\linewidth]{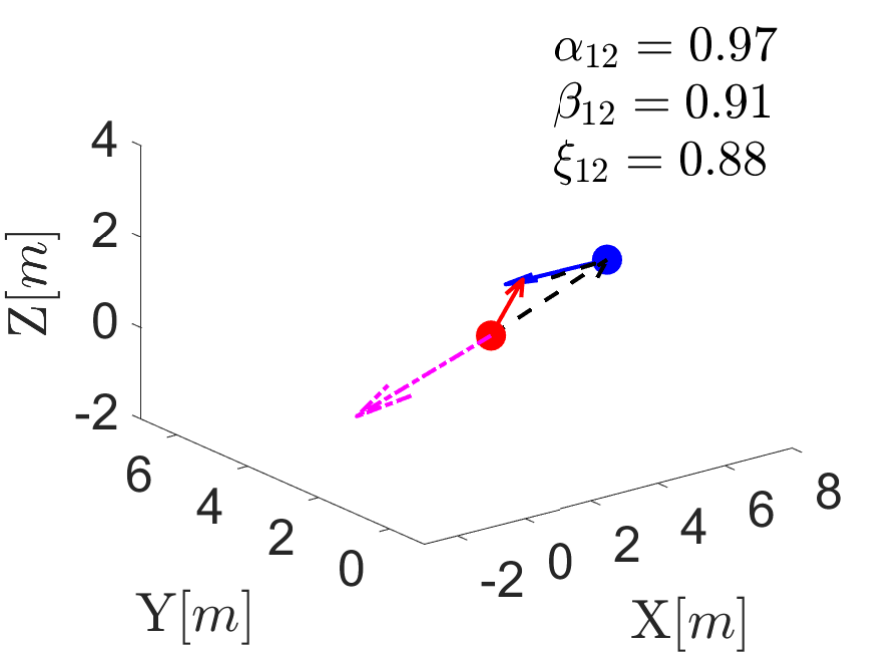}
        \subcaption[]{} \label{vector:e}
    \end{subfigure}
    \hfill
    \begin{subfigure}{0.24\textwidth}
        \centering
        \includegraphics[width=\linewidth]{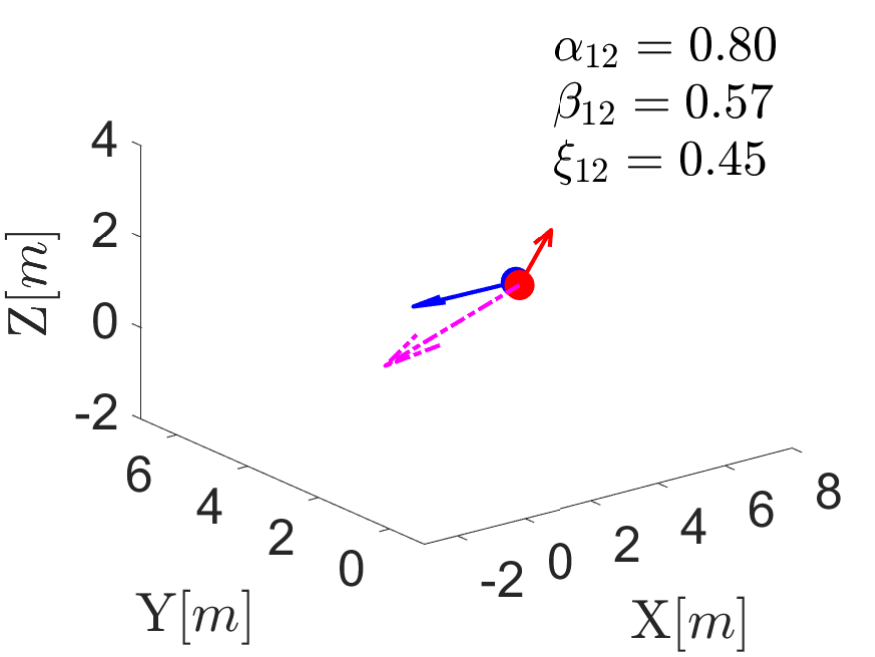}
        \subcaption[]{} \label{vector:f}
    \end{subfigure}
    \hfill
    \begin{subfigure}{0.24\textwidth}
        \centering
        \includegraphics[width=\linewidth]{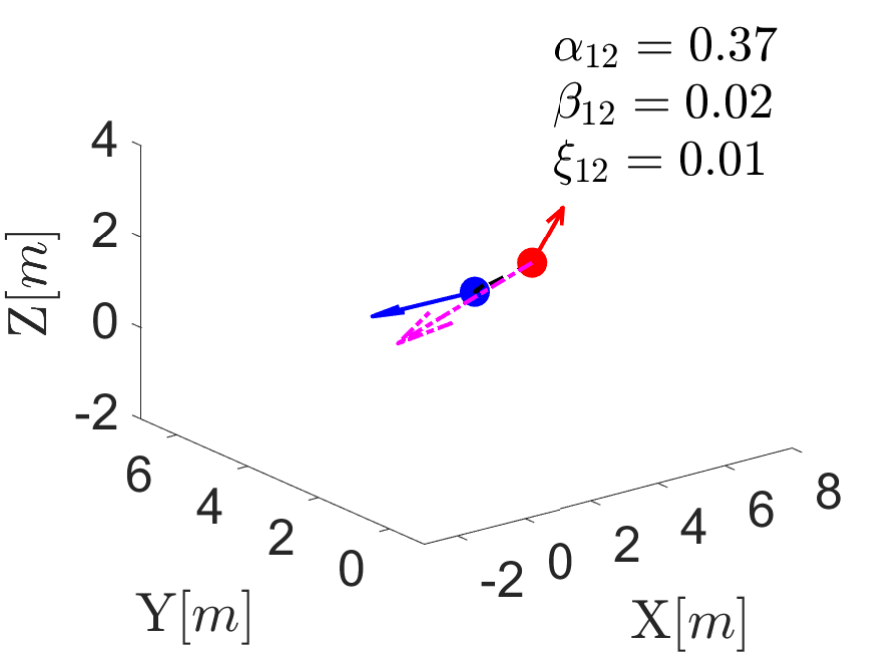}
        \subcaption[]{} \label{vector:g}
    \end{subfigure}
    \hfill
    \begin{subfigure}{0.24\textwidth}
        \centering
        \includegraphics[width=\linewidth]{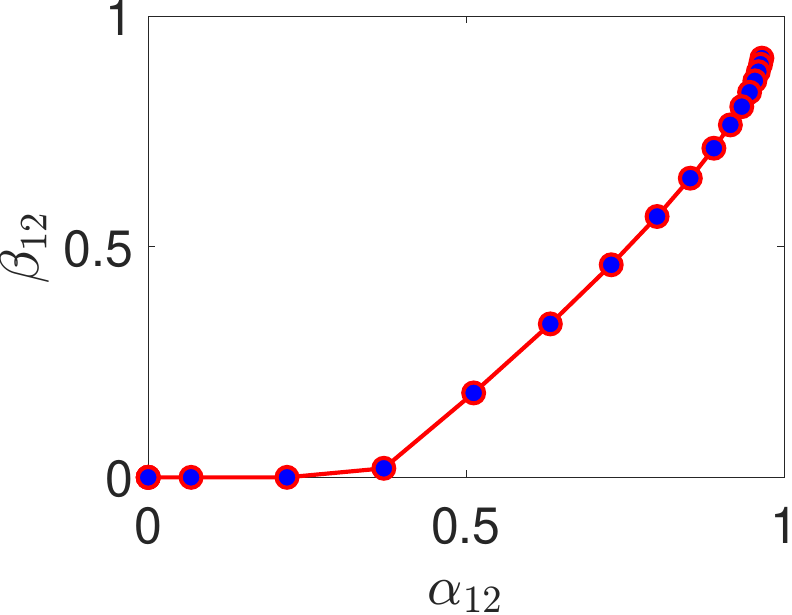}
        \subcaption[]{} \label{vector:h}
    \end{subfigure}
    \caption{Evolution of relative vectors, \(\alpha_{12}\), \(\beta_{12}\), and \(\xi_{12}\) for two UAVs in two scenarios. Scenario 1 is depicted in subfigures (a) to (d), and Scenario 2 in subfigures (e) to (h).}
    \label{fig:relative_vectors}
\end{figure}

We then incorporate the directionally aware penalty function \(v^U_{ij}\) defined as
  \begin{equation*}
   v^U_{ij} = \xi_{ij}v_{ij},
  \end{equation*}
into the HJB equation, replacing \(v_{ij}\) with \(v^U_{ij}\) in the minimization to account for the new penalty function. In the feedback control law \( \mathbf{u}_{\mathbf{z}}(t)\) in~\eqref{eq:feedback} now \(\frac{\partial V}{\partial \mathbf{z}}\) accounts for the directionally aware collision avoidance strategy. This approach provides a collision avoidance control that emphasizes avoiding collisions in the forward path as well as along an approach trajectory, promoting adaptive behavior, and enhancing collision-free formation tracking efficiency.

\begin{remark} \label{remark:role} The weighting factors \( \alpha_{ij},\beta_{ij},\xi_{ij} \) are introduced to dynamically adjust the sensitivity of the penalty function \( v_{ij} \) based on the relative motion of the agents. In the context of directionally aware collision avoidance, \( v_{ij} \) quantifies the collision risk between agents based on their proximity, and \( \alpha_{ij},\beta_{ij},\xi_{ij} \) modulate the impact of \( v_{ij} \) by accounting for the agents' forward paths and relative approach. This distinction ensures that the weighting factors do not alter the fundamental role of \( v_{ij} \) in the collision avoidance mechanism. Rather, they dynamically scale its effect based on directional considerations.  
\end{remark}

\begin{assumption} \label{assumption:constancy} (Constancy of weighting factors in derivatives of the penalty function)
Given the distinct roles of the weighting factors \( \alpha_{ij}, \beta_{ij}, \xi_{ij} \) in dynamically adjusting the sensitivity of the penalty function \( v_{ij} \), as stated in Remark~\ref{remark:role}, we assume that these factors are treated as constants when calculating the time and partial derivatives of \( v_{ij} \).
\end{assumption}

From their definitions, the boundedness of the weighting factors is evident (\( \alpha_{ij}, \beta_{ij}, \xi_{ij} \in [0, 1] \)). In Subsection~\ref{sec:collision-avoidance}, we demonstrate that the HJB equation, constructed using the time and partial derivatives of \( v_{ij} \), remains non-increasing within the collision reaction region \( \Gamma(\mathbf{p}) \). The inclusion of the bounded weighting factors as coefficients to \( v_{ij} \), in accordance with Remark~\ref{remark:role}, does not alter the directional behavior of the HJB equation nor increase its magnitude in \( \Gamma(\mathbf{p}) \). This is because the weighting factors do not amplify the penalty function beyond its original bounds, preventing any unbounded growth that could destabilize the HJB equation. Therefore, under Assumption~\ref{assumption:constancy}, the conditions outlined in Subsection~\ref{sec:collision-avoidance} are preserved.

\begin{remark} \label{remark:stability}
Assumption~\ref{assumption:constancy} simplifies the computation of the penalty function derivatives and improves computational efficiency, robustness, and smoothness of control responses. However, as demonstrated in Appendix~\ref{app:dir}, the control strategy in~\eqref{eq:feedback}, utilizing the forward-path-aware penalty function \( v_{ij}^F \), ensures bounded control inputs even without Assumption~\ref{assumption:constancy}, thereby guaranteeing stability. Analogous stability conditions can be derived for the approach-aware penalty function \( v_{ij}^A \) and the unified penalty function \( v_{ij}^U \), reinforcing the robustness of the directionally aware collision avoidance framework.
\end{remark}

\section{Simulation Results}\label{sec:sim}
In this section, we first consider a four-UAV ($N=4$) (re)formation problem to demonstrate the efficacy of the presented control scheme outlined in Fig.~\ref{fig:control-scheme}. In the initial configuration, the UAVs are positioned in a square formation in the $x-y$ plane from the top, as illustrated in Fig.~\ref{fig:formation}(\subref{fig:initial}). The formation shape is determined by the offset vectors $\mathbf{d}_{12}=[-4,-4,0]^\top\, \mathrm{m}$ and $\mathbf{d}_{13}=\mathbf{d}_{24}=[4,-4,0]^\top\, \mathrm{m}$, which characterize a diamond alignment in the $x-y$ plane, as depicted in Fig.~\ref{fig:formation}(\subref{fig:final}).

\begin{figure}[htbp]
	\centering
	\begin{subfigure}{0.3\linewidth}
		\includegraphics[width=\linewidth]{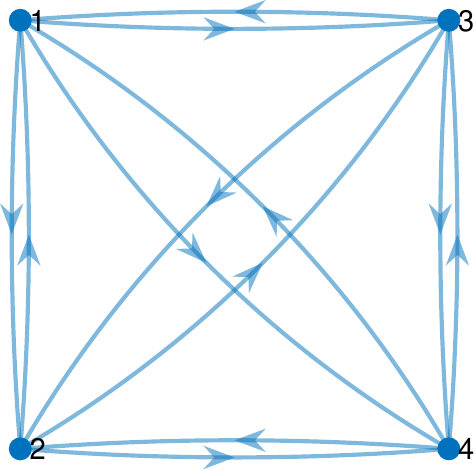}
		\caption{Communication graph}
		\label{fig:initial}
	\end{subfigure}
	\hspace{0.2\linewidth}  
	\begin{subfigure}{0.3\linewidth}
	  \includegraphics[width=\linewidth]{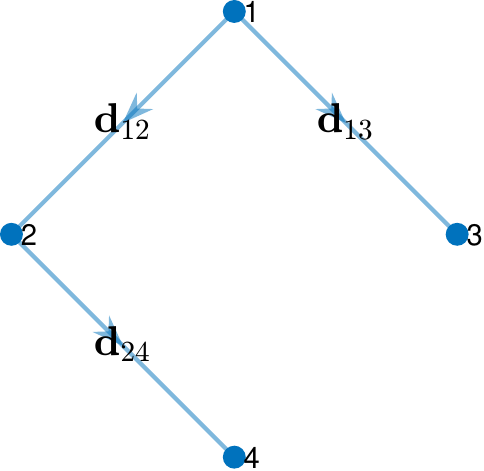}
		\caption{Desired formation graph}
		\label{fig:final}
    \end{subfigure}
\caption{The communication graph \(\mathcal{G}_c(\mathcal{V},\mathcal{E}_c)\) and the desired formation graph \(\mathcal{G}_f(\mathcal{V},\mathcal{E}_f)\) for a team of four UAVs. The graphs illustrate the spatial arrangement of the UAVs in the \(x\)--\(y\) plane, corresponding to their initial and desired formations, respectively.}
\label{fig:formation}
\end{figure}

The initial states of the UAVs in the flat coordinates are $\mathbf{p}_1(0)=[0,0,5]^\top\, \mathrm{m}$, $\mathbf{p}_2(0)=[5,0,0]^\top\, \mathrm{m}$, $\mathbf{p}_3(0)=[0,5,0]^\top\, \mathrm{m}$, $\mathbf{p}_4(0)=[5,5,0]^\top\, \mathrm{m}$, $\dot{\mathbf{p}}_i(0)=[0,0,0]^\top\, \mathrm{m/s}$, $\ddot{\mathbf{p}}_i(0)=[0,0,0]^\top\, \mathrm{m/s^2}$, and $\psi_i(0)=0\, \mathrm{rad}$, $\dot{\psi}_i(0)=0\, \mathrm{rad/s}$, $\ddot{\psi}_i(0)=0\, \mathrm{rad/s^2}$ for $i\in\{1,2,3,4\}$. The weighting parameters of the performance index (\ref{eq:PI}) are as $\mu_{12}=0.9$, $\mu_{13}=0.7$, $\mu_{14}=0.5$, $\mu_{23}=0.6$, $\mu_{24}=0.8$, $\mu_{34}=0.4$, $\omega_{ij}=1$ for all $(i,j)\in\mathcal{E}_f$, and $\gamma_i=1$ for all $i\in\{1,2,3,4\}$. The horizon length is $t_f=10\, \mathrm{s}$. For each UAV, its mass, distance from the rotor to the center of mass, gravity acceleration, and moment of inertia are respectively set as \(m=1.0\, \mathrm{kg}\), \(l=0.2\, \mathrm{m}\), \(g=9.81\, \mathrm{m/s^2}\), and \(\mathcal{I}_{xx}=\mathcal{I}_{xx}=\mathcal{I}_{xx}= 0.016\, \mathrm{kg.m^2}\).

Figs.~\ref{fig:planning}–\ref{fig:planning-history} show the optimal formation trajectories (from Theorem~\ref{theorem:sol-open}) and their time histories. Without collision avoidance, these trajectories demonstrate optimal control strategies for achieving the diamond formation within $t_f=10\, \mathrm{s}$. The decreasing control inputs \(\mathbf{u}_{\mathbf{r}}\) over time, as shown in the time histories in Figs.~\ref{fig:planning-history}, reflect the achievement of the optimal performance index, which includes a penalty on control effort in addition to the formation requirements. 

\begin{figure}[htbp]
	\centering
	\begin{subfigure}{0.49\linewidth}
		\includegraphics[width=\linewidth]{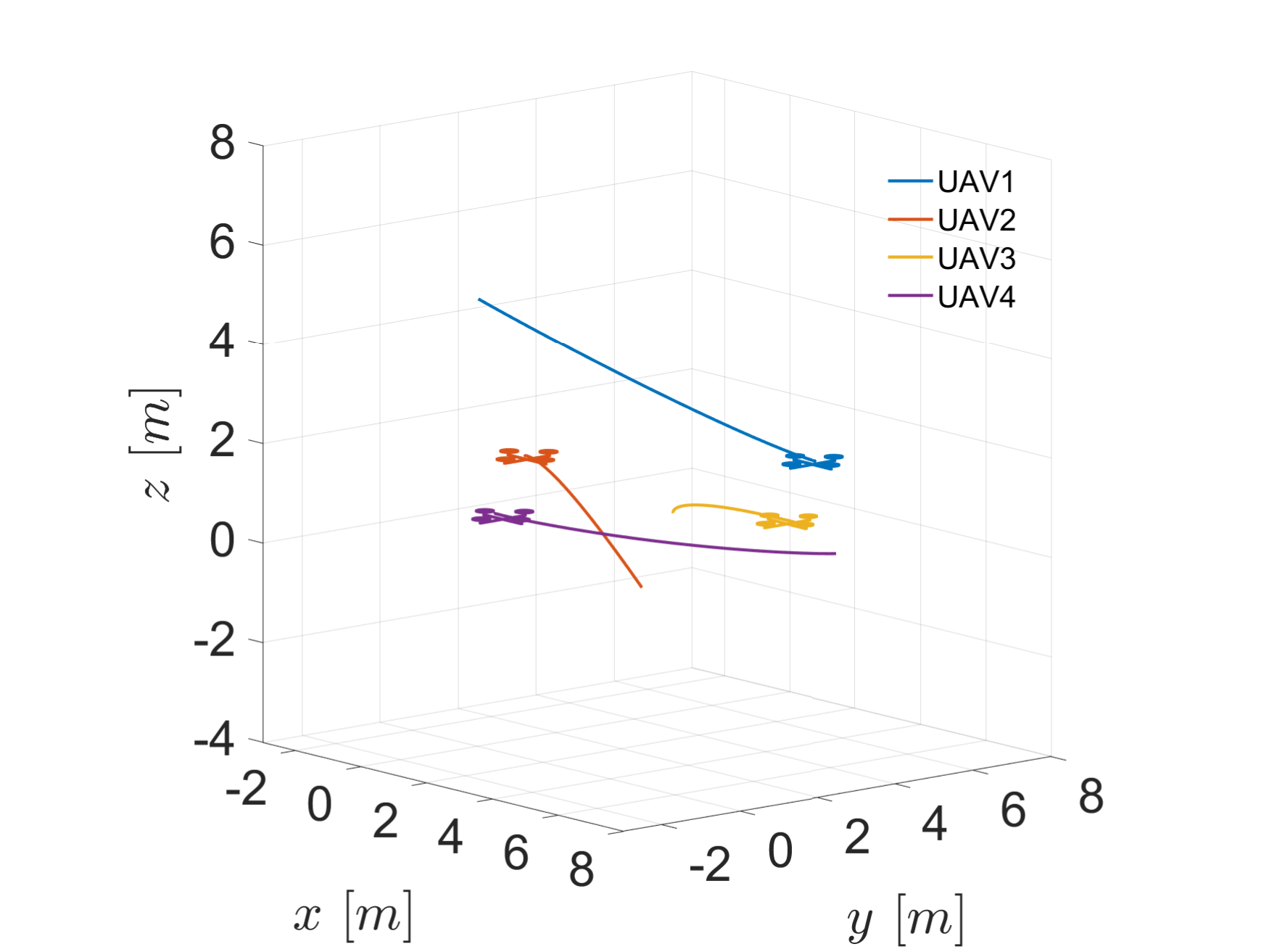}
		\caption{Optimal formation trajectories}
	\end{subfigure}
	\begin{subfigure}{0.49\linewidth}
	  \includegraphics[width=\linewidth]{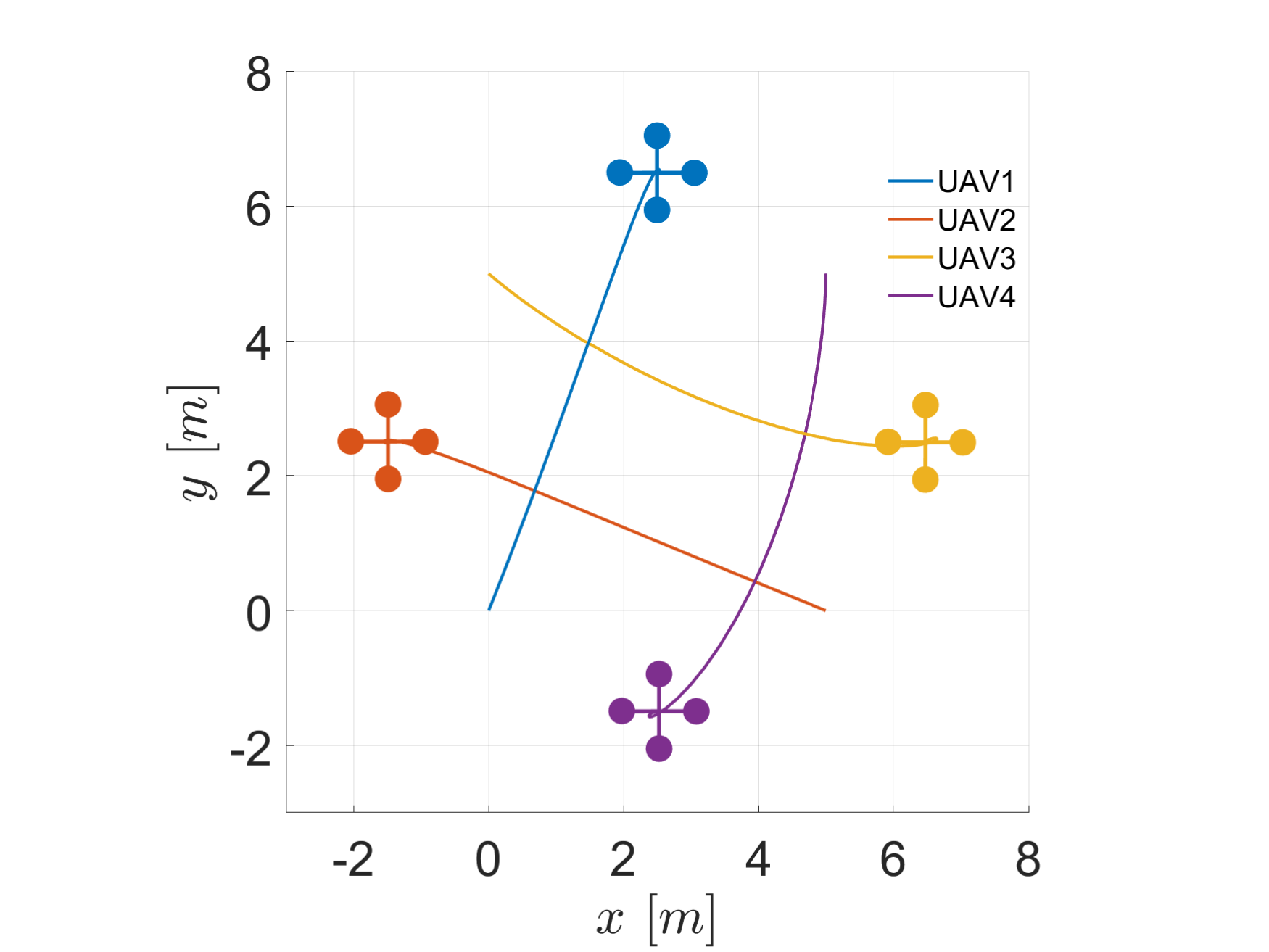}
		\caption{Top view of the trajectories}
    \end{subfigure}
\caption{Optimal formation trajectories of UAVs. Initially, the UAVs are placed such that from the top, they are seen in a square formation and then transition into a diamond formation in the $x-y$ plane.}
\label{fig:planning}
\end{figure}

\begin{figure}[htbp]
	\centering
	\begin{subfigure}{0.24\linewidth}
		\includegraphics[width=\linewidth]{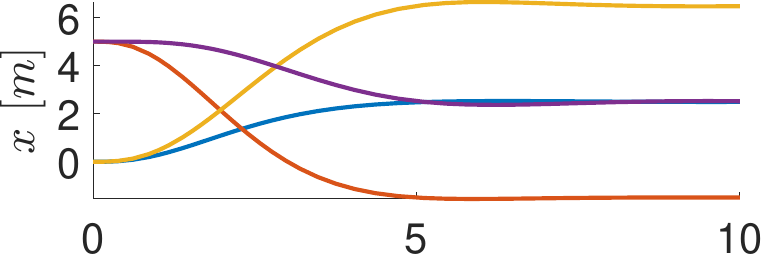}
        \includegraphics[width=\linewidth]{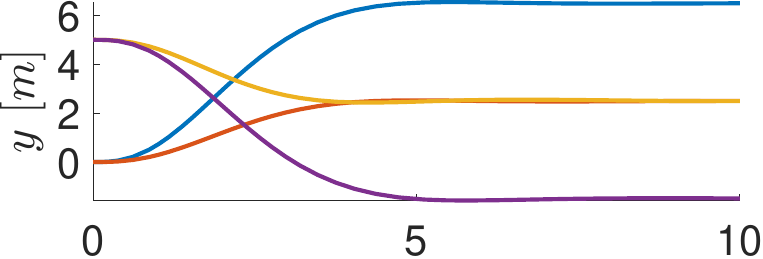}
        \includegraphics[width=\linewidth]{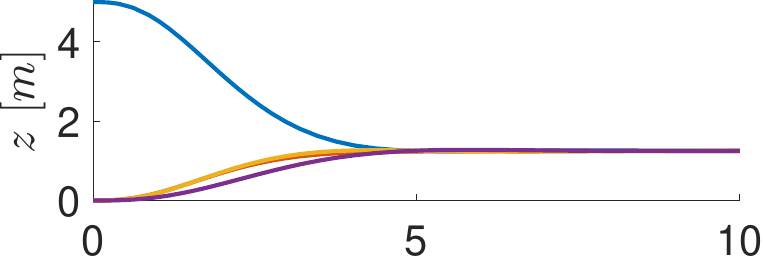}
        \includegraphics[width=\linewidth]{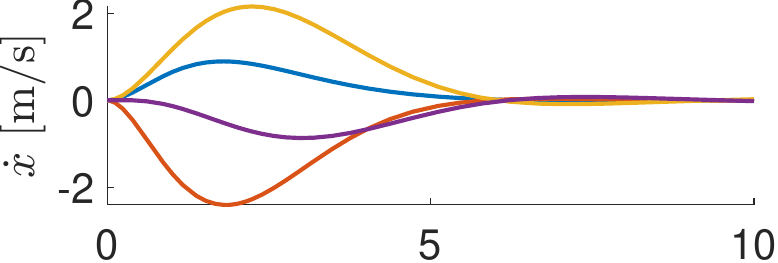}
        \includegraphics[width=\linewidth]{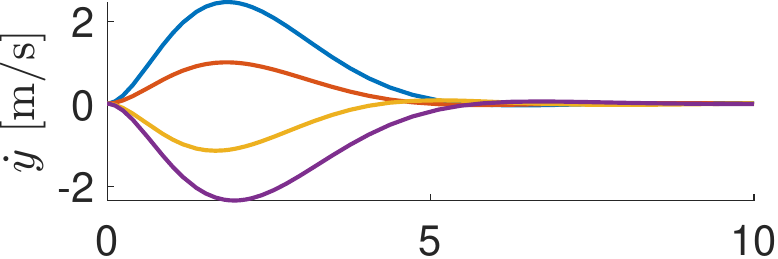}
        \includegraphics[width=\linewidth]{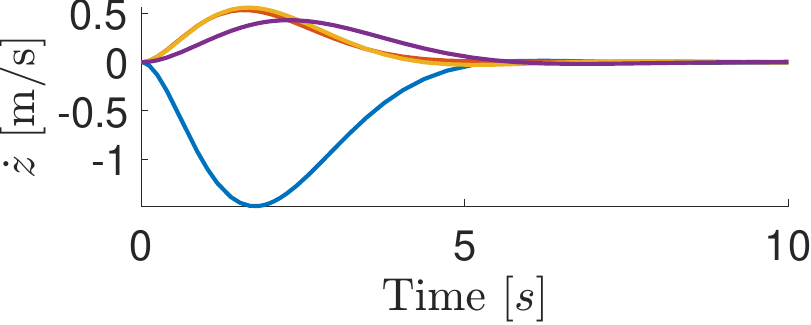}
		\caption{Positions and velocities}
	\end{subfigure}
    	\begin{subfigure}{0.24\linewidth}
		\includegraphics[width=\linewidth]{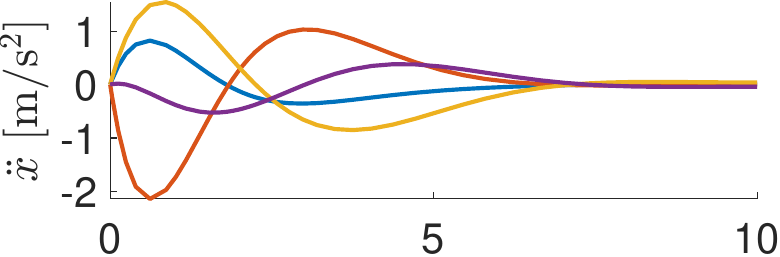}
        \includegraphics[width=\linewidth]{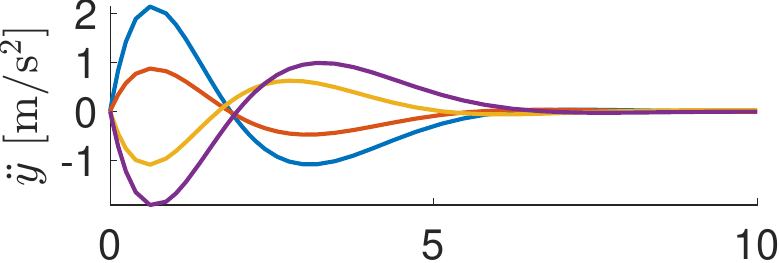}
        \includegraphics[width=\linewidth]{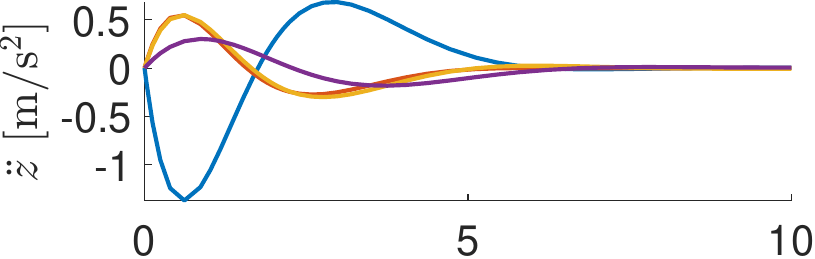}
        \includegraphics[width=\linewidth]{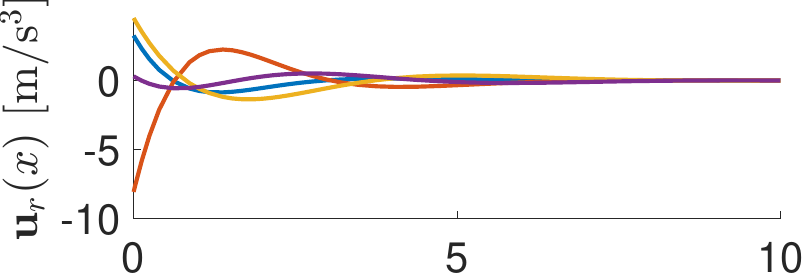}
        \includegraphics[width=\linewidth]{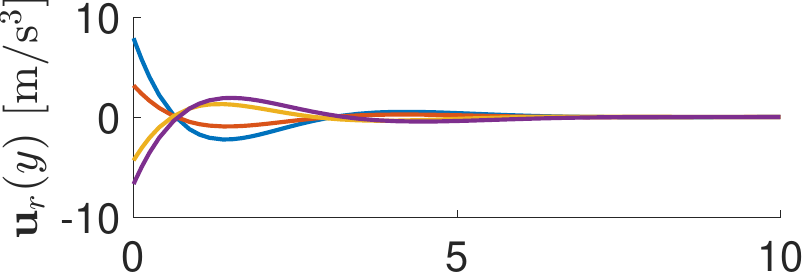}
        \includegraphics[width=\linewidth]{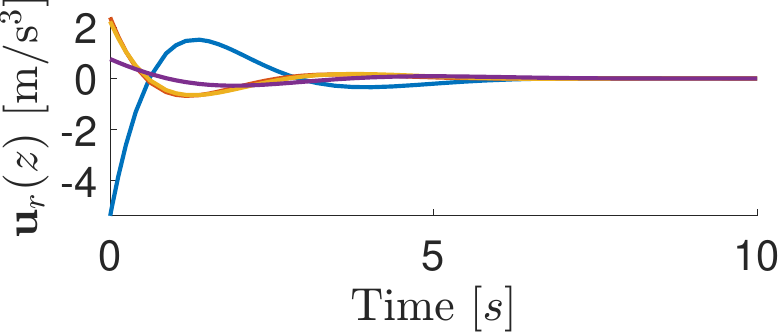}
		\caption{Accelerations and \(\mathbf{u}_{\mathbf{r}}\)}
	\end{subfigure}
    \begin{subfigure}{0.24\linewidth}
		\includegraphics[width=\linewidth]{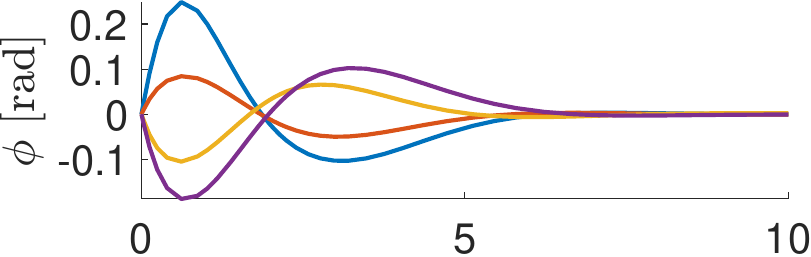}
        \includegraphics[width=\linewidth]{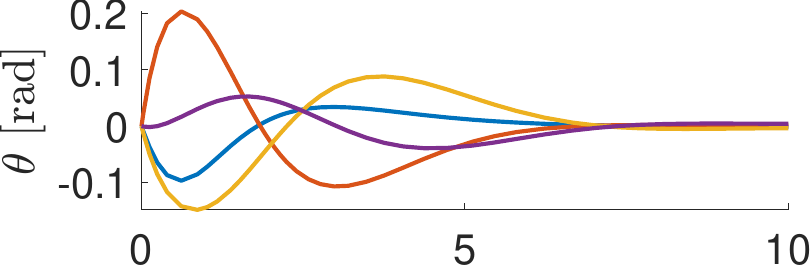}
        \includegraphics[width=\linewidth]{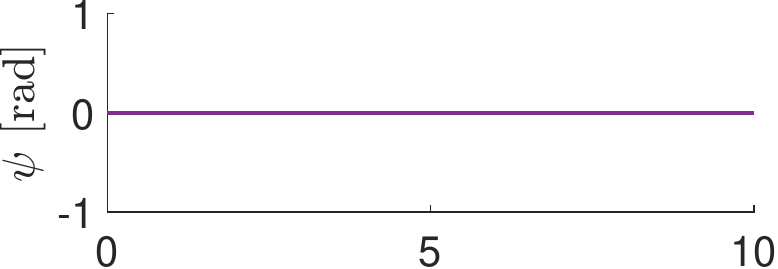}
        \includegraphics[width=\linewidth]{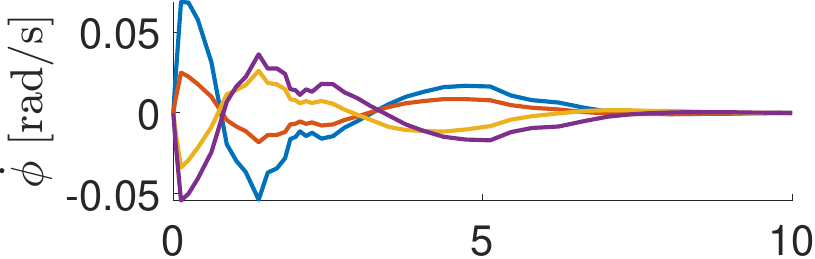}
        \includegraphics[width=\linewidth]{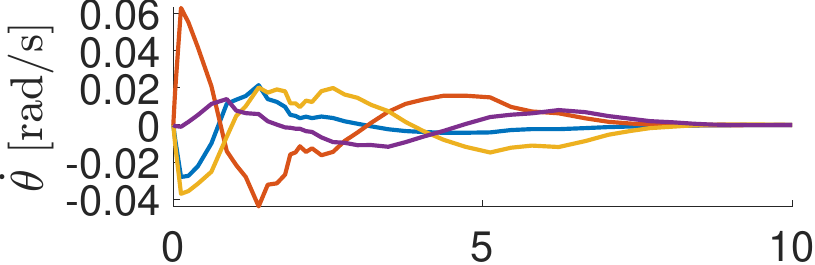}
        \includegraphics[width=\linewidth]{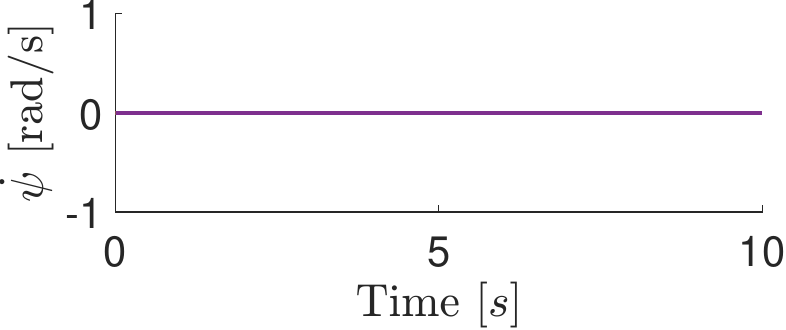}
		\caption{Attitudes with rates}
	\end{subfigure}
	\begin{subfigure}{0.45\linewidth}
	  \includegraphics[width=\linewidth]{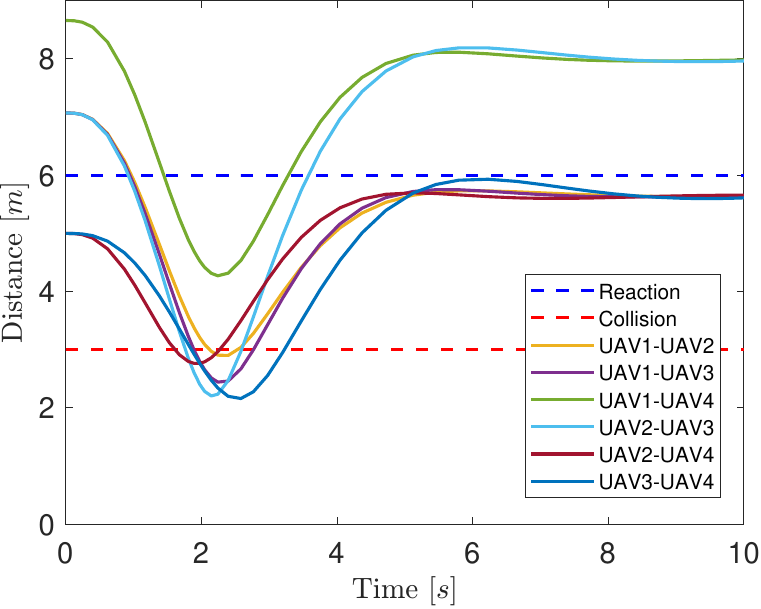}
		\caption{Euclidean distances}
		\label{fig:traj-dir}
    \end{subfigure}
\caption{Time histories of positions, velocities, accelerations, jerks (i.e., control inputs \(\mathbf{u}_{\mathbf{r}}\)), attitudes with rates, and Euclidean distances for all UAV pairs $(i,j) \in \mathcal{E}_f$. The blue dashed line represents the distance before entering the collision reaction region (i.e., $R_i+R_j$), and the red dashed line represents the distance before entering the collision region (i.e., $r_i+r_j$).}
\label{fig:planning-history}
\end{figure}

The formation trajectory tracking involves tracking the optimal formation trajectory planning signal while considering inter-UAV collisions, as outlined in Fig.~\ref{fig:control-scheme}. The avoidance region radius for all UAVs is defined as the same value $r_i=1.5\, \mathrm{m}$ for all $i={1,\cdots,4}$. This implies that a collision occurs for every UAV pair $(i,j)$ if they approach each other closer than $3\, \mathrm{m}$. We choose the radius of the collision reaction region $R_i=2r_i$ for all $i={1,\cdots,4}$.
The parameters of the tracking performance index (\ref{eq:PI-track-comp}) are set as $\zeta_{ij}=10\mu_{ij}$, $\delta_{ij}=10\omega_{ij}$, and $\eta_i=\gamma_i$ for all $i,j\in\{1,2,3,4\}$ and $(i,j)$ in the communication graph. Utilizing the collision penalty function from Subsection~\ref{sec:collision-avoidance}, the formation tracking trajectories and their corresponding time histories are shown in Figs.~\ref{fig:tracking}-\ref{fig:tracking-history}, respectively. It is apparent from both figures that the UAV team tracks the optimal trajectories and successfully achieves the desired formation at the terminal time. The relative distances between any UAV pair never decrease below $r_i+r_j$, even though some UAV pairs initially have a relative distance between $r_i+r_j$ and $R_i+R_j$. Despite the collision avoidance control strategy being applied from the outset, it separates the UAV pairs before they reach the collision region.

\begin{figure}[htbp]
	\centering
	\begin{subfigure}{0.49\linewidth}
		\includegraphics[width=\linewidth]{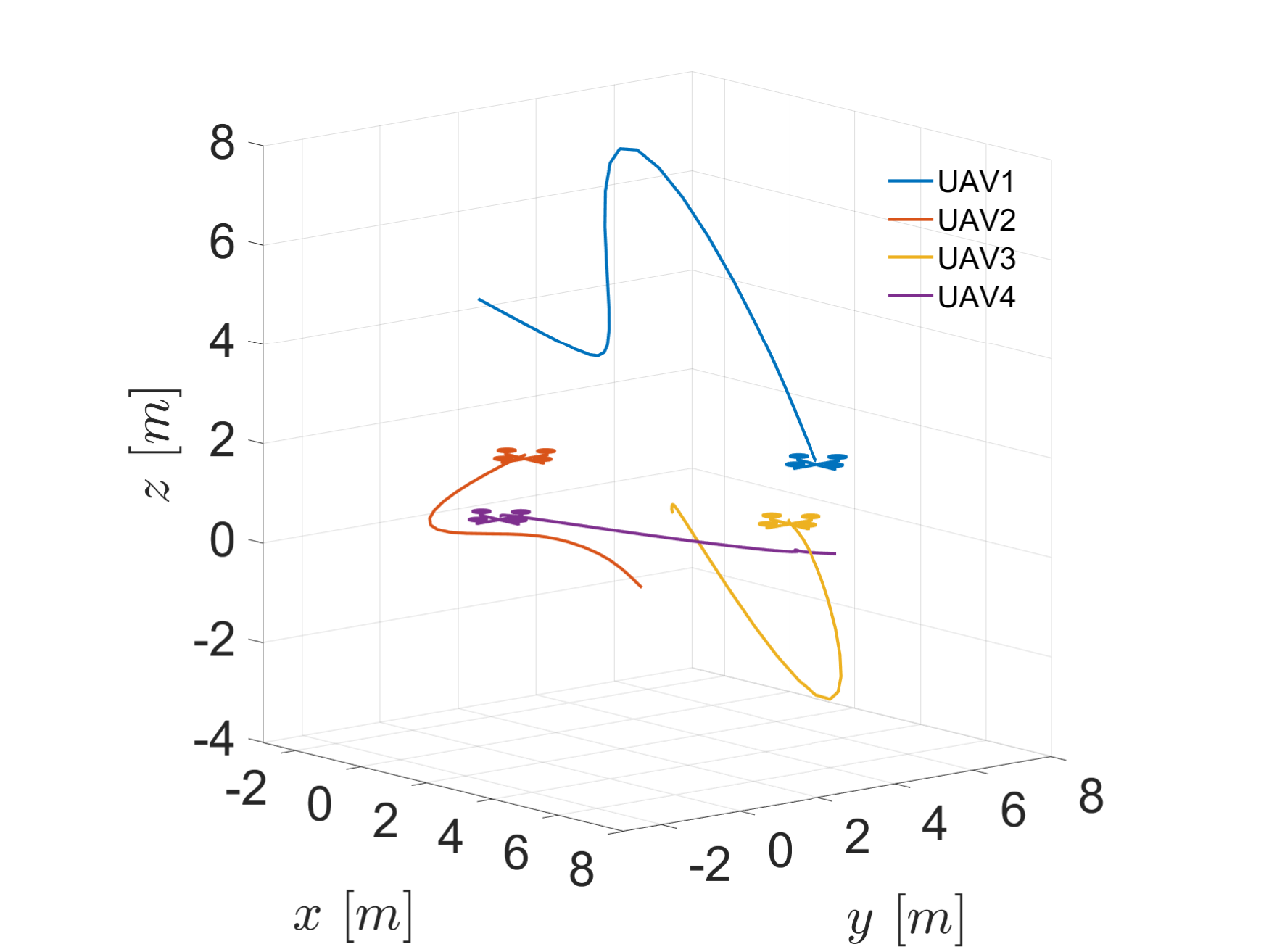}
		\caption{Formation tracking trajectories}
	\end{subfigure}
	\begin{subfigure}{0.49\linewidth}
	  \includegraphics[width=\linewidth]{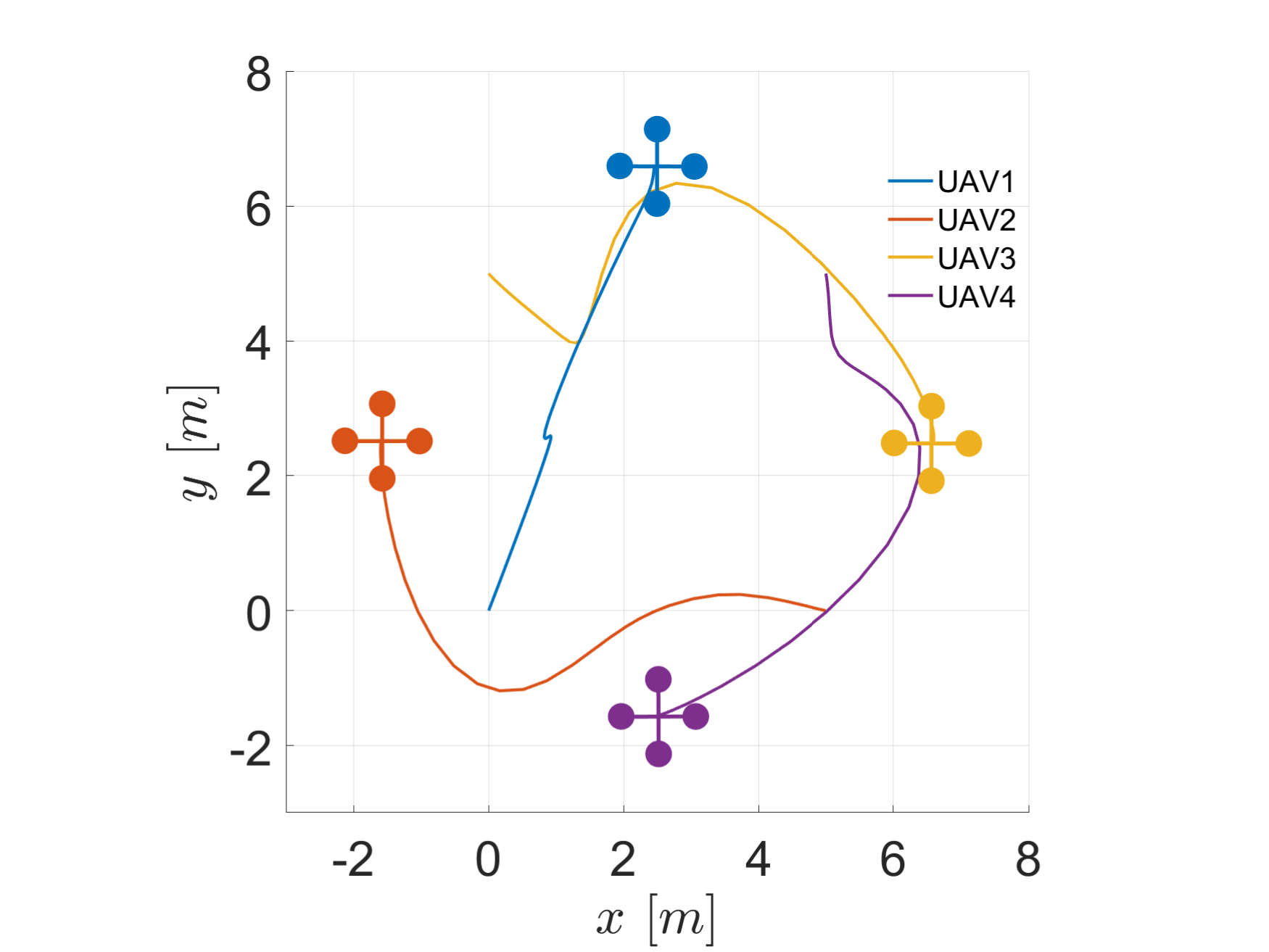}
		\caption{Top view of the tracking trajectories}
    \end{subfigure}
\caption{Formation trajectory tracking, which involves following the optimal formation trajectory while accounting for inter-UAV collisions.}
\label{fig:tracking}
\end{figure}

\begin{figure}[htbp]
	\centering
	\begin{subfigure}{0.24\linewidth}
		\includegraphics[width=\linewidth]{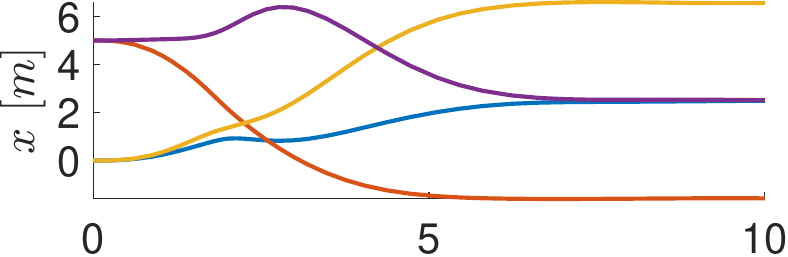}
        \includegraphics[width=\linewidth]{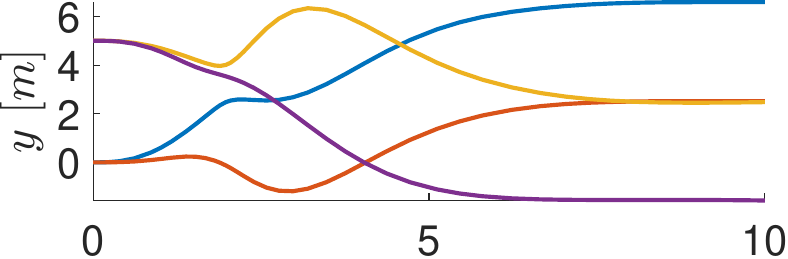}
        \includegraphics[width=\linewidth]{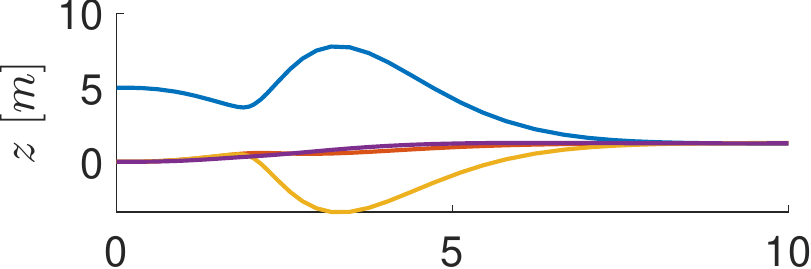}
        \includegraphics[width=\linewidth]{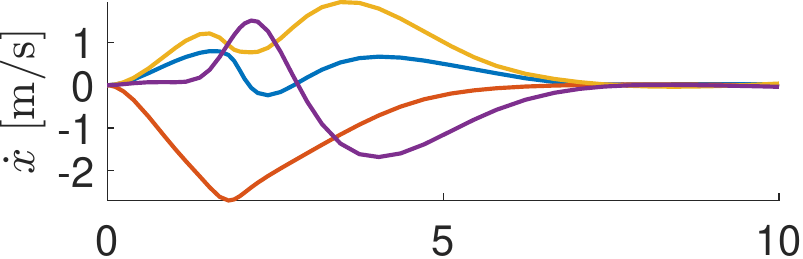}
        \includegraphics[width=\linewidth]{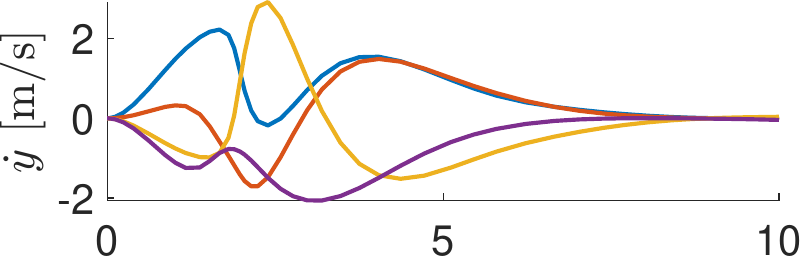}
        \includegraphics[width=\linewidth]{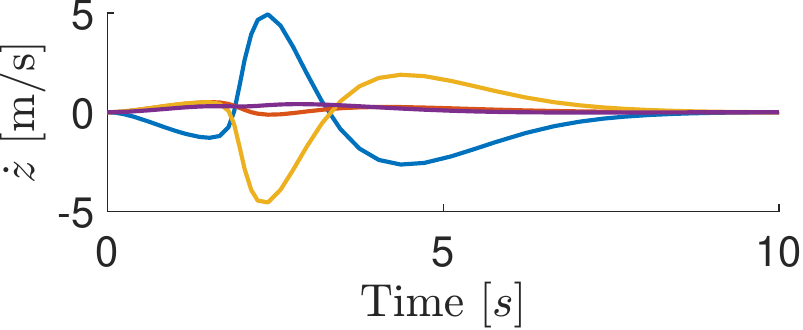}
		\caption{Positions and velocities}
	\end{subfigure}
    \begin{subfigure}{0.24\linewidth}
		\includegraphics[width=\linewidth]{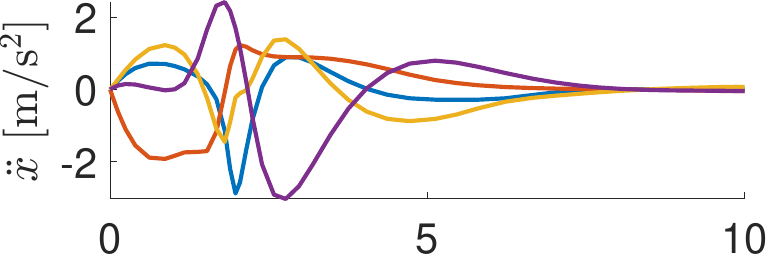}
        \includegraphics[width=\linewidth]{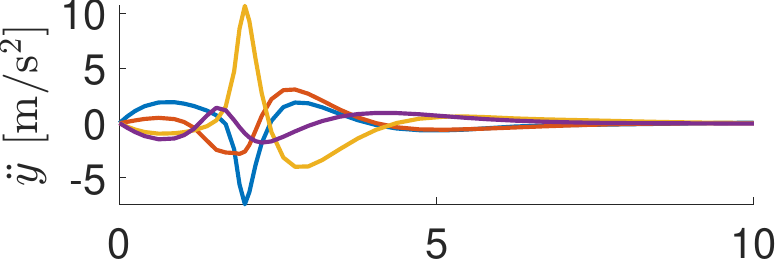}
        \includegraphics[width=\linewidth]{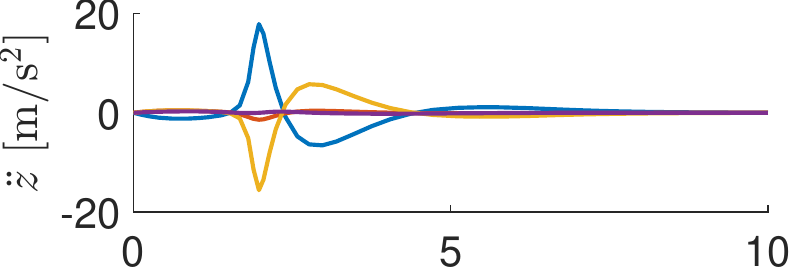}
        \includegraphics[width=\linewidth]{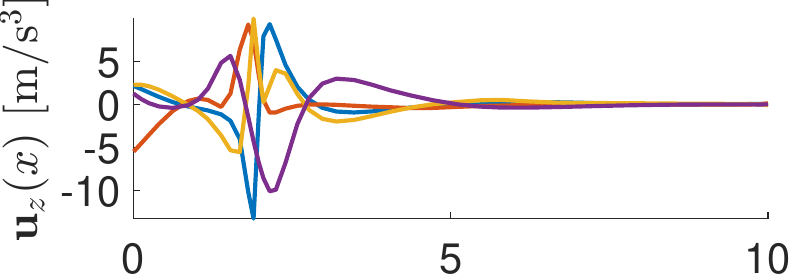}
        \includegraphics[width=\linewidth]{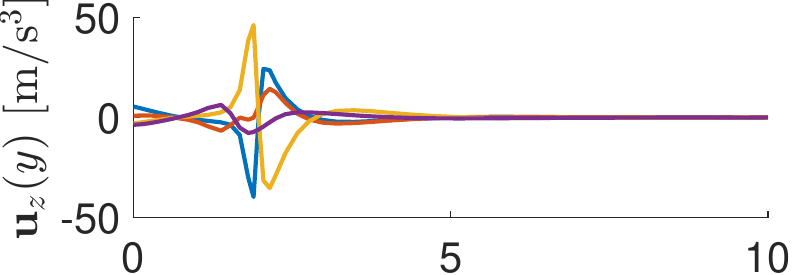}
        \includegraphics[width=\linewidth]{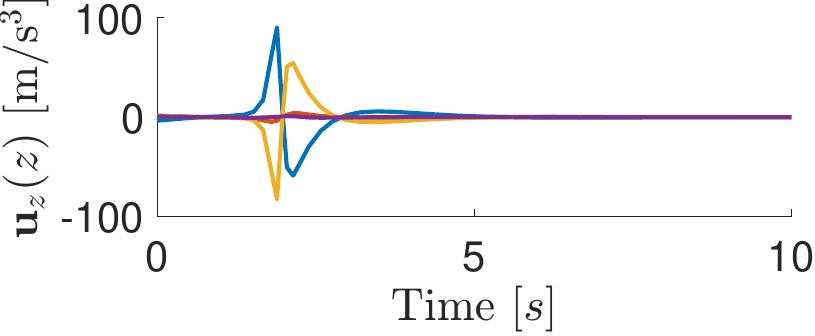}
		\caption{Accelerations and \(\mathbf{u}_{\mathbf{z}}\)}
	\end{subfigure}
    \begin{subfigure}{0.24\linewidth}
		\includegraphics[width=\linewidth]{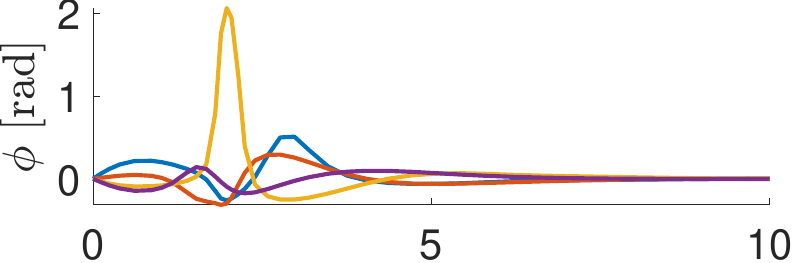}
        \includegraphics[width=\linewidth]{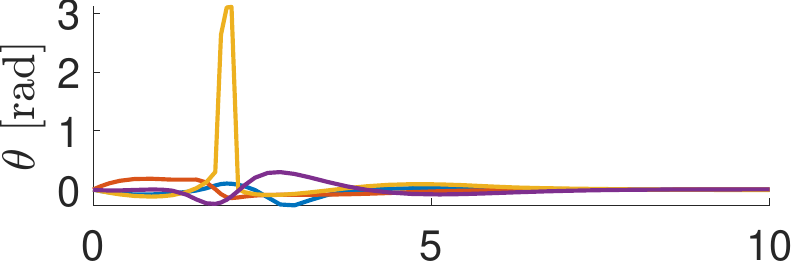}
        \includegraphics[width=\linewidth]{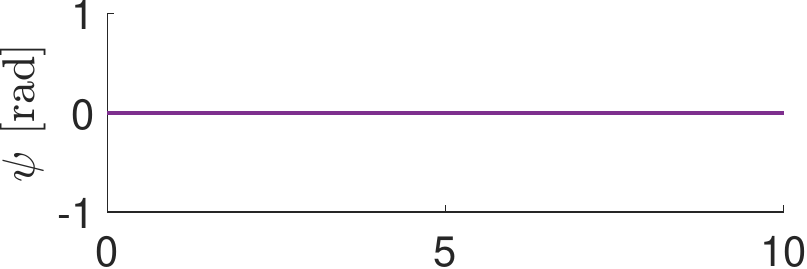}
        \includegraphics[width=\linewidth]{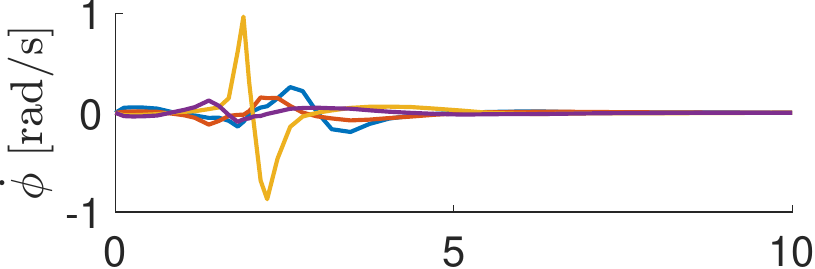}
        \includegraphics[width=\linewidth]{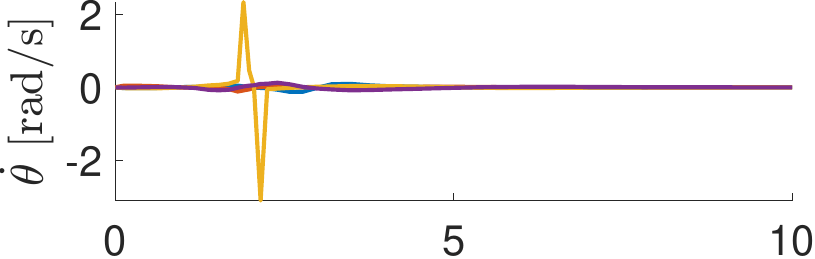}
        \includegraphics[width=\linewidth]{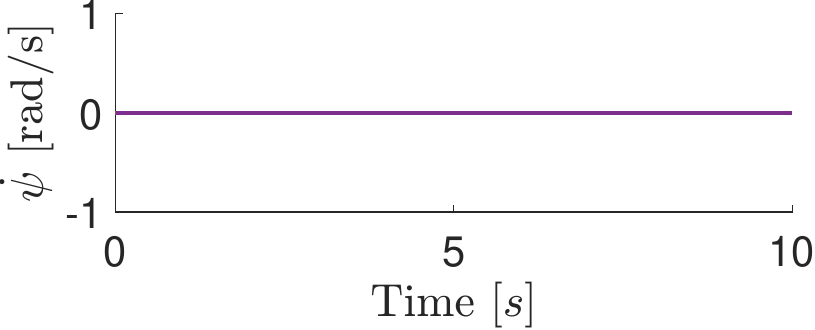}
		\caption{Attitudes with rates}
	\end{subfigure}
	\begin{subfigure}{0.45\linewidth}
	  \includegraphics[width=\linewidth]{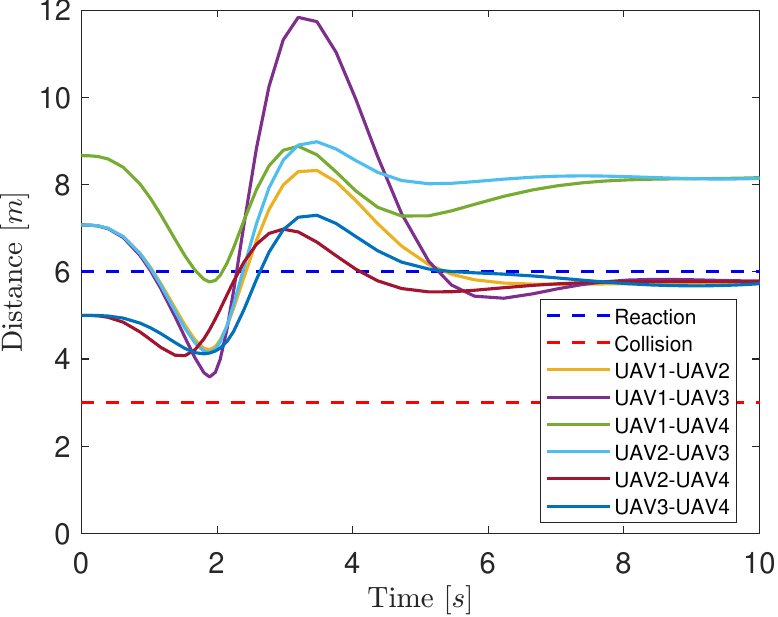}
		\caption{Euclidean distances}
    \end{subfigure}
\caption{Time histories of positions, velocities, accelerations, jerks (\(\mathbf{u}_{\mathbf{z}}\)), attitudes with rates, and Euclidean distances for formation tracking trajectories.}
\label{fig:tracking-history}
\end{figure}

The formation trajectories and their corresponding time histories using the directionally aware collision avoidance tracking strategy are shown in Fig.~\ref{fig:aware}-\ref{fig:aware-history}, respectively. Compared to the non-directionally aware approach (Figs.~\ref{fig:tracking}–\ref{fig:tracking-history}), the directionally aware strategy minimizes deviations from optimal trajectories while avoiding collisions. The results highlight the effectiveness of the proposed formation trajectory planning and tracking, demonstrating successful collision-free multi-UAV formation control. Notably, UAVs employing the directionally aware collision avoidance strategy require minimal deviation from their optimal paths to ensure collision avoidance. Thus, the directionally aware collision avoidance approach proves more efficient by adaptively prioritizing collision avoidance in both the forward path and the relative approach for each UAV.

\begin{figure}[htbp]
	\centering
	\begin{subfigure}{0.49\linewidth}
		\includegraphics[width=\linewidth]{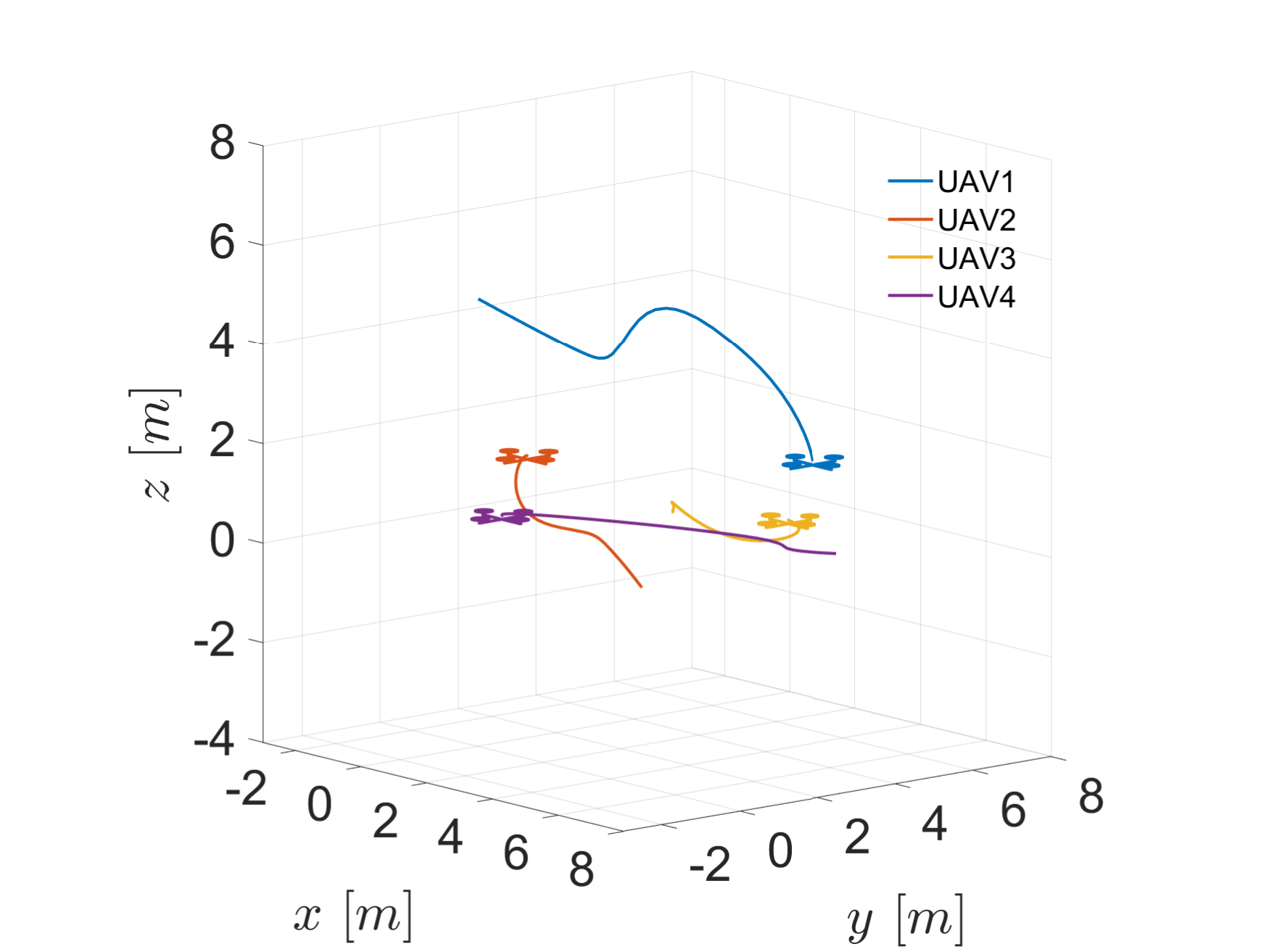}
		\caption{Directionally aware formation tracking trajectories}
	\end{subfigure}
	\begin{subfigure}{0.49\linewidth}
	  \includegraphics[width=\linewidth]{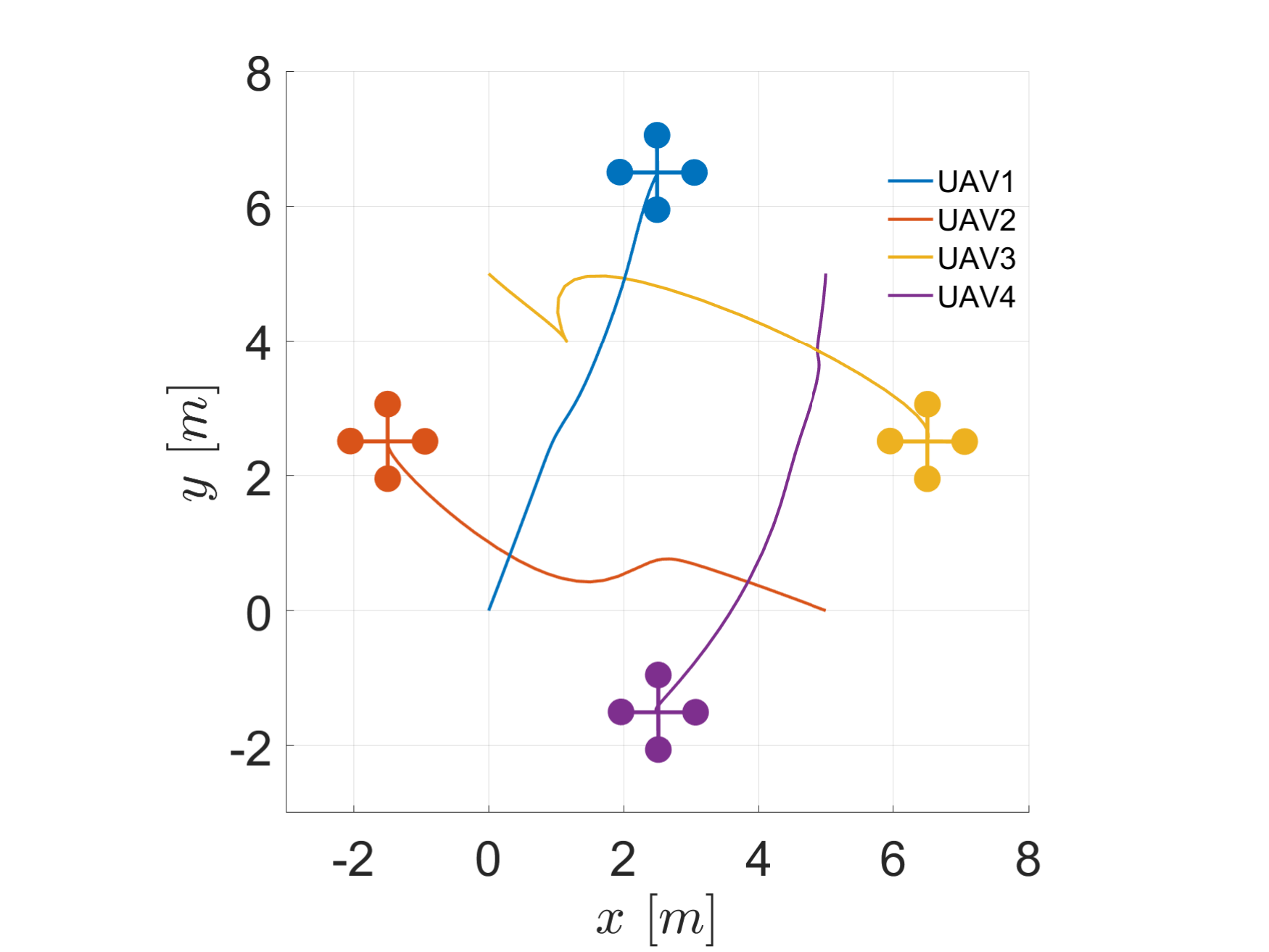}
		\caption{Top view of the tracking trajectories}
    \end{subfigure}
\caption{Formation trajectory tracking utilizing the directionally aware collision avoidance strategy.}
\label{fig:aware}
\end{figure}

\begin{figure}[htbp]
	\centering
	\begin{subfigure}{0.24\linewidth}
		\includegraphics[width=\linewidth]{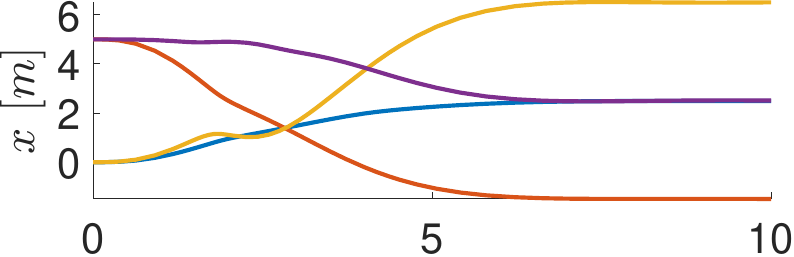}
        \includegraphics[width=\linewidth]{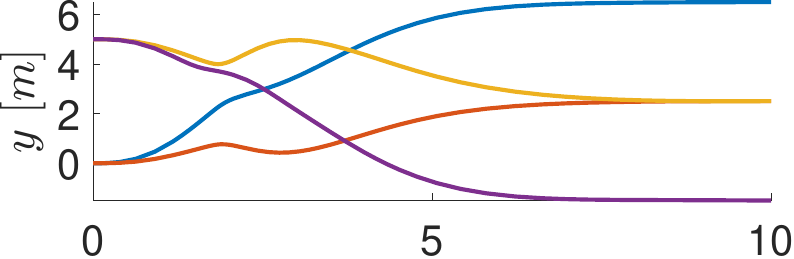}
        \includegraphics[width=\linewidth]{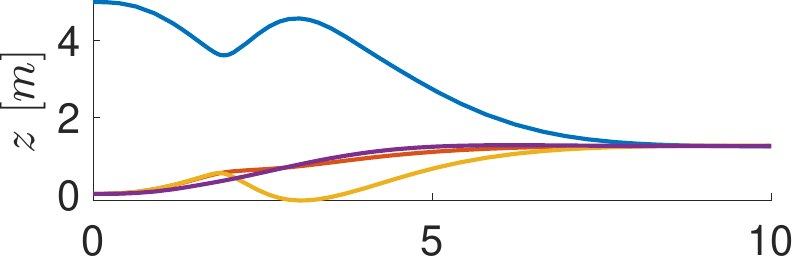}
        \includegraphics[width=\linewidth]{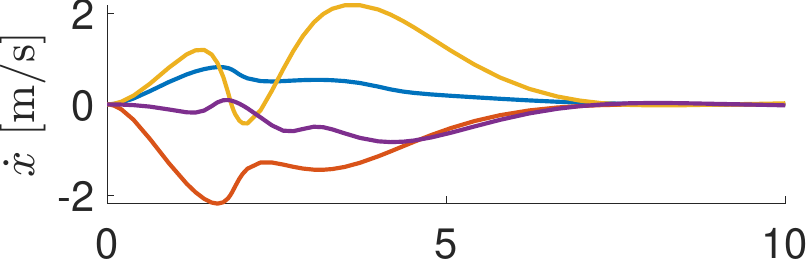}
        \includegraphics[width=\linewidth]{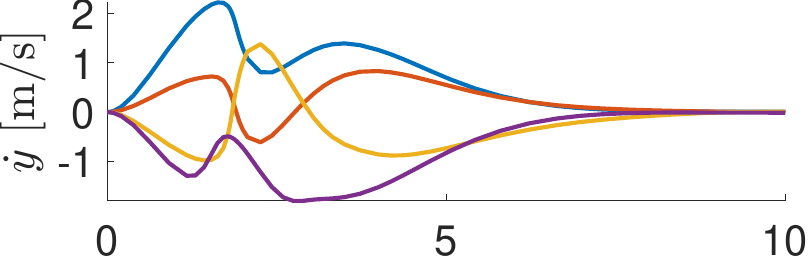}
        \includegraphics[width=\linewidth]{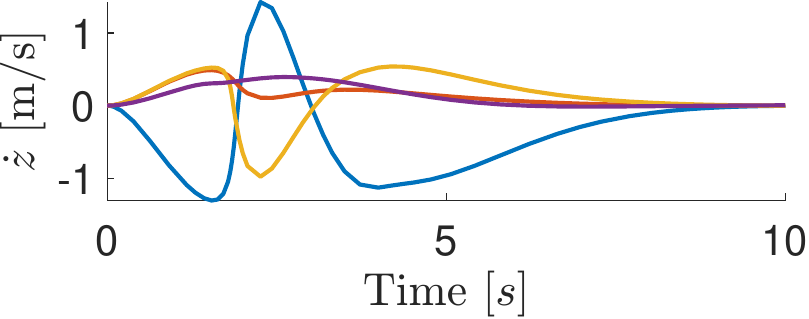}
		\caption{Positions and velocities}
	\end{subfigure}
    \begin{subfigure}{0.24\linewidth}
		\includegraphics[width=\linewidth]{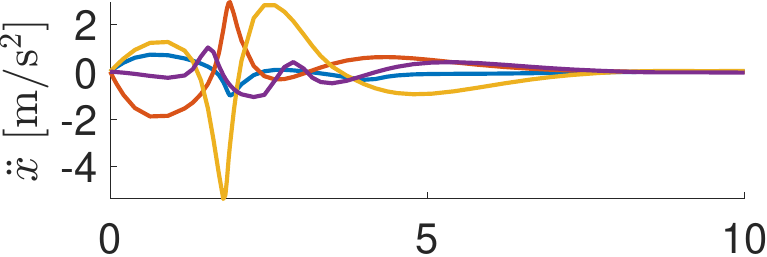}
        \includegraphics[width=\linewidth]{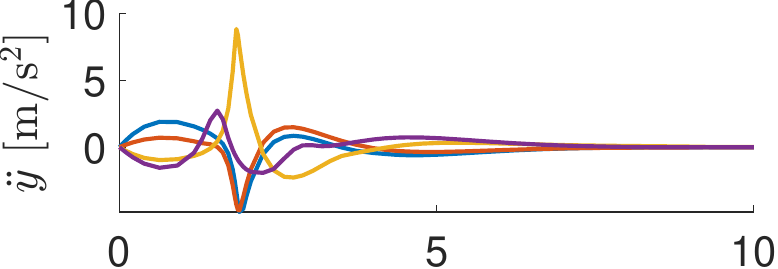}
        \includegraphics[width=\linewidth]{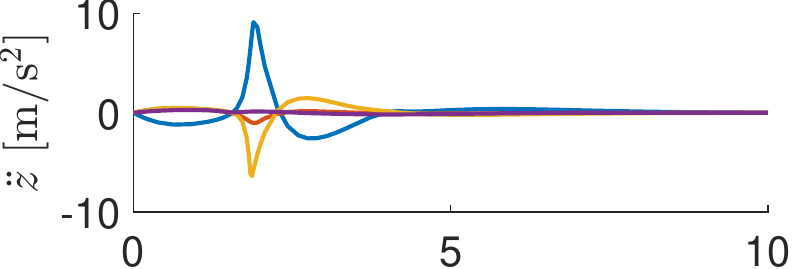}
        \includegraphics[width=\linewidth]{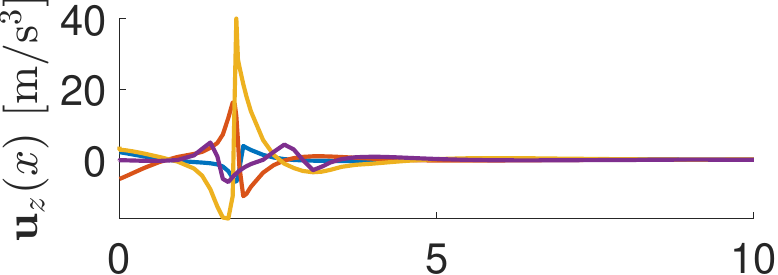}
        \includegraphics[width=\linewidth]{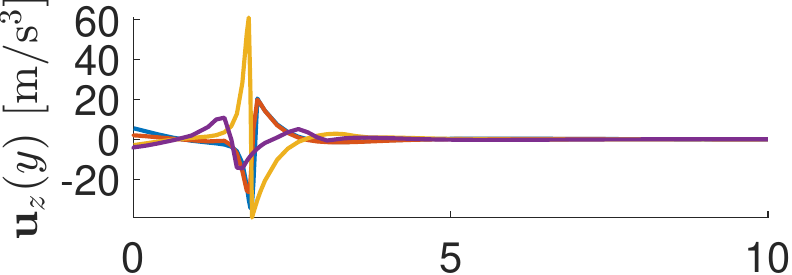}
        \includegraphics[width=\linewidth]{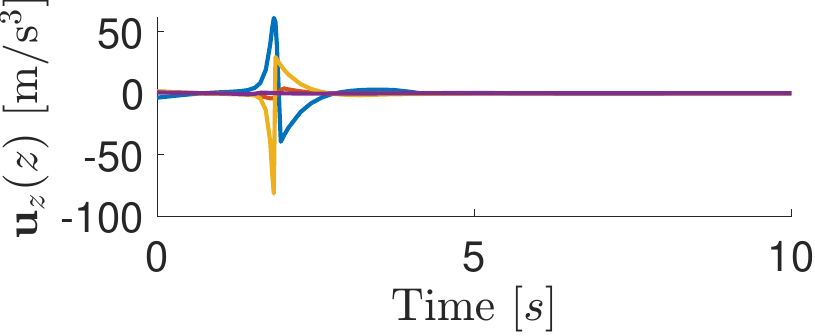}
		\caption{Accelerations and \(\mathbf{u}_{\mathbf{z}}\)}
	\end{subfigure}
    \begin{subfigure}{0.24\linewidth}
		\includegraphics[width=\linewidth]{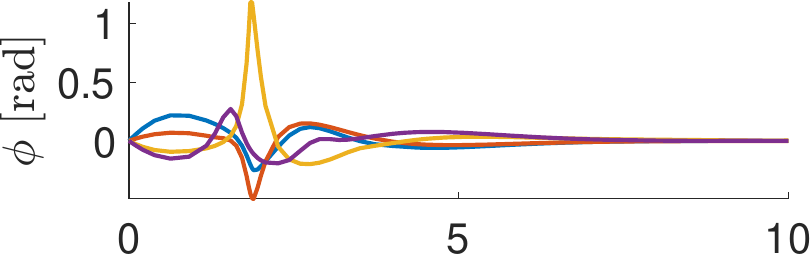}
        \includegraphics[width=\linewidth]{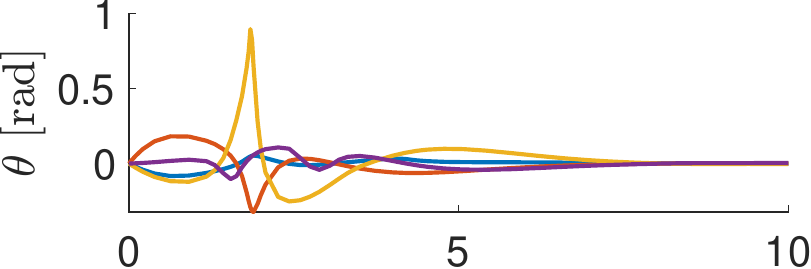}
        \includegraphics[width=\linewidth]{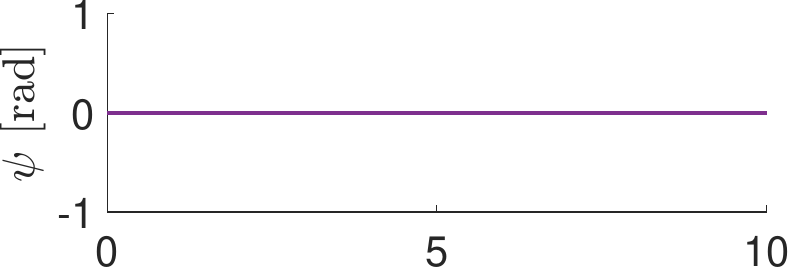}
        \includegraphics[width=\linewidth]{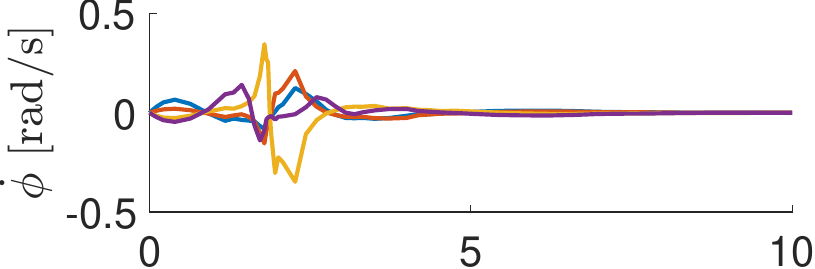}
        \includegraphics[width=\linewidth]{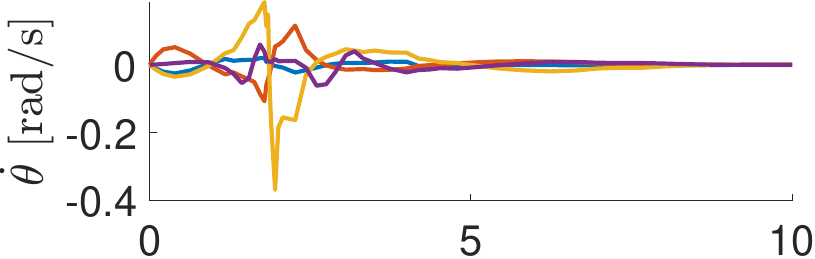}
        \includegraphics[width=\linewidth]{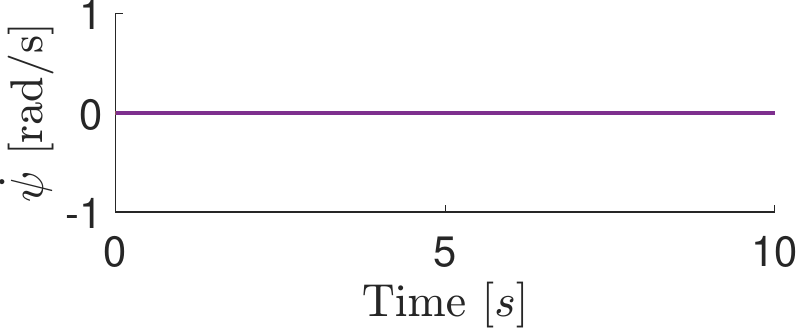}
		\caption{Attitudes with rates}
	\end{subfigure}
	\begin{subfigure}{0.45\linewidth}
	  \includegraphics[width=\linewidth]{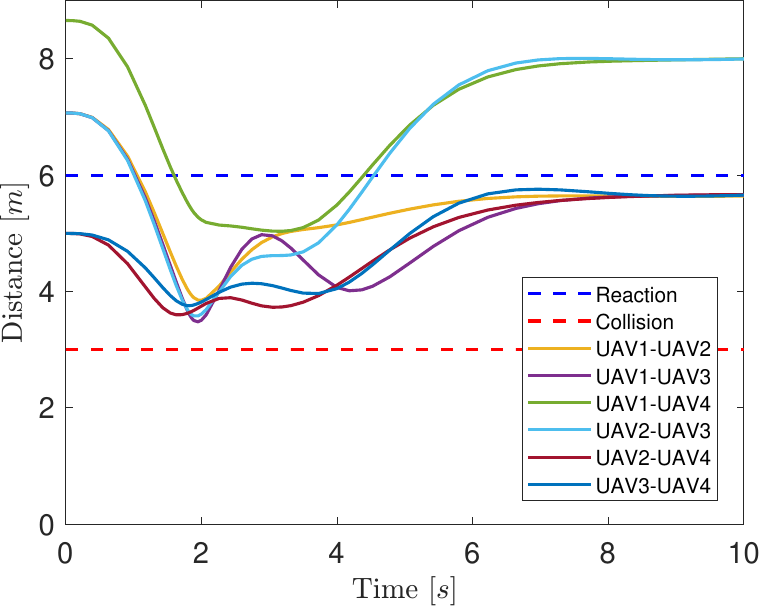}
		\caption{Euclidean distances}
    \end{subfigure}
    \caption{Time histories of positions, velocities, accelerations, jerks (\(\mathbf{u}_{\mathbf{z}}\)), attitudes with rates, and Euclidean distances for formation tracking trajectories utilizing the directionally aware collision avoidance strategy.}
\label{fig:aware-history}
\end{figure}

To further assess the performance of the proposed UAV formation control scheme, we extend our analysis to a larger group of seven UAVs (\(N=7\)). Initially, the UAVs are arranged in a cubic configuration, as shown in Fig.~\ref{fig:formation7}(\subref{fig:initial7}). The desired formation shape is defined by the offset vectors \(\mathbf{d}_{12}=\mathbf{d}_{24}=\mathbf{d}_{46}=[-8,8,0]^\top\, \mathrm{m}\) and \(\mathbf{d}_{13}=\mathbf{d}_{35}=\mathbf{d}_{57}=[8,-8,0]^\top\, \mathrm{m}\), forming a V-shape in the \(x\)-\(y\) plane, as illustrated in Fig.~\ref{fig:formation7}(\subref{fig:final7}). Initially, \(\dot{\mathbf{p}}_i(0) = [0,0,0]^\top\, \mathrm{m/s}\), \(\ddot{\mathbf{p}}_i(0) = [0,0,0]^\top\, \mathrm{m/s^2}\), and \(\psi_i(0) = 0\, \mathrm{rad}\), \(\dot{\psi}_i(0) = 0\, \mathrm{rad/s}\), and \(\ddot{\psi}_i(0) = 0\, \mathrm{rad/s^2}\) for \(i \in \{1, \dots, 7\}\). The weighting parameters of the performance index (\ref{eq:PI}) are assigned as \(\mu_{12} = 0.9\), \(\mu_{13} = 0.7\), \(\mu_{24} = 0.8\), \(\mu_{35} = 0.5\), \(\mu_{46} = 0.7\), \(\mu_{57} = 0.6\), with \(\omega_{ij} = 1\) and \(\gamma_i = 1\) for all \(i\). All other parameters remain consistent with the first example.

\begin{figure}[htbp]
	\centering
	\begin{subfigure}{0.5\linewidth}
		\includegraphics[width=\linewidth]{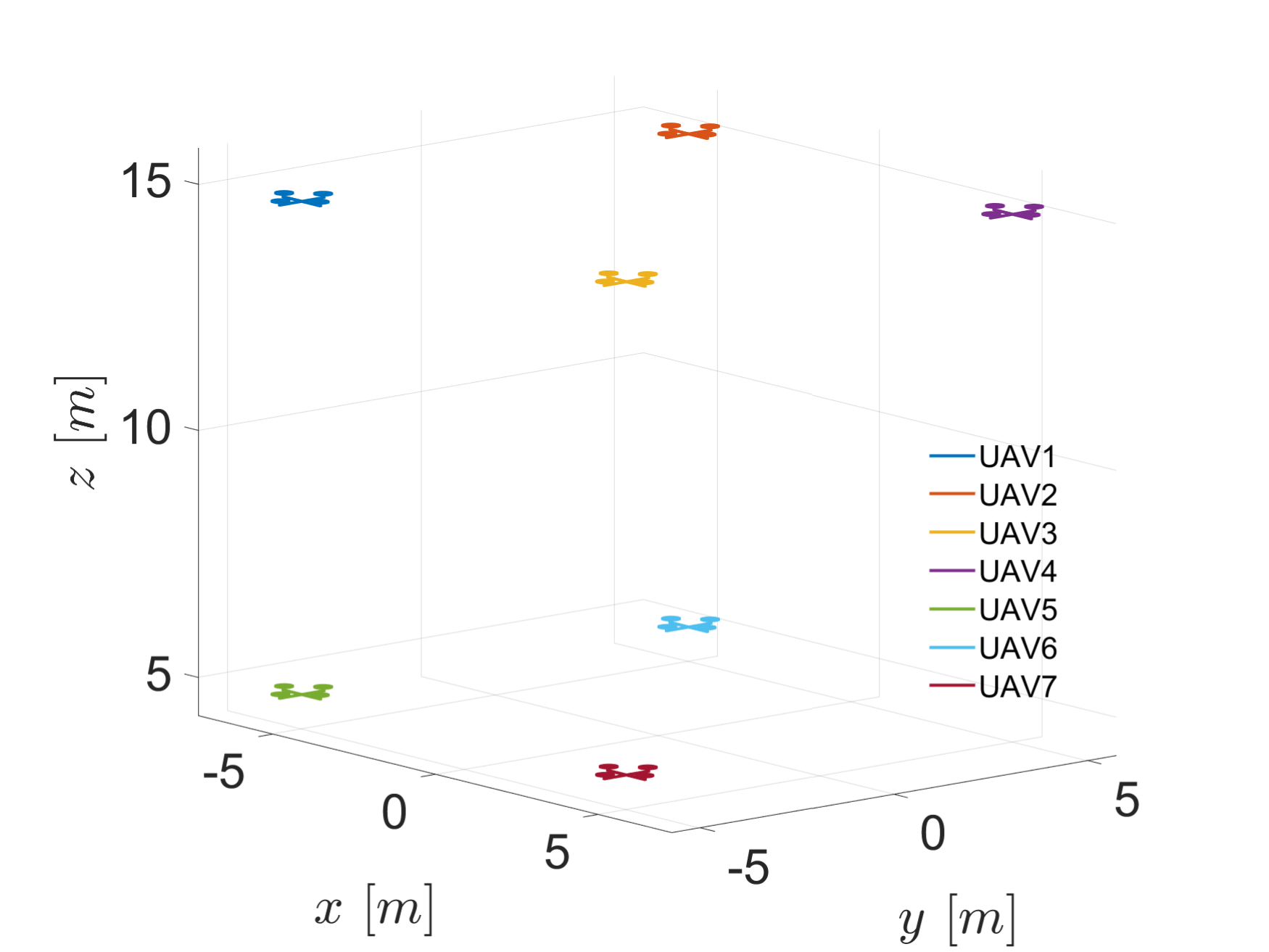}
		\caption{Initial UAV placement}
		\label{fig:initial7}
	\end{subfigure}
	\hspace{0.1\linewidth}  
	\begin{subfigure}{0.35\linewidth}
	  \includegraphics[width=0.5\linewidth]{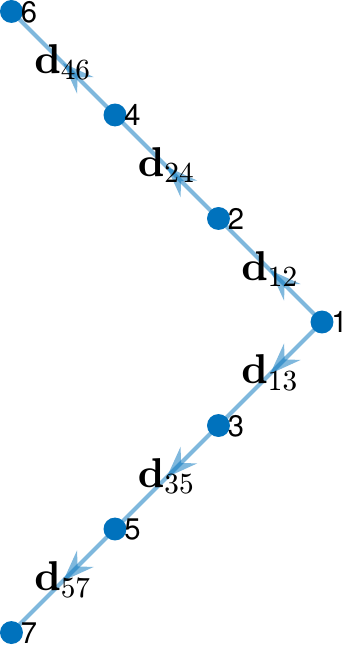}
		\caption{Desired formation graph}
		\label{fig:final7}
    \end{subfigure}
\caption{Initial configuration and desired V-shape formation graph for a team of seven UAVs.}
\label{fig:formation7}
\end{figure}

The optimal formation trajectories and their corresponding time histories are depicted in Figs.~\ref{fig:planning7}-\ref{fig:planning-history7}. The relative distances between UAV pairs occasionally fall slightly below \( r_i + r_j \), indicating a collision situation. The formation tracking trajectories and their time histories are presented in Figs.~\ref{fig:tracking7}-\ref{fig:tracking-history7}. While the UAV team achieves collision-free formation tracking (no violations of $r_i + r_j$), the control input $\mathbf{u}_{\mathbf{z}}$ shows abrupt jumps (e.g., at $t=2.5\, \mathrm{s}$), which are infeasible for physical UAVs due to actuation limits. The formation trajectories and their time histories under the directionally aware collision avoidance tracking strategy are shown in Figs.~\ref{fig:aware7}-\ref{fig:aware-history7}. Similarly to the first example, this strategy ensures that no UAV pair violates the \( r_i + r_j \) constraint while allowing only minimal deviations from the optimal trajectory. These results further validate the effectiveness of the proposed formation trajectory planning and tracking framework, particularly demonstrating the efficiency of the directionally aware collision avoidance approach.

\begin{figure}[htbp]
	\centering
	\begin{subfigure}{0.49\linewidth}
		\includegraphics[width=\linewidth]{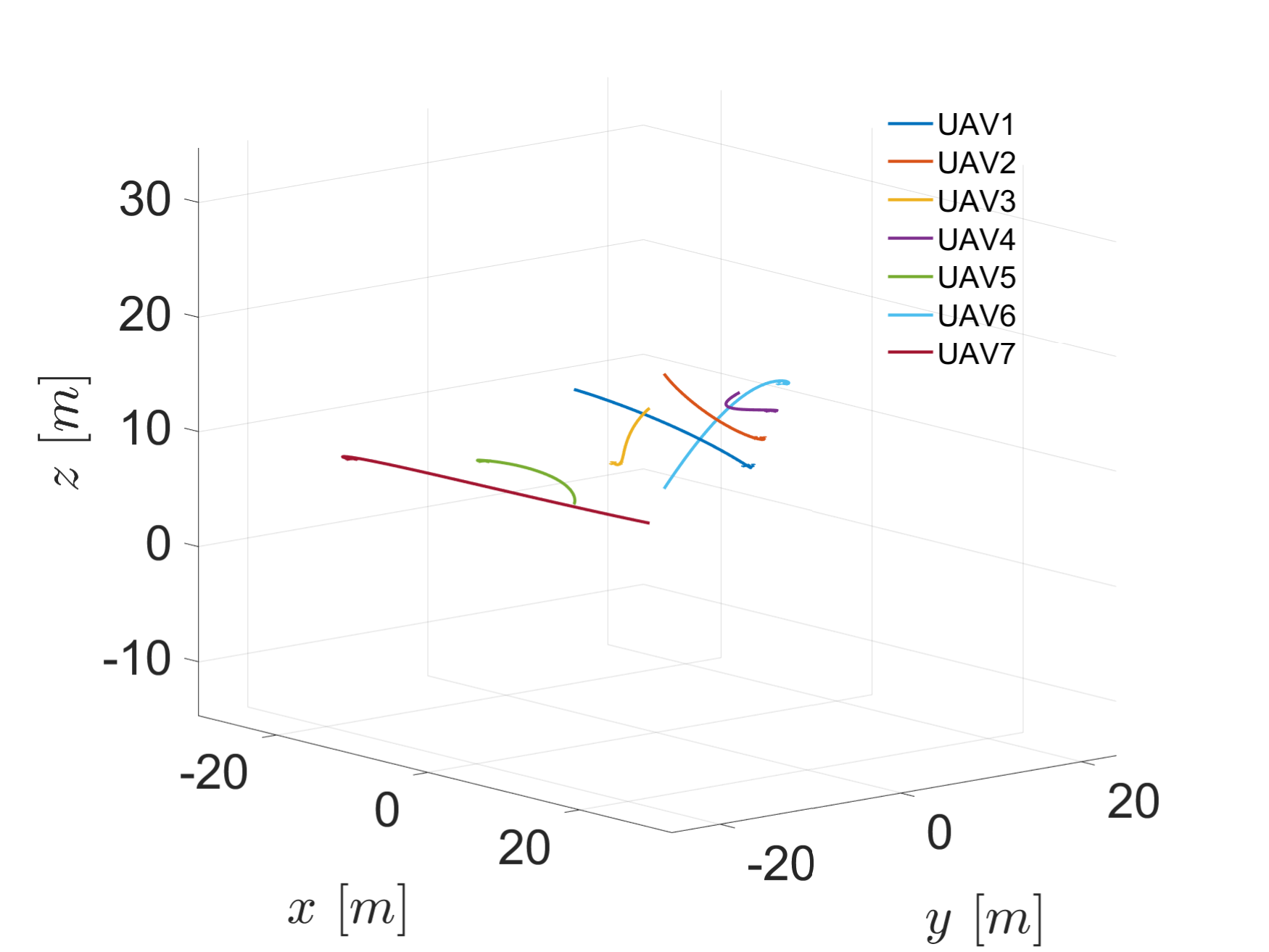}
		\caption{Optimal formation trajectories}
	\end{subfigure}
	\begin{subfigure}{0.49\linewidth}
	  \includegraphics[width=\linewidth]{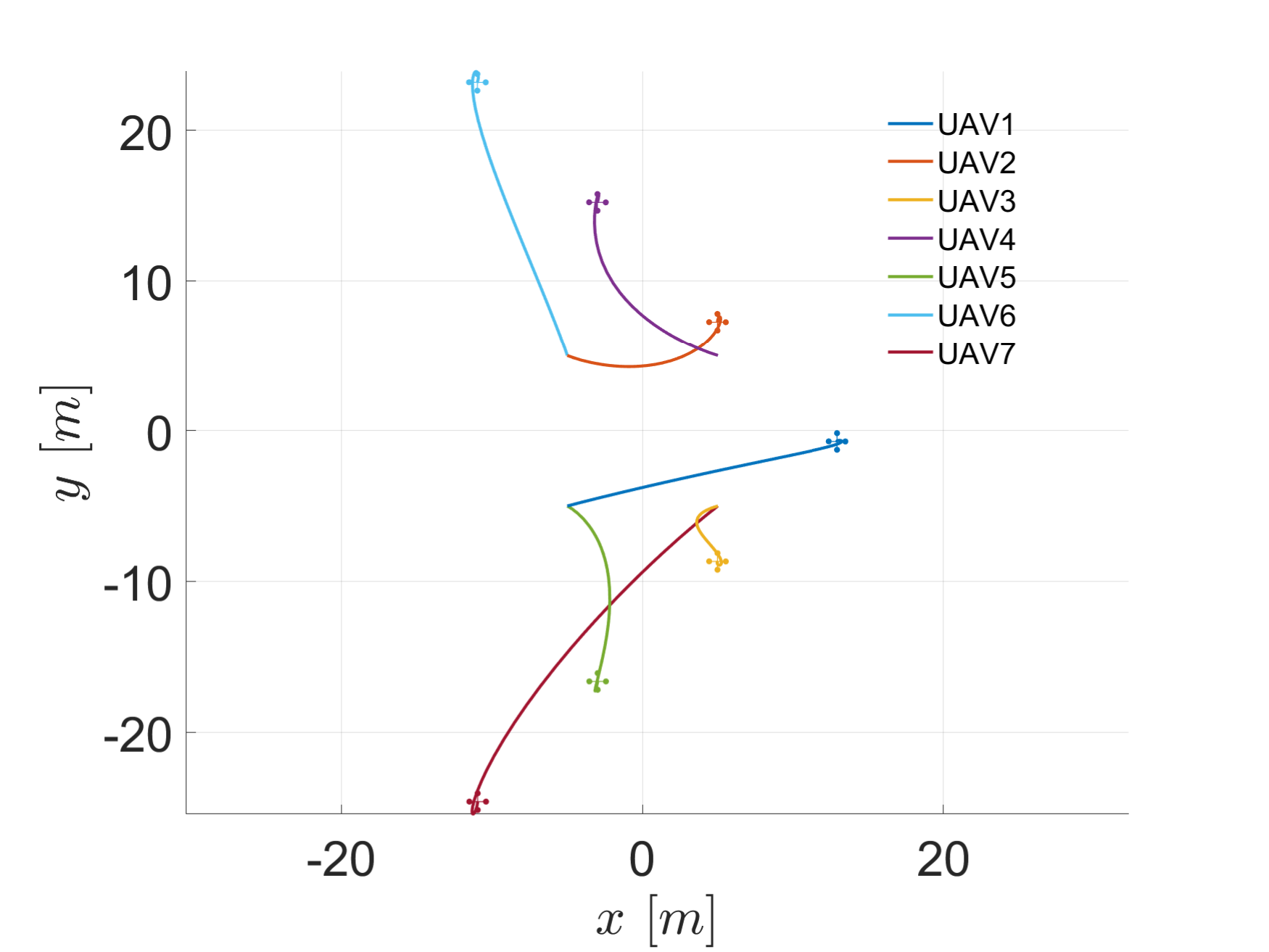}
		\caption{Top view of the trajectories}
    \end{subfigure}
\caption{Optimal formation trajectories of the seven-UAV team.}
\label{fig:planning7}
\end{figure}

\begin{figure}[htbp]
	\centering
	\begin{subfigure}{0.24\linewidth}
		\includegraphics[width=\linewidth]{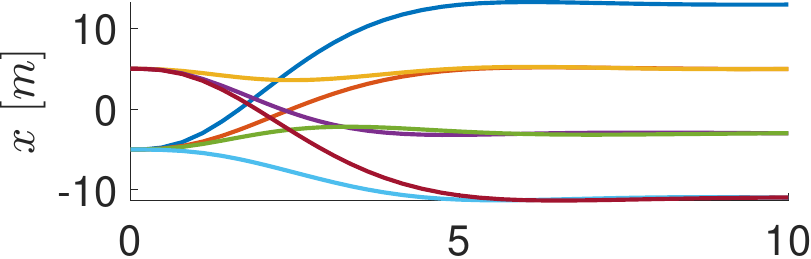}
        \includegraphics[width=\linewidth]{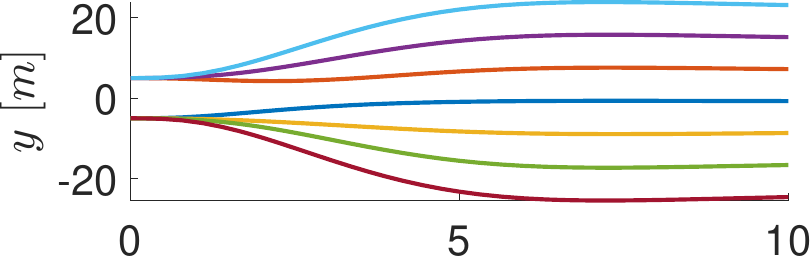}
        \includegraphics[width=\linewidth]{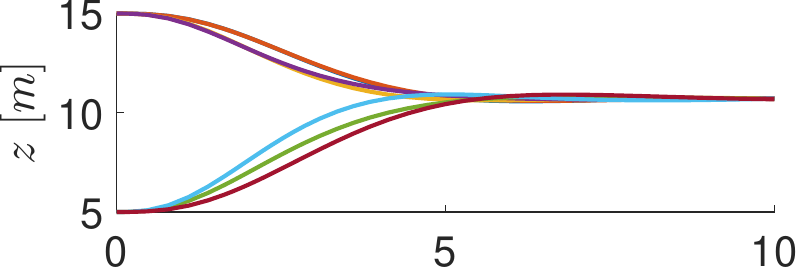}
        \includegraphics[width=\linewidth]{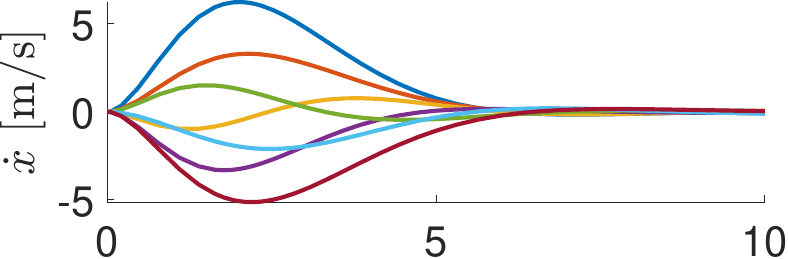}
        \includegraphics[width=\linewidth]{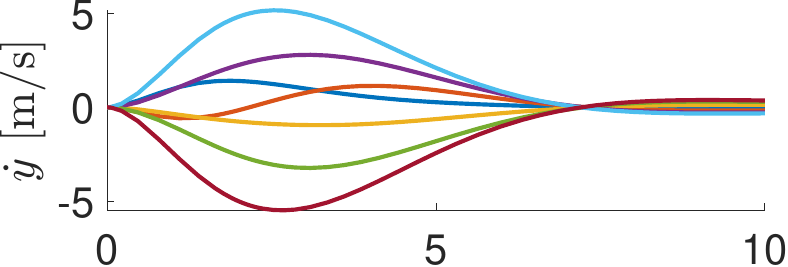}
        \includegraphics[width=\linewidth]{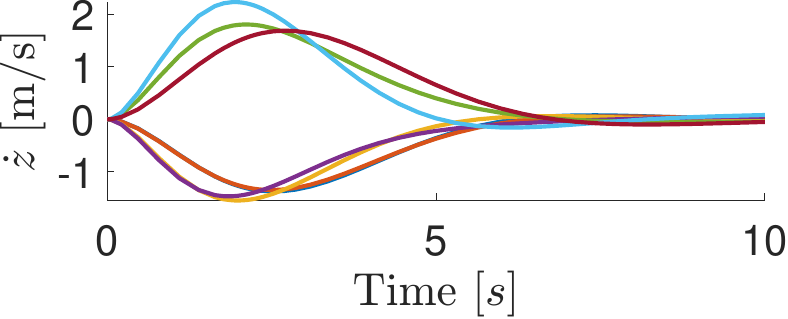}
		\caption{Positions and velocities}
	\end{subfigure}
    	\begin{subfigure}{0.24\linewidth}
		\includegraphics[width=\linewidth]{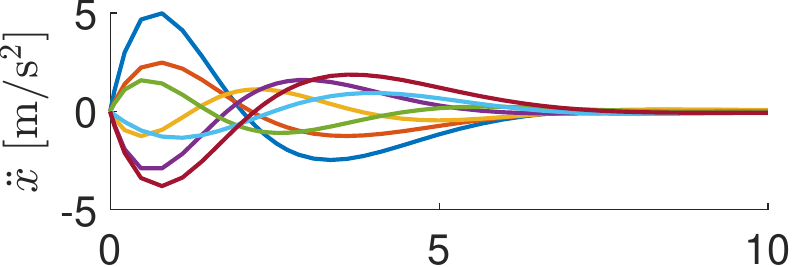}
        \includegraphics[width=\linewidth]{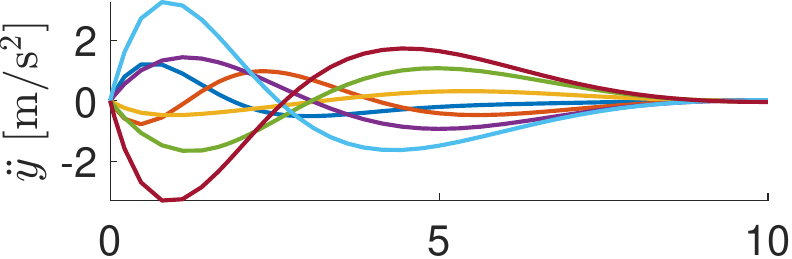}
        \includegraphics[width=\linewidth]{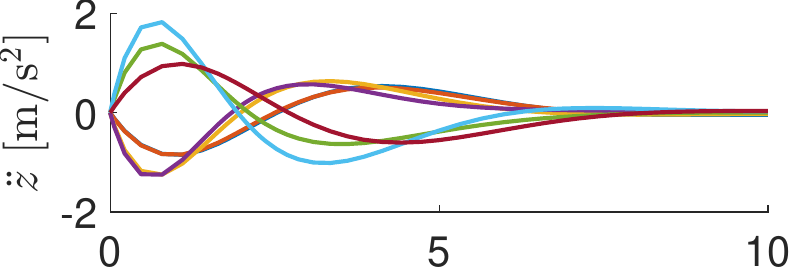}
        \includegraphics[width=\linewidth]{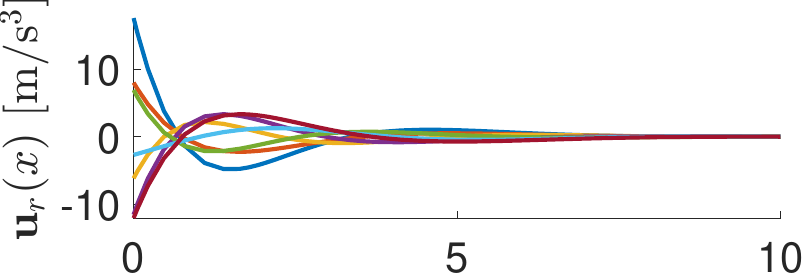}
        \includegraphics[width=\linewidth]{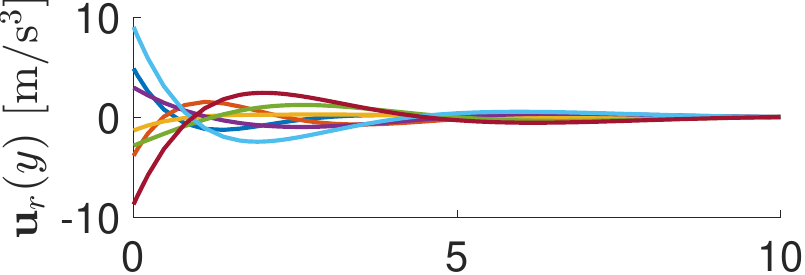}
        \includegraphics[width=\linewidth]{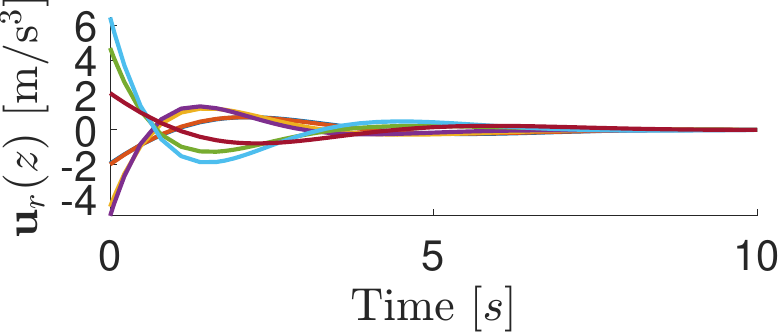}
		\caption{Accelerations and \(\mathbf{u}_{\mathbf{r}}\)}
	\end{subfigure}
    \begin{subfigure}{0.24\linewidth}
		\includegraphics[width=\linewidth]{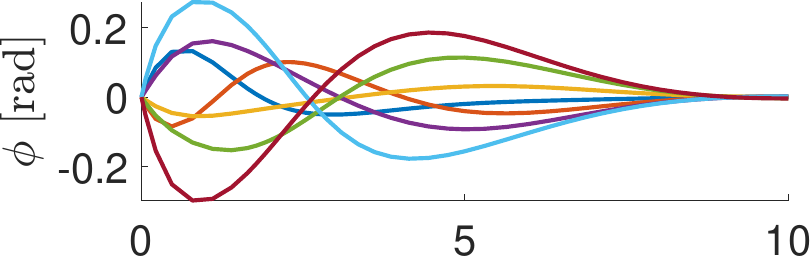}
        \includegraphics[width=\linewidth]{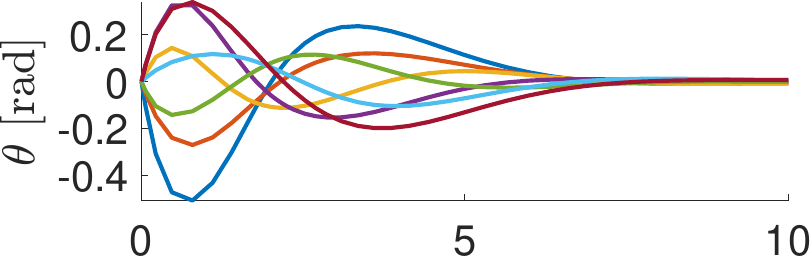}
        \includegraphics[width=\linewidth]{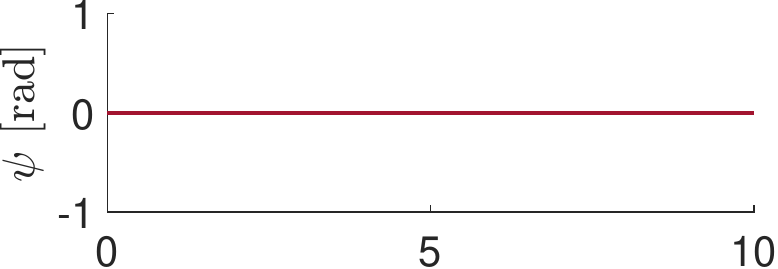}
        \includegraphics[width=\linewidth]{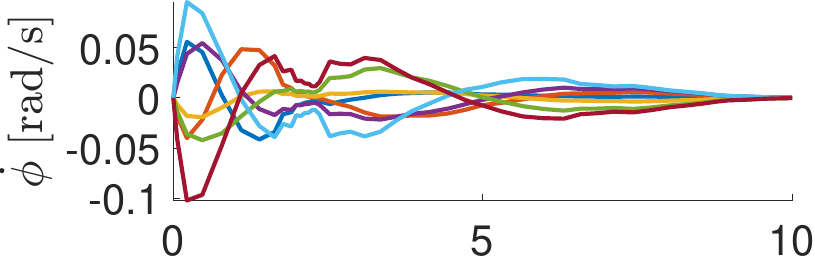}
        \includegraphics[width=\linewidth]{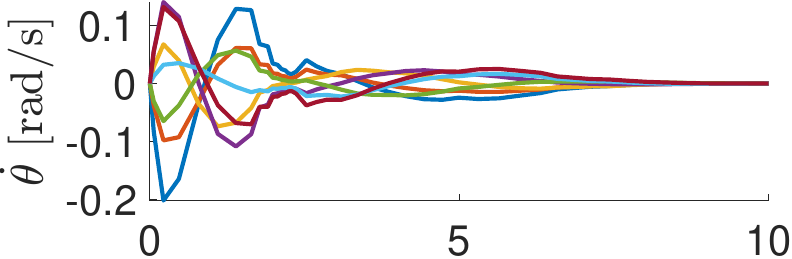}
        \includegraphics[width=\linewidth]{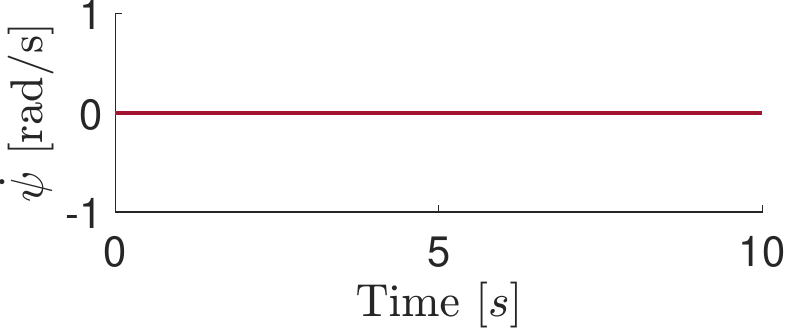}
		\caption{Attitudes with rates}
	\end{subfigure}
	\begin{subfigure}{0.45\linewidth}
	  \includegraphics[width=\linewidth]{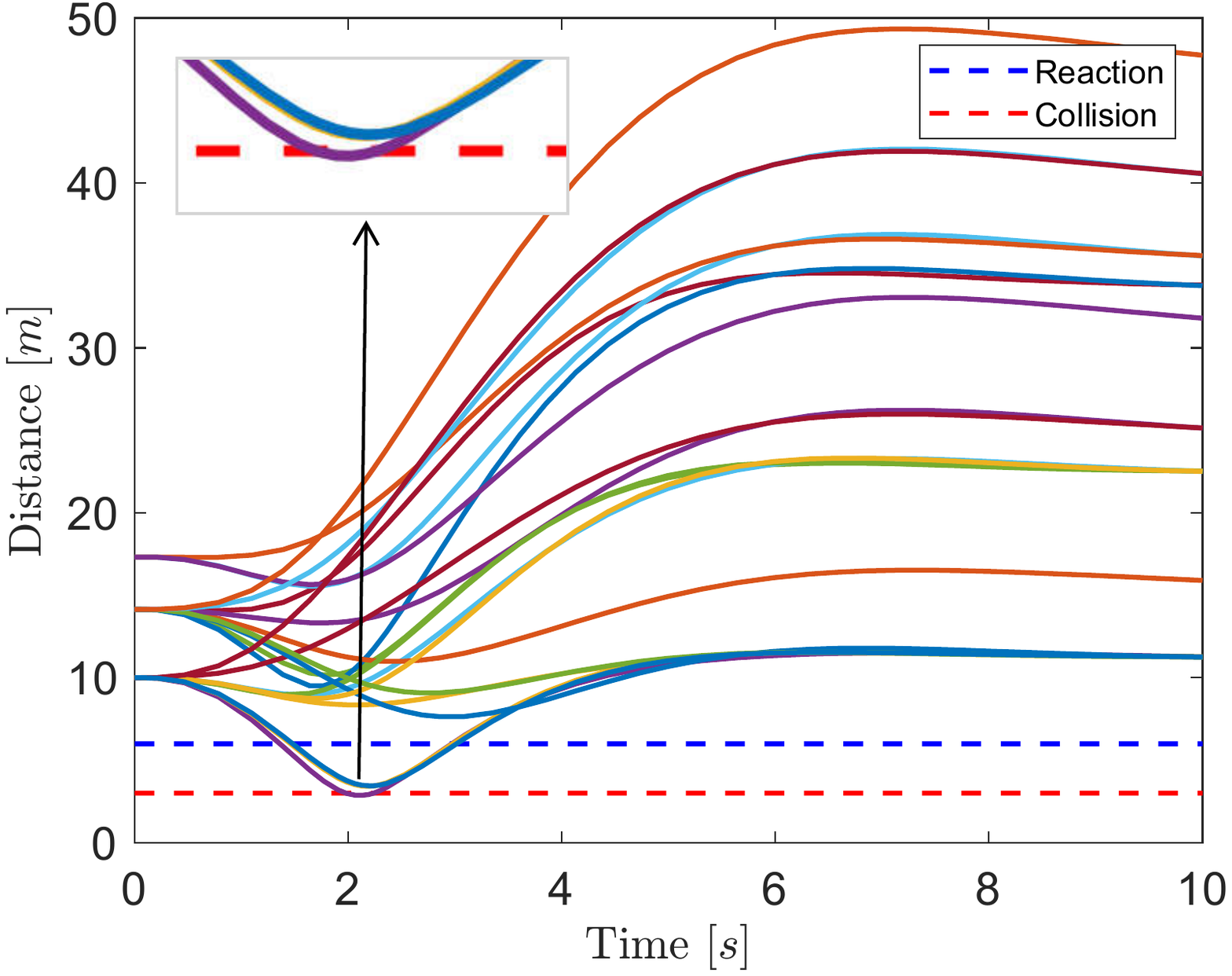}
		\caption{Euclidean distances}
		\label{fig:traj-dir7}
    \end{subfigure}
\caption{Time histories of positions, velocities, accelerations, jerks, attitudes with their rates, and Euclidean distances for the optimal formation trajectories of the seven-UAV team. (d) indicates a collision situation.}
\label{fig:planning-history7}
\end{figure}

\begin{figure}[htbp]
	\centering
	\begin{subfigure}{0.49\linewidth}
		\includegraphics[width=\linewidth]{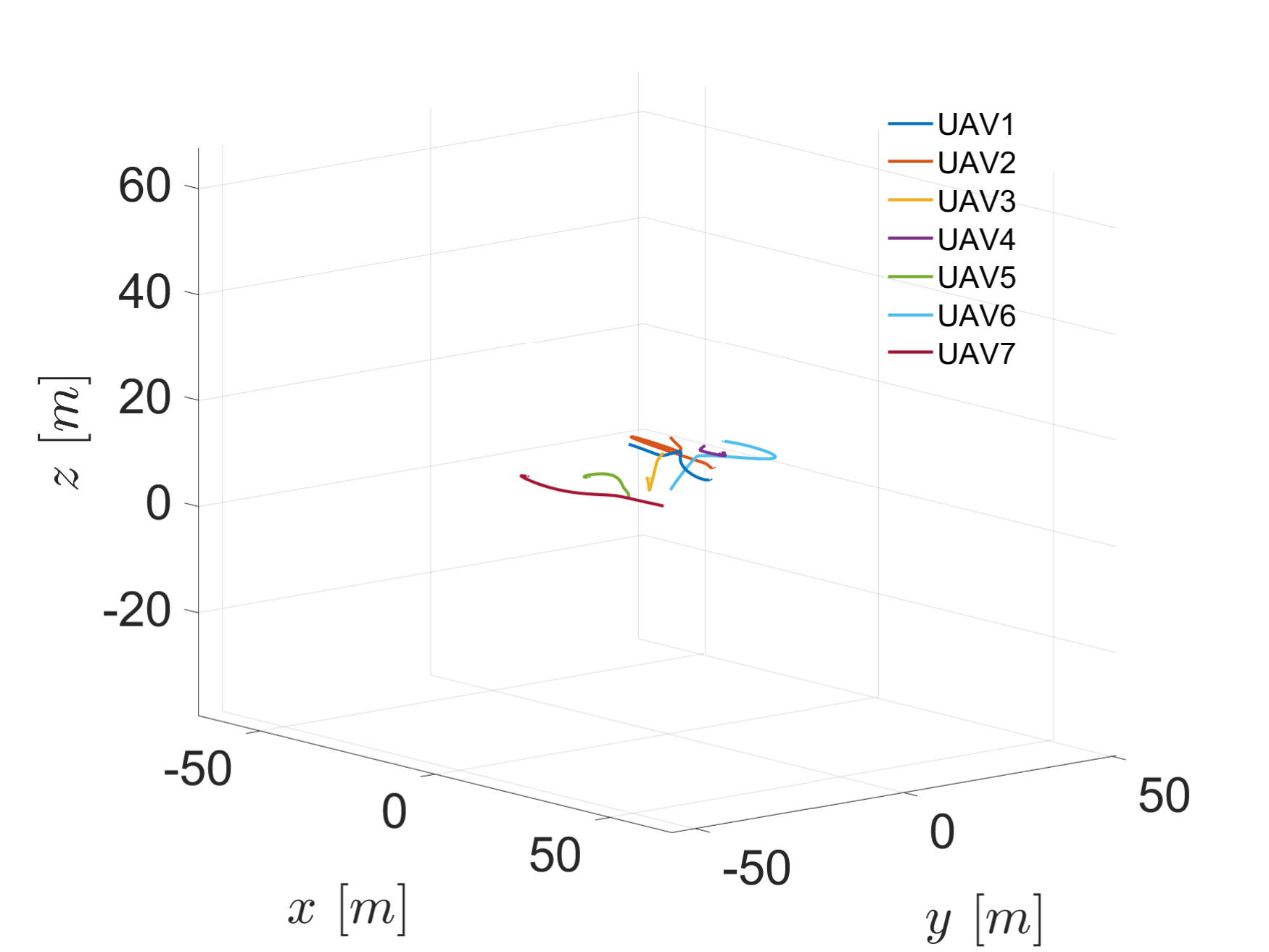}
		\caption{Formation tracking trajectories}
	\end{subfigure}
	\begin{subfigure}{0.49\linewidth}
	  \includegraphics[width=\linewidth]{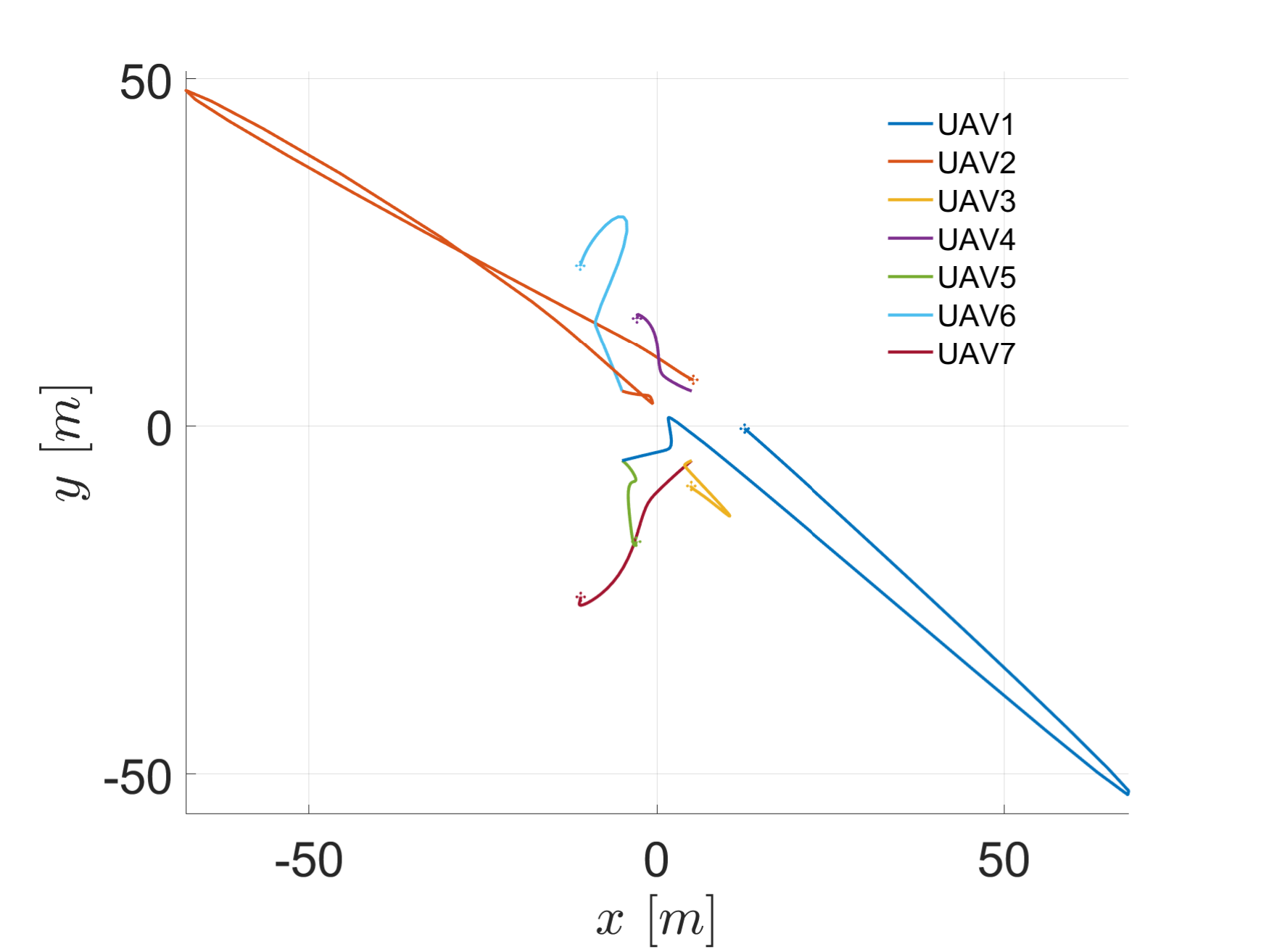}
		\caption{Top view of the tracking trajectories}
    \end{subfigure}
\caption{Formation trajectory tracking of the seven-UAV team.}
\label{fig:tracking7}
\end{figure}

\begin{figure}[htbp]
	\centering
	\begin{subfigure}{0.24\linewidth}
		\includegraphics[width=\linewidth]{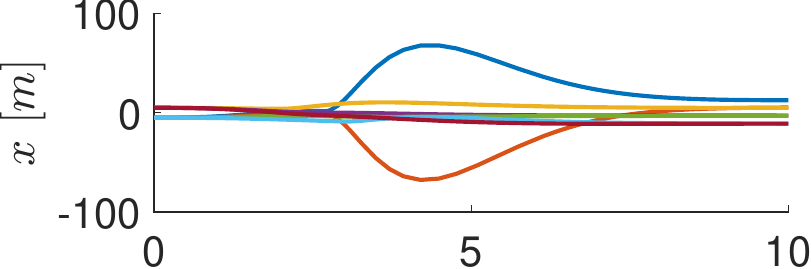}
        \includegraphics[width=\linewidth]{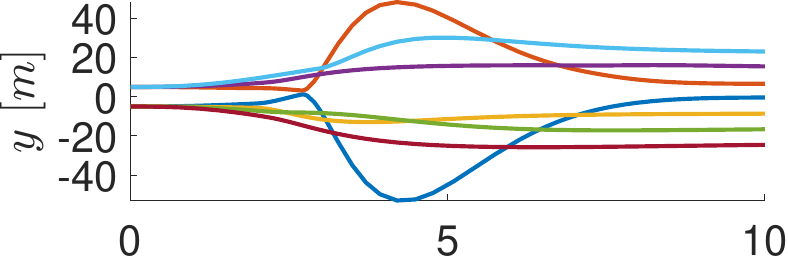}
        \includegraphics[width=\linewidth]{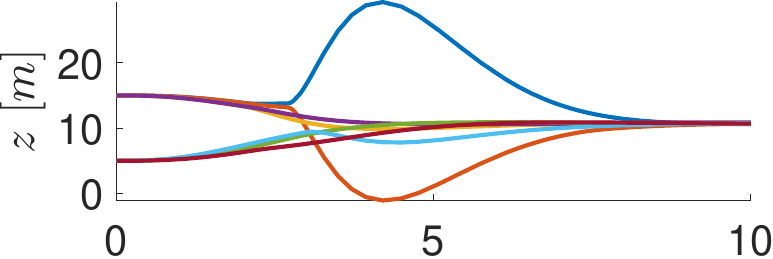}
        \includegraphics[width=\linewidth]{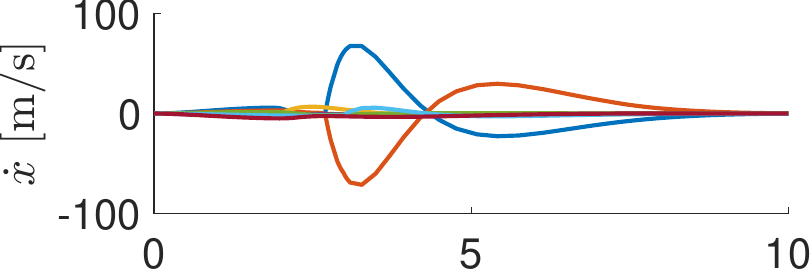}
        \includegraphics[width=\linewidth]{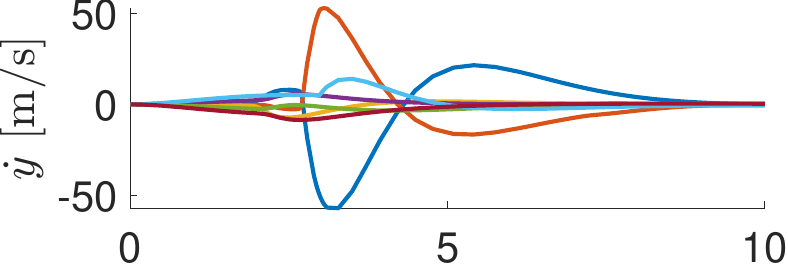}
        \includegraphics[width=\linewidth]{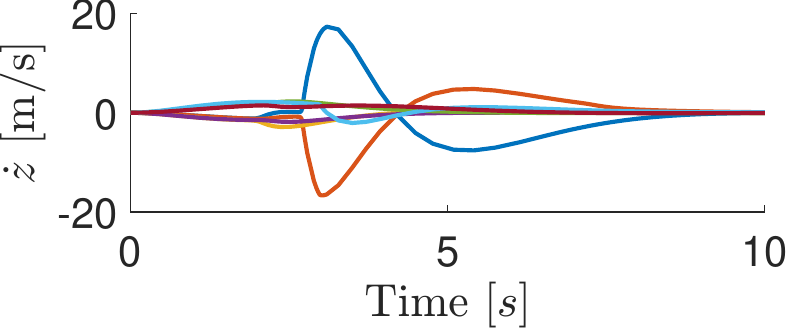}
		\caption{Positions and velocities}
	\end{subfigure}
    \begin{subfigure}{0.24\linewidth}
		\includegraphics[width=\linewidth]{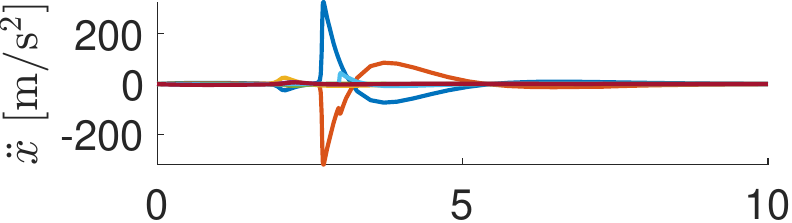}
        \includegraphics[width=\linewidth]{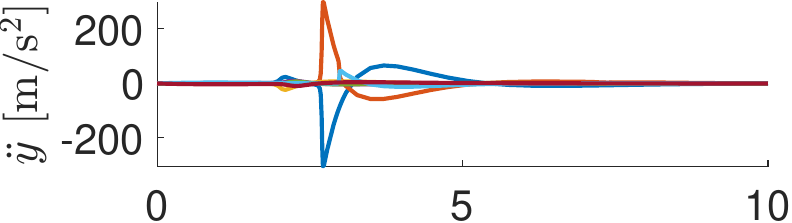}
        \includegraphics[width=\linewidth]{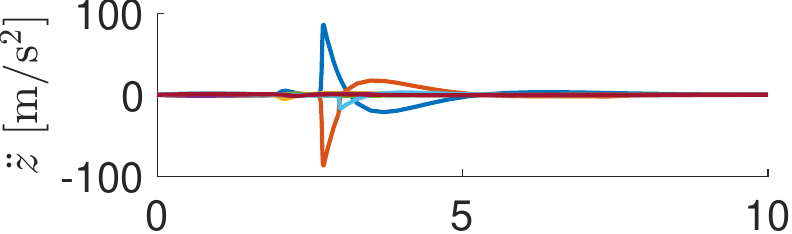}
        \includegraphics[width=\linewidth]{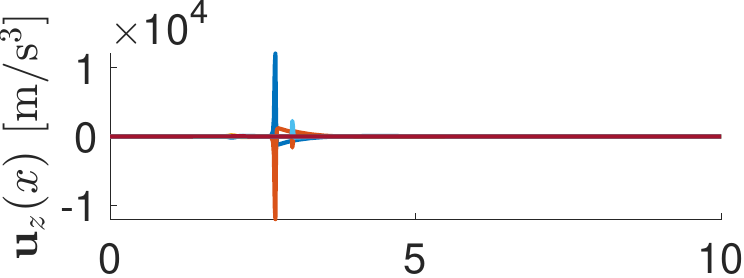}
        \includegraphics[width=\linewidth]{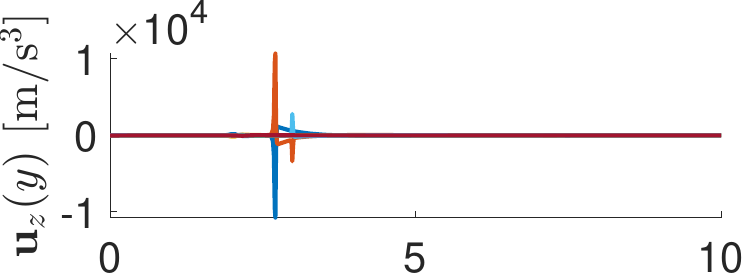}
        \includegraphics[width=\linewidth]{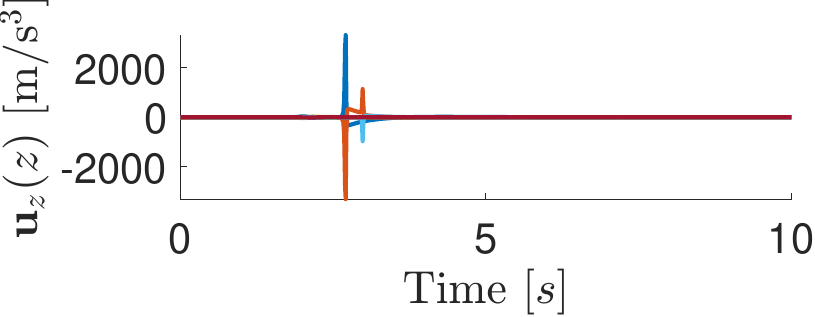}
		\caption{Accelerations and \(\mathbf{u}_{\mathbf{z}}\)}
	\end{subfigure}
    \begin{subfigure}{0.24\linewidth}
		\includegraphics[width=\linewidth]{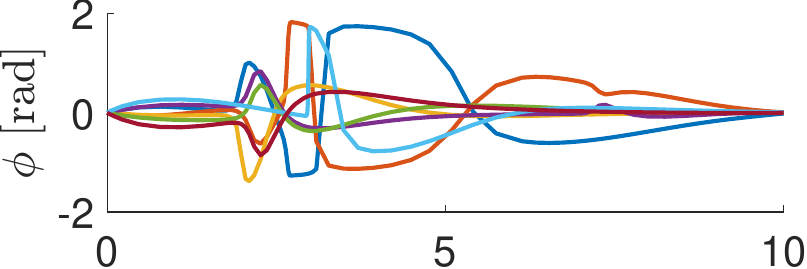}
        \includegraphics[width=\linewidth]{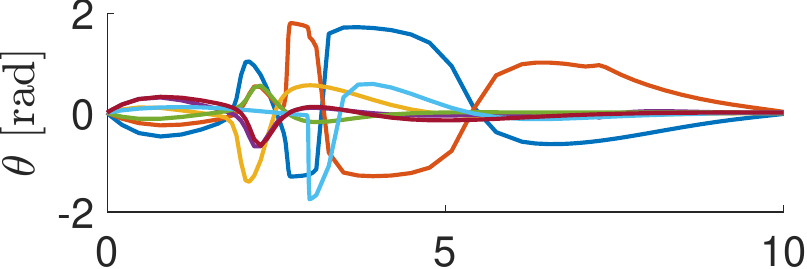}
        \includegraphics[width=\linewidth]{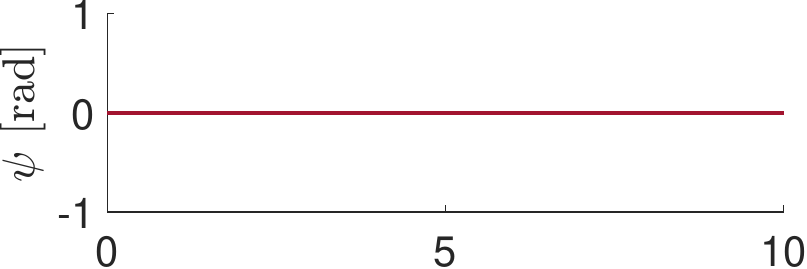}
        \includegraphics[width=\linewidth]{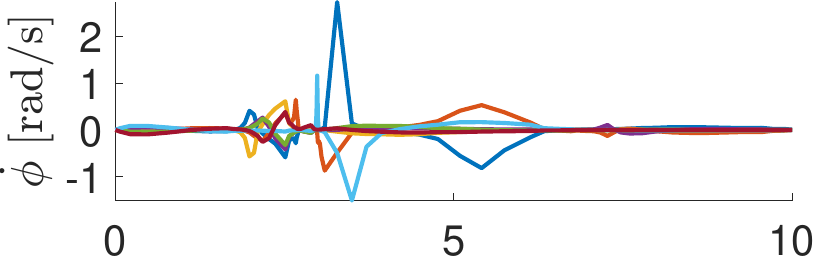}
        \includegraphics[width=\linewidth]{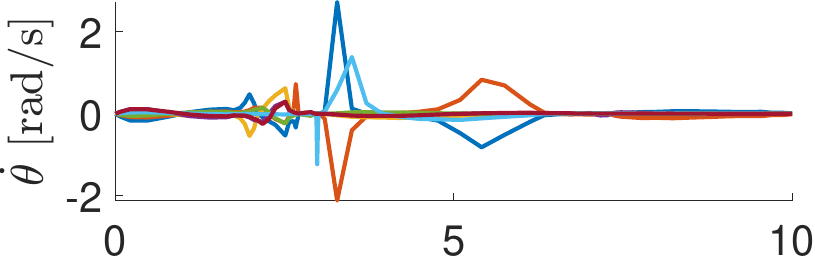}
        \includegraphics[width=\linewidth]{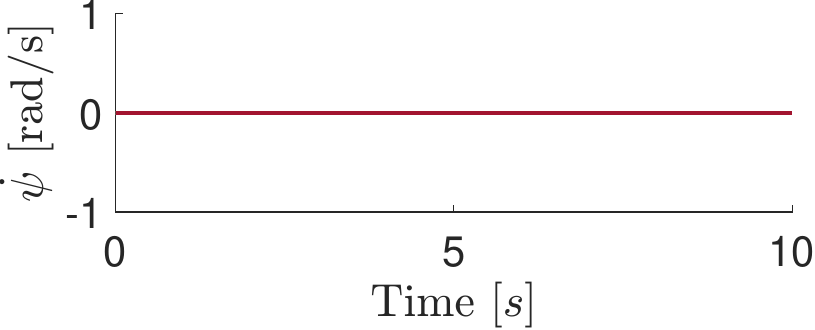}
		\caption{Attitudes with rates}
	\end{subfigure}
	\begin{subfigure}{0.45\linewidth}
	  \includegraphics[width=\linewidth]{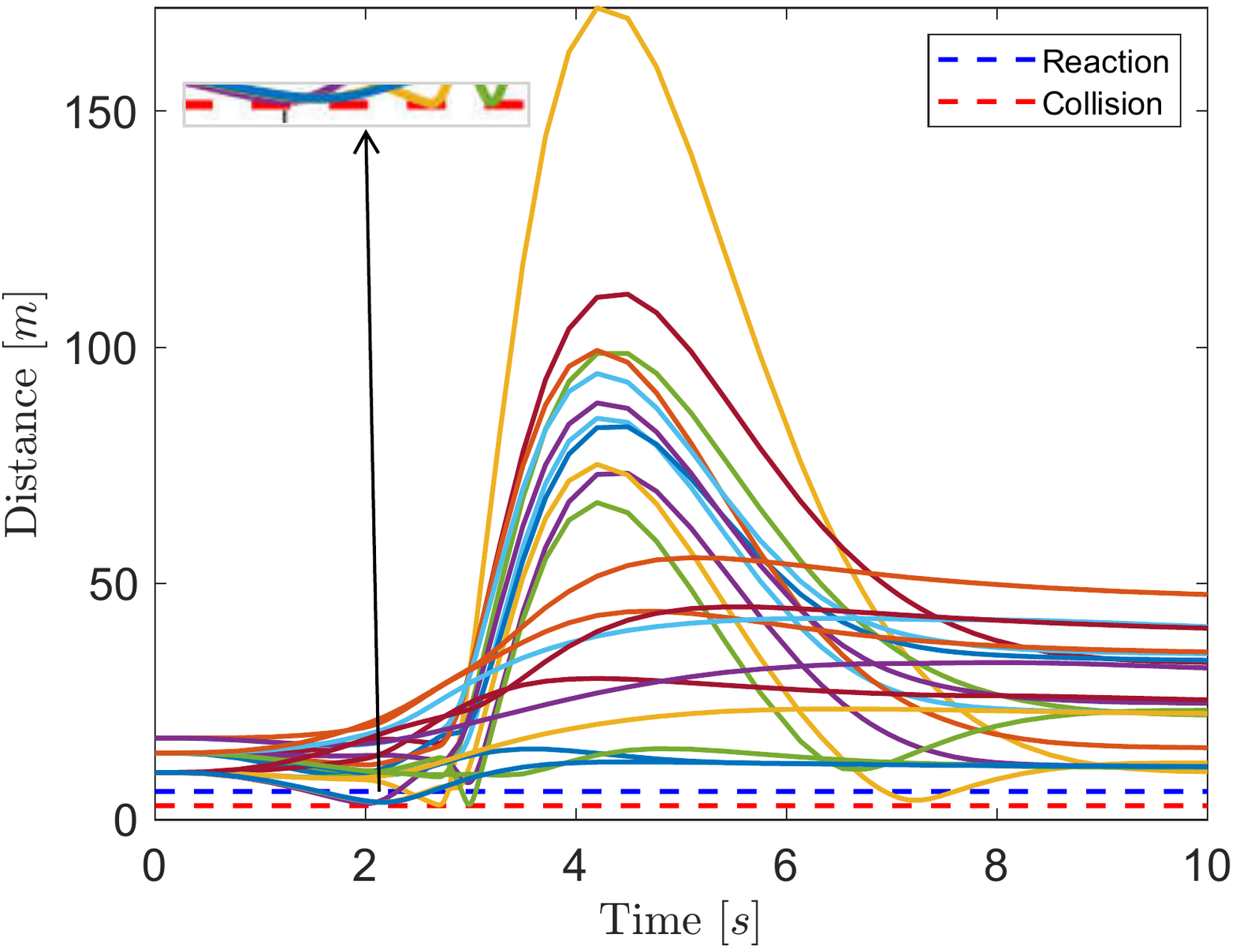}
		\caption{Euclidean distances}
    \end{subfigure}
\caption{Time histories of positions, velocities, accelerations, jerks, attitudes with their rates, and Euclidean distances for formation tracking trajectories of the seven-UAV team. (b) shows a sudden large jump in the tracking control input \( \mathbf{u}_z \) around \(t=2.5\, \mathrm{s}\).}
\label{fig:tracking-history7}
\end{figure}

\begin{figure}[htbp]
	\centering
	\begin{subfigure}{0.49\linewidth}
		\includegraphics[width=\linewidth]{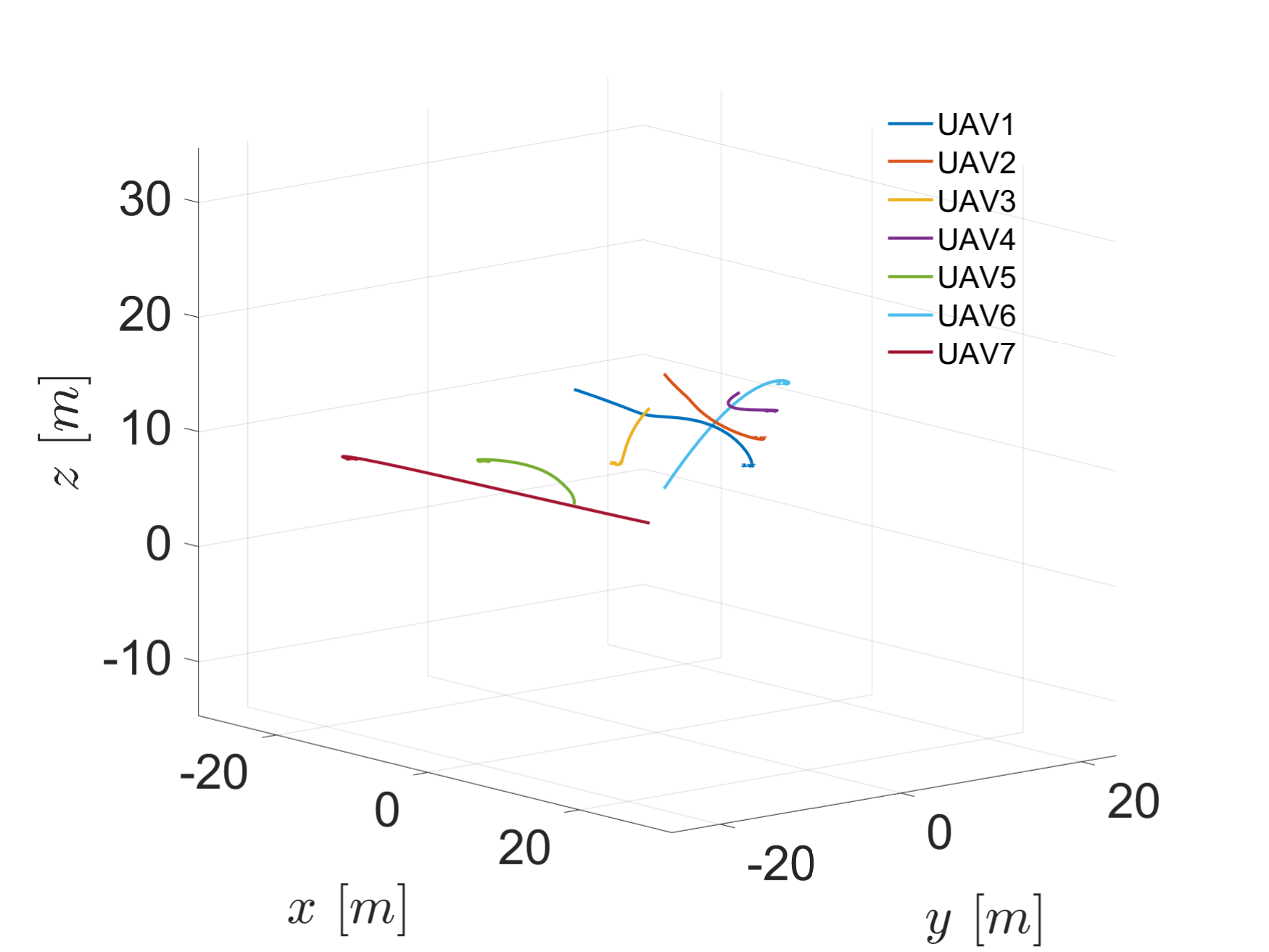}
		\caption{Directionally aware formation tracking trajectories}
	\end{subfigure}
	\begin{subfigure}{0.49\linewidth}
	  \includegraphics[width=\linewidth]{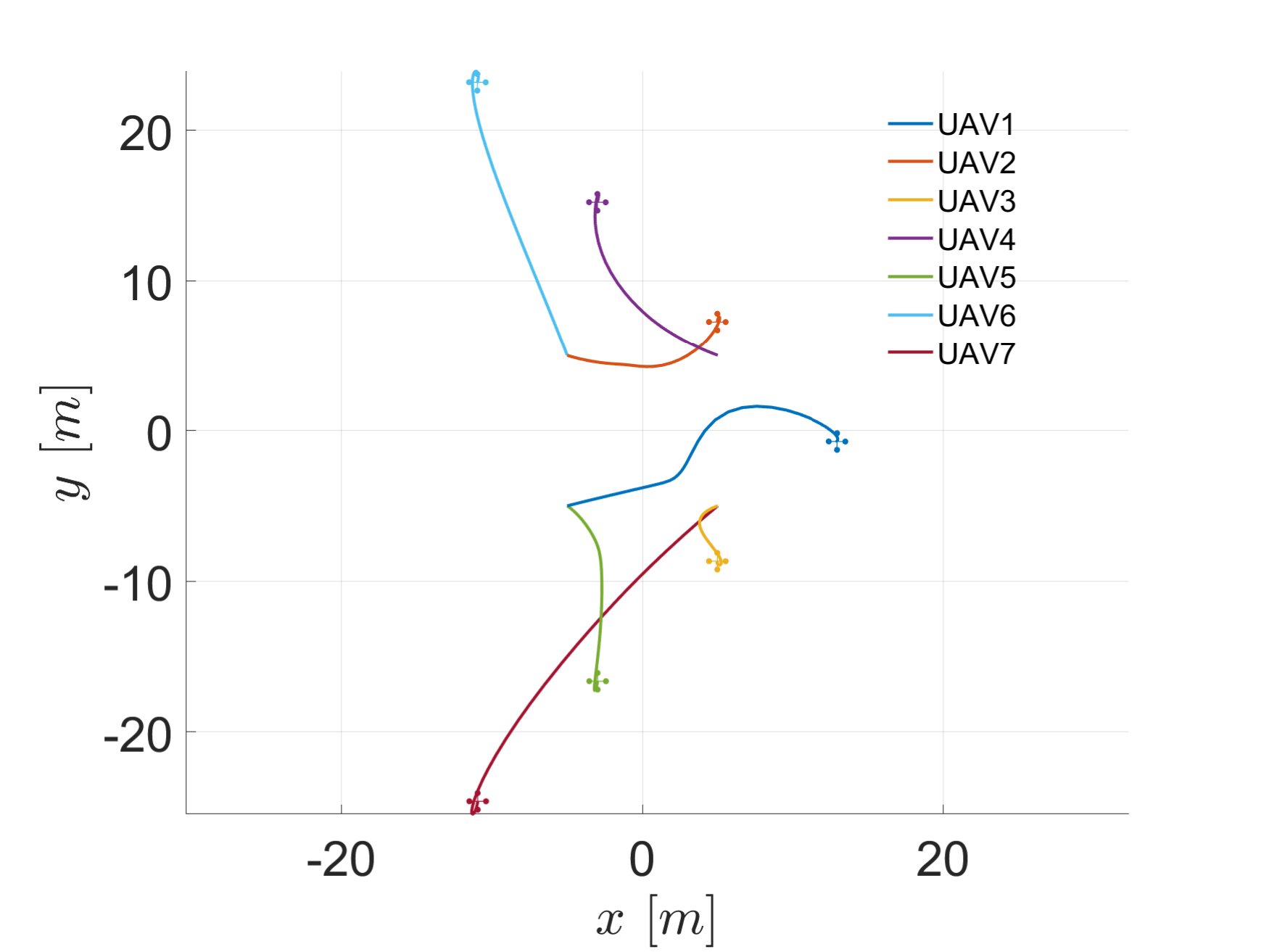}
		\caption{Top view of the tracking trajectories}
    \end{subfigure}
\caption{Formation trajectory tracking of the seven-UAV team utilizing the directionally aware collision avoidance strategy.}
\label{fig:aware7}
\end{figure}

\begin{figure}[htbp]
	\centering
	\begin{subfigure}{0.24\linewidth}
		\includegraphics[width=\linewidth]{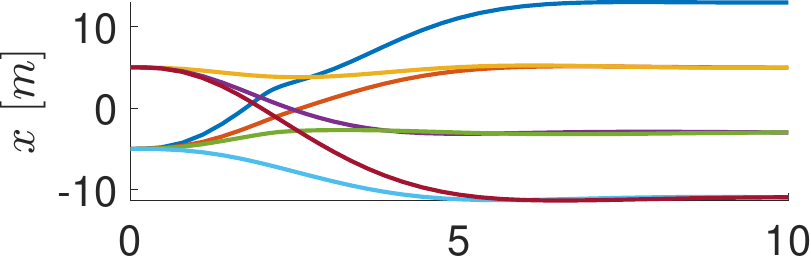}
        \includegraphics[width=\linewidth]{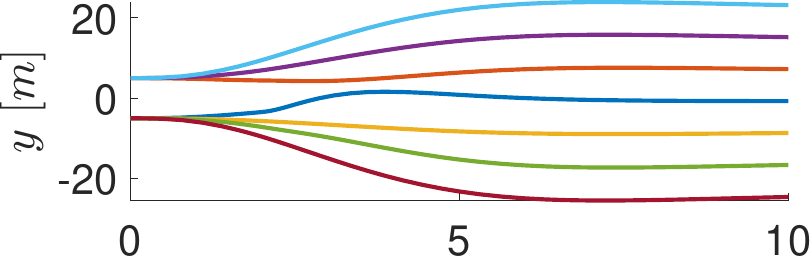}
        \includegraphics[width=\linewidth]{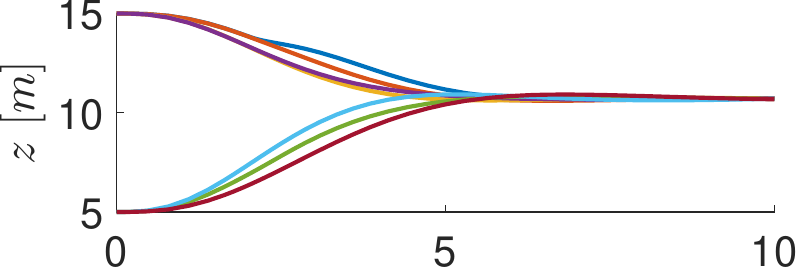}
        \includegraphics[width=\linewidth]{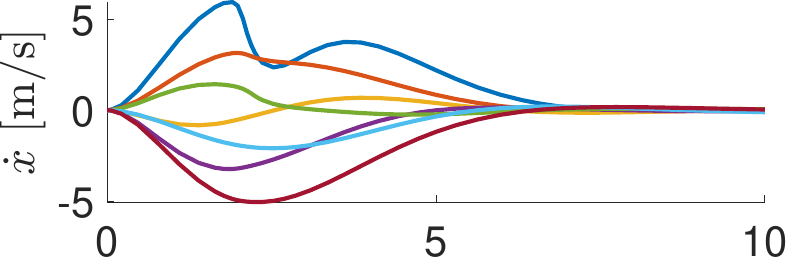}
        \includegraphics[width=\linewidth]{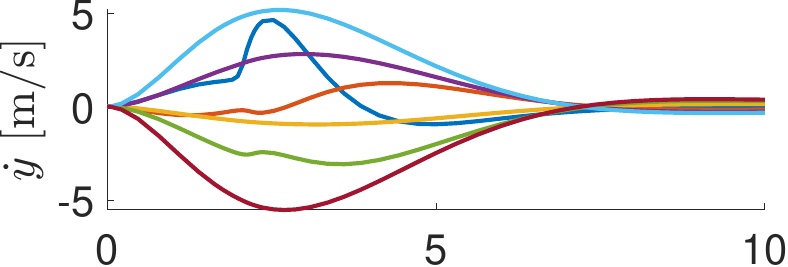}
        \includegraphics[width=\linewidth]{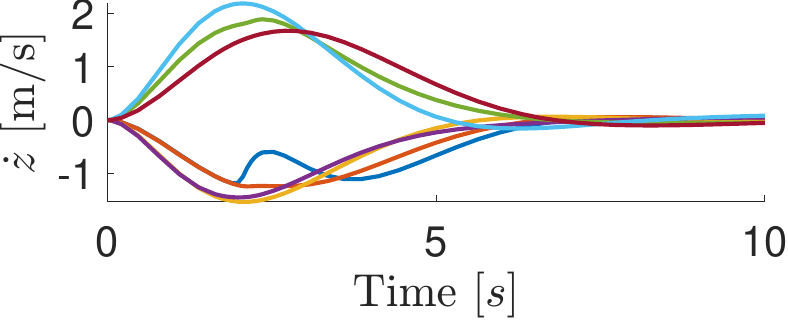}
		\caption{Positions and velocities}
	\end{subfigure}
    \begin{subfigure}{0.24\linewidth}
		\includegraphics[width=\linewidth]{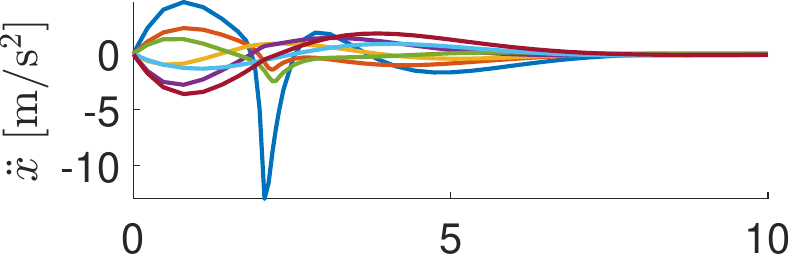}
        \includegraphics[width=\linewidth]{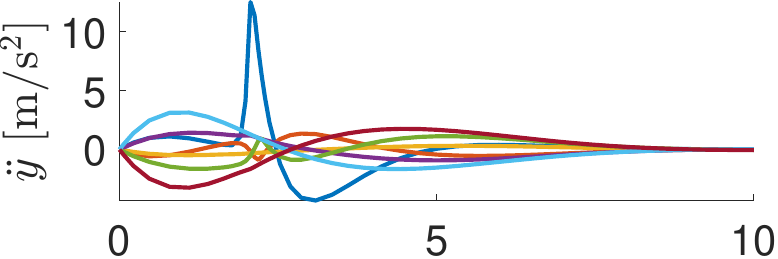}
        \includegraphics[width=\linewidth]{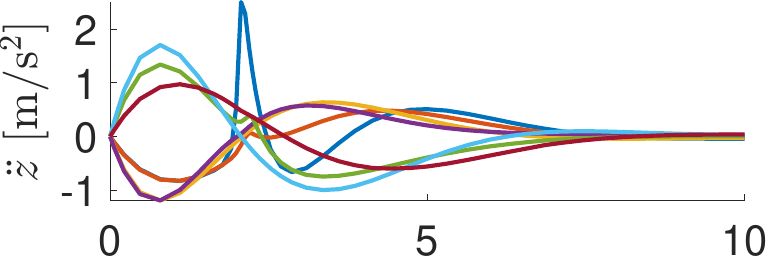}
        \includegraphics[width=\linewidth]{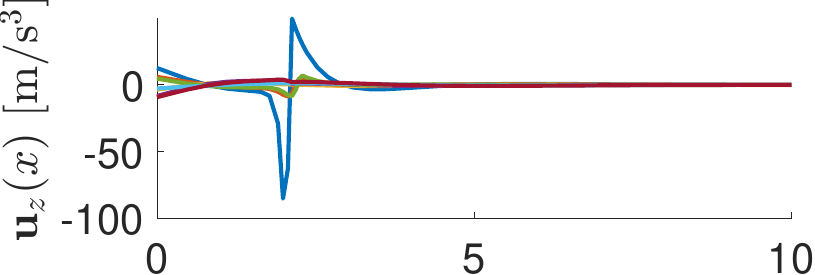}
        \includegraphics[width=\linewidth]{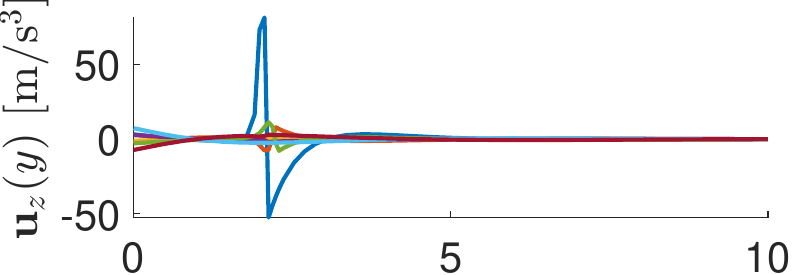}
        \includegraphics[width=\linewidth]{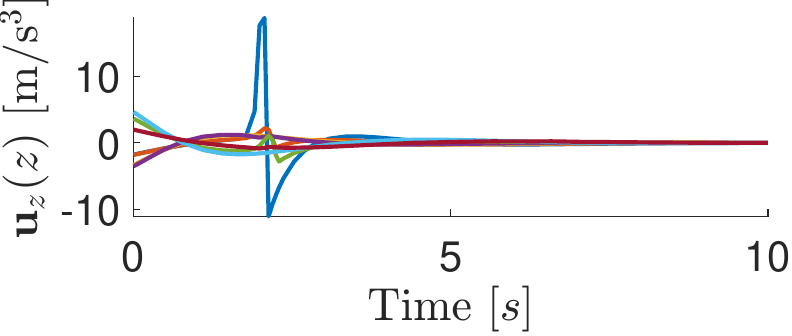}
		\caption{Accelerations and \(\mathbf{u}_{\mathbf{z}}\)}
	\end{subfigure}
    \begin{subfigure}{0.24\linewidth}
		\includegraphics[width=\linewidth]{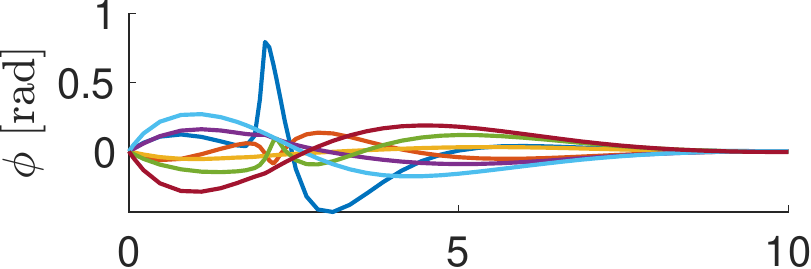}
        \includegraphics[width=\linewidth]{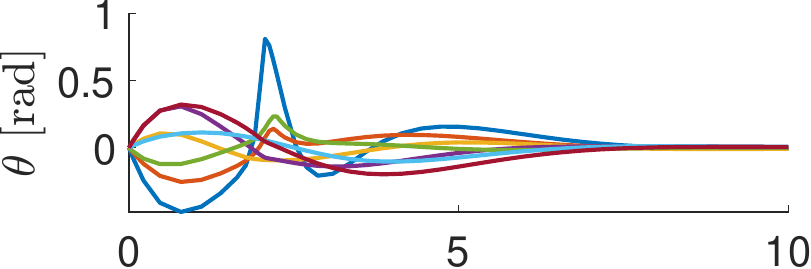}
        \includegraphics[width=\linewidth]{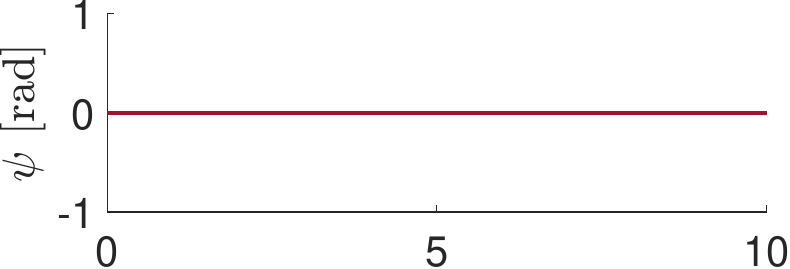}
        \includegraphics[width=\linewidth]{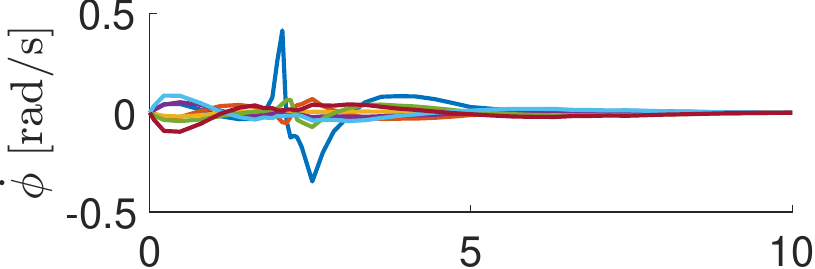}
        \includegraphics[width=\linewidth]{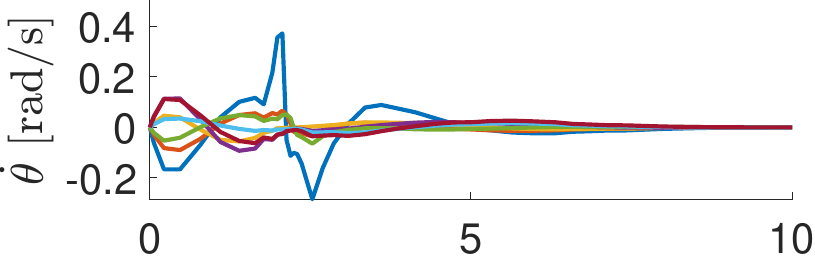}
        \includegraphics[width=\linewidth]{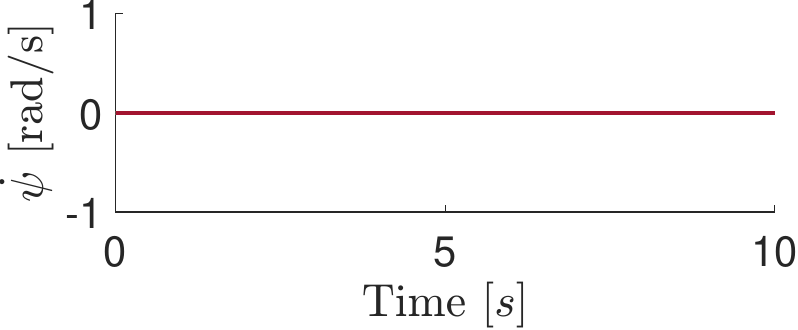}
		\caption{Attitudes with rates}
	\end{subfigure}
	\begin{subfigure}{0.45\linewidth}
	  \includegraphics[width=\linewidth]{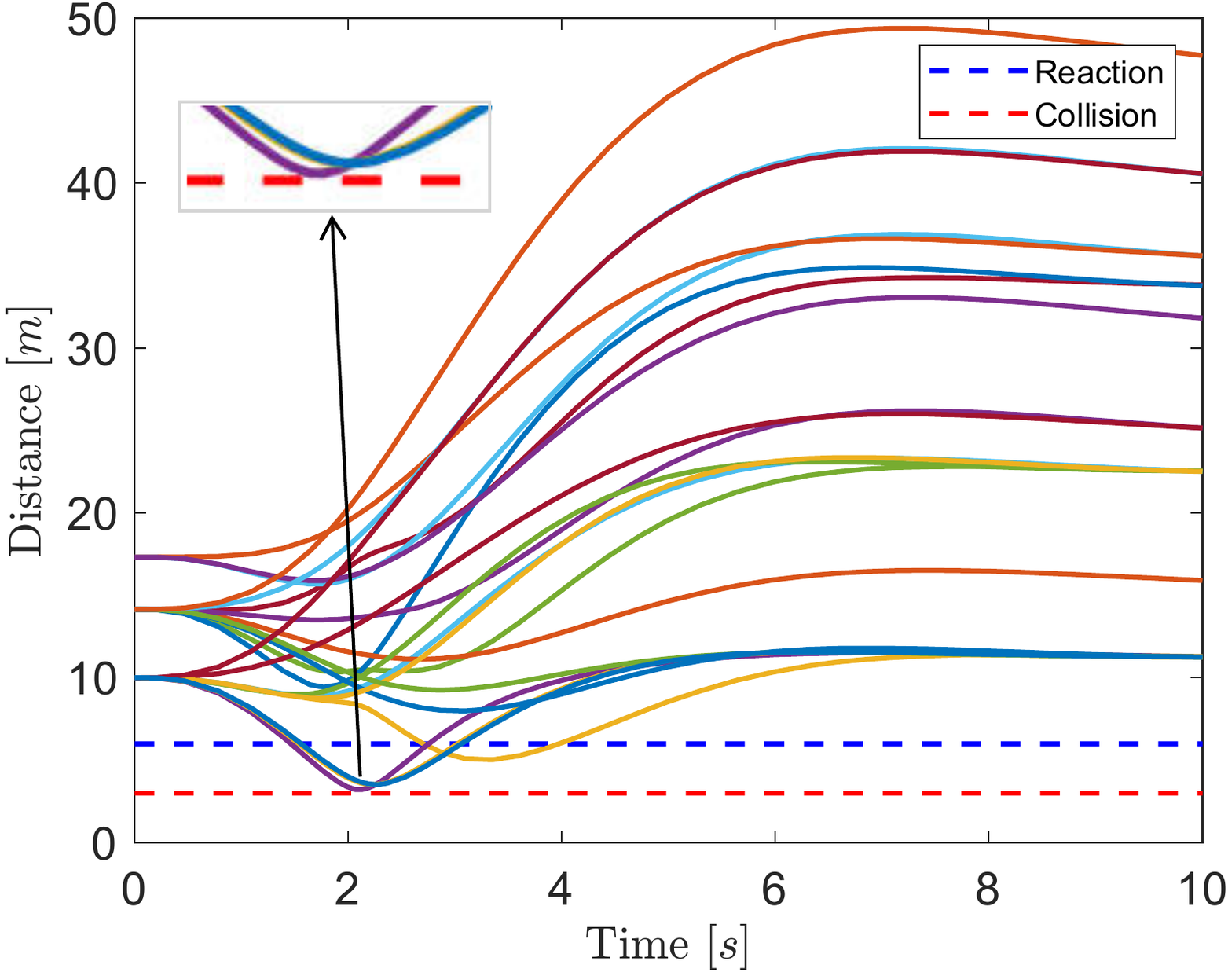}
		\caption{Euclidean distances}
    \end{subfigure}
    \caption{Time histories of positions, velocities, accelerations, jerks, attitudes with their rates, and Euclidean distances for formation tracking trajectories of the seven-UAV team utilizing the directionally aware collision avoidance strategy.}
\label{fig:aware-history7}
\end{figure}

The trajectory in Fig.~\ref{fig:aware-history7} may still be impractical for some physical UAVs due to actuator constraints. However, the proposed formation control design can adapt to these limitations. For instance, reducing the weighting parameters \(\mu_{ij}\) of the formation performance index improves trajectory feasibility. Fig.~\ref{fig:aware-history7_slow} shows the time histories of trajectories using the directionally aware collision avoidance strategy from Fig.~\ref{fig:aware-history7}, with \(\mu_{ij}\) adjusted to 10\% of their original values. The improved feasibility is evident from the reduced magnitudes of velocities, accelerations, jerks, attitudes, and their rates.

To validate the feasibility of the planned trajectories, a simulation experiment was conducted using the high-fidelity Gazebo simulator in conjunction with the MRS UAV system\footnote{\url{https://github.com/ctu-mrs/mrs_uav_system}}.
The simulation incorporates full system dynamics, aerodynamic effects, and sensor noise while utilizing the same control and estimation software as the real platforms.
Over years of validation through extensive use and comparisons with real-world experiments, it has demonstrated high accuracy, rendering the simulation-to-reality gap negligible \cite{Baca2021}.
The simulation was performed using UAV platforms based on the DJI Flame Wheel F450.

Snapshots of Gazebo simulator-based formation trajectories and their time histories under the directionally aware collision avoidance tracking strategy with improved feasibility are shown in Figs.~\ref{fig:aware7_slow}-\ref{fig:aware-history7_slow}. The reference trajectories from the latest example (shown in green in Fig.~\ref{fig:aware7_slow}) generate feasible control commands suitable for direct hardware execution~\cite{baca2018model,petrlik2020robust}. By leveraging this high-fidelity simulator, we validated the feasibility of the formation control scheme in a realistic simulation environment, ensuring minimal discrepancies between simulation and real-world implementation. These results further substantiate the effectiveness of the proposed formation trajectory planning and tracking framework, demonstrating its potential for the deployment of real-world UAV teams.

\begin{figure}[htbp]
	\centering
	\begin{subfigure}{0.3\linewidth}
		\includegraphics[width=\linewidth]{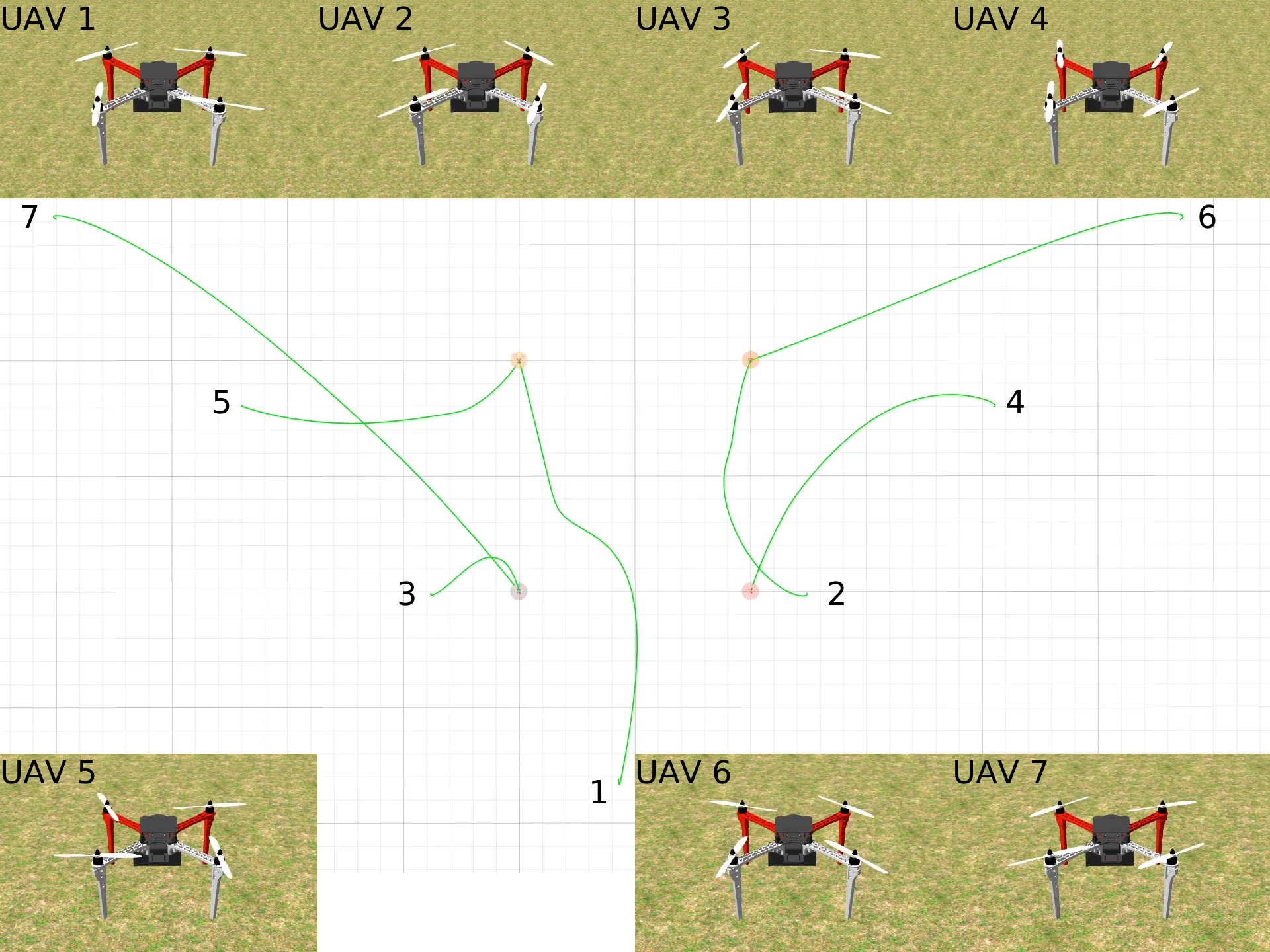}
		\caption{\(t=0\,\mathrm{s}\)}
	\end{subfigure}
    \hspace{0.025\linewidth}
    \begin{subfigure}{0.3\linewidth}
		\includegraphics[width=\linewidth]{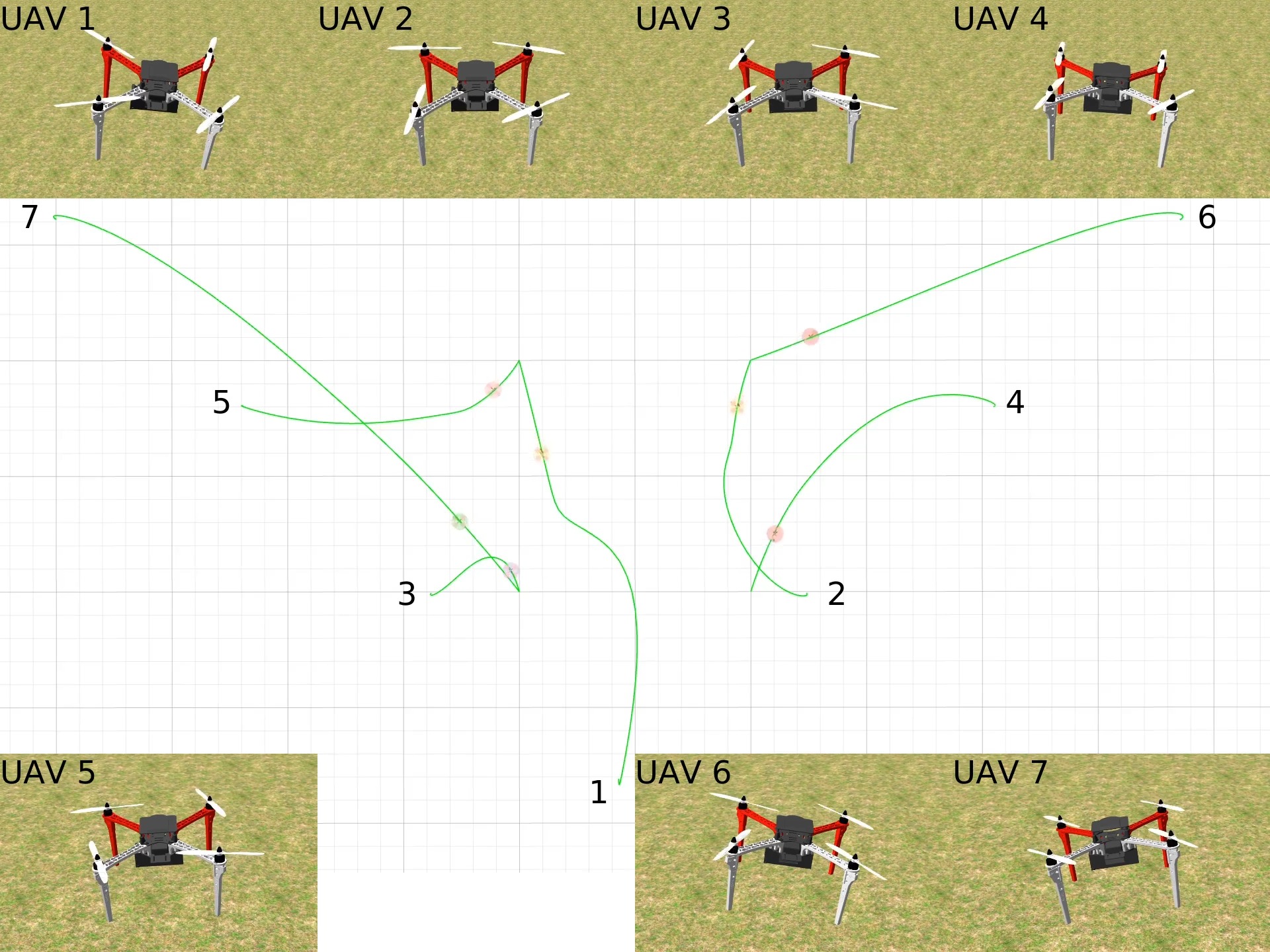}
		\caption{\(t=2\,\mathrm{s}\)}
	\end{subfigure}
    \hspace{0.025\linewidth}
    \begin{subfigure}{0.3\linewidth}
		\includegraphics[width=\linewidth]{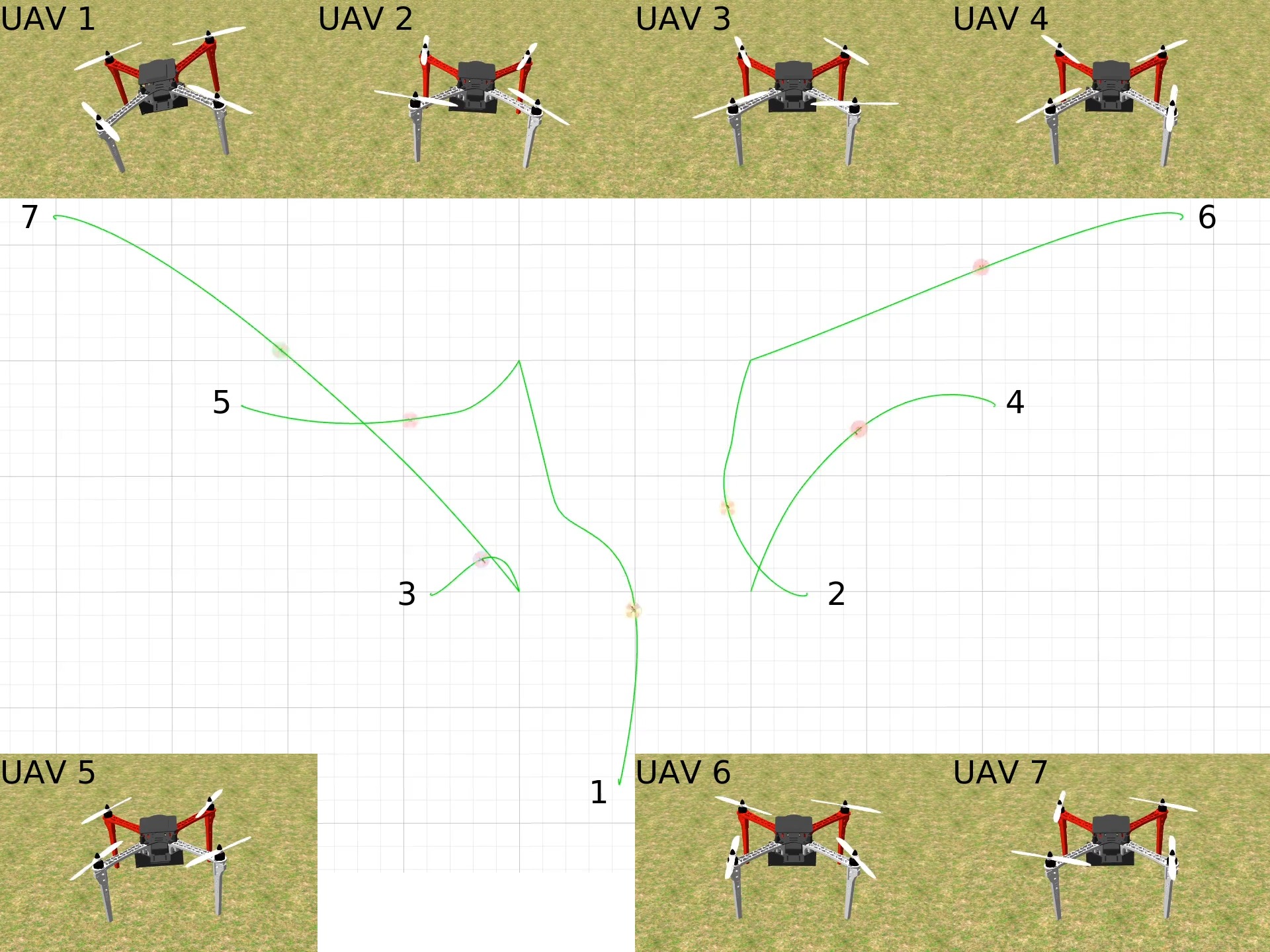}
		\caption{\(t=4\,\mathrm{s}\)}
	\end{subfigure}
    \begin{subfigure}{0.3\linewidth}
		\includegraphics[width=\linewidth]{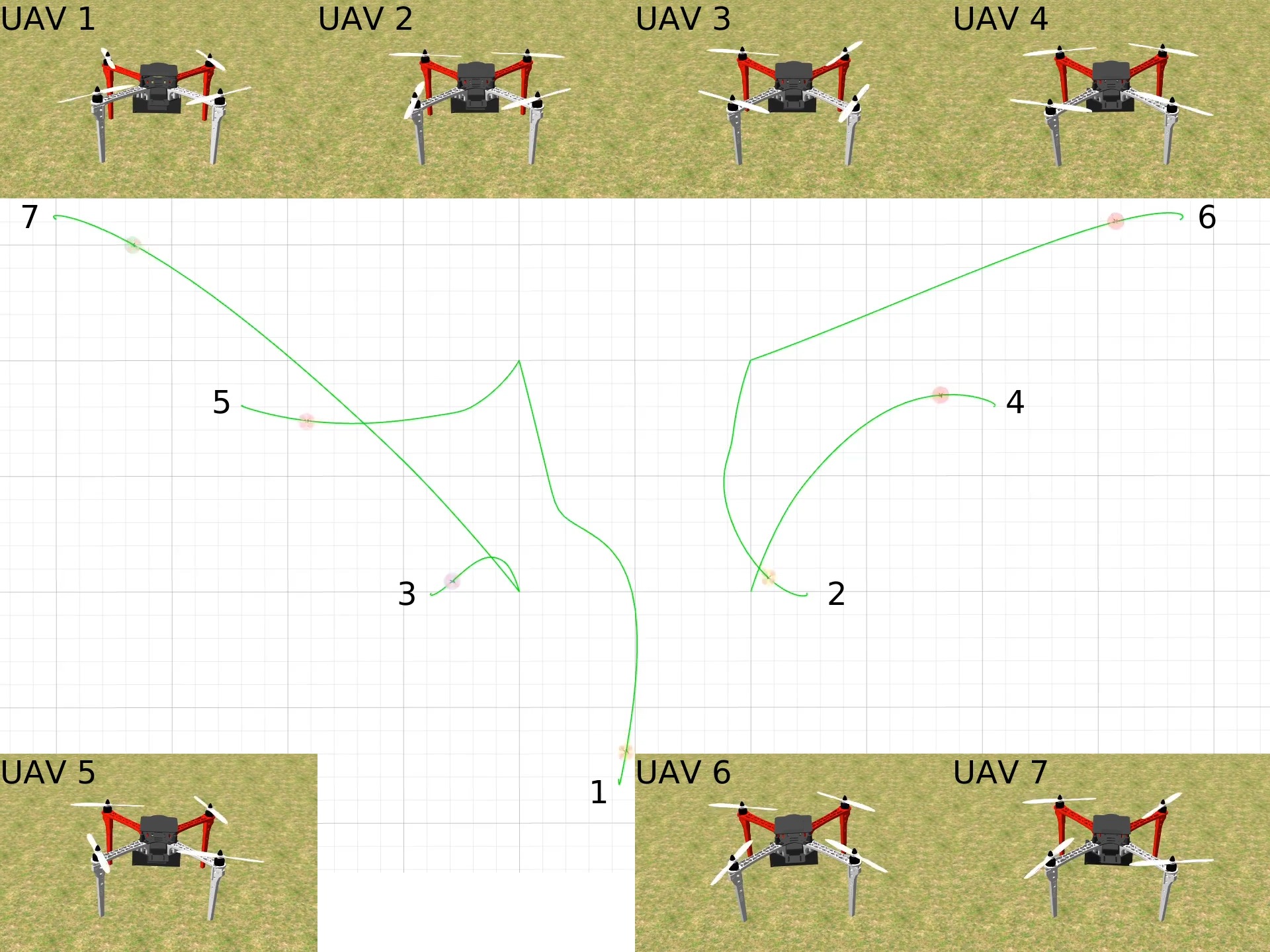}
		\caption{\(t=6\,\mathrm{s}\)}
	\end{subfigure}
     \hspace{0.025\linewidth}
	\begin{subfigure}{0.3\linewidth}
		\includegraphics[width=\linewidth]{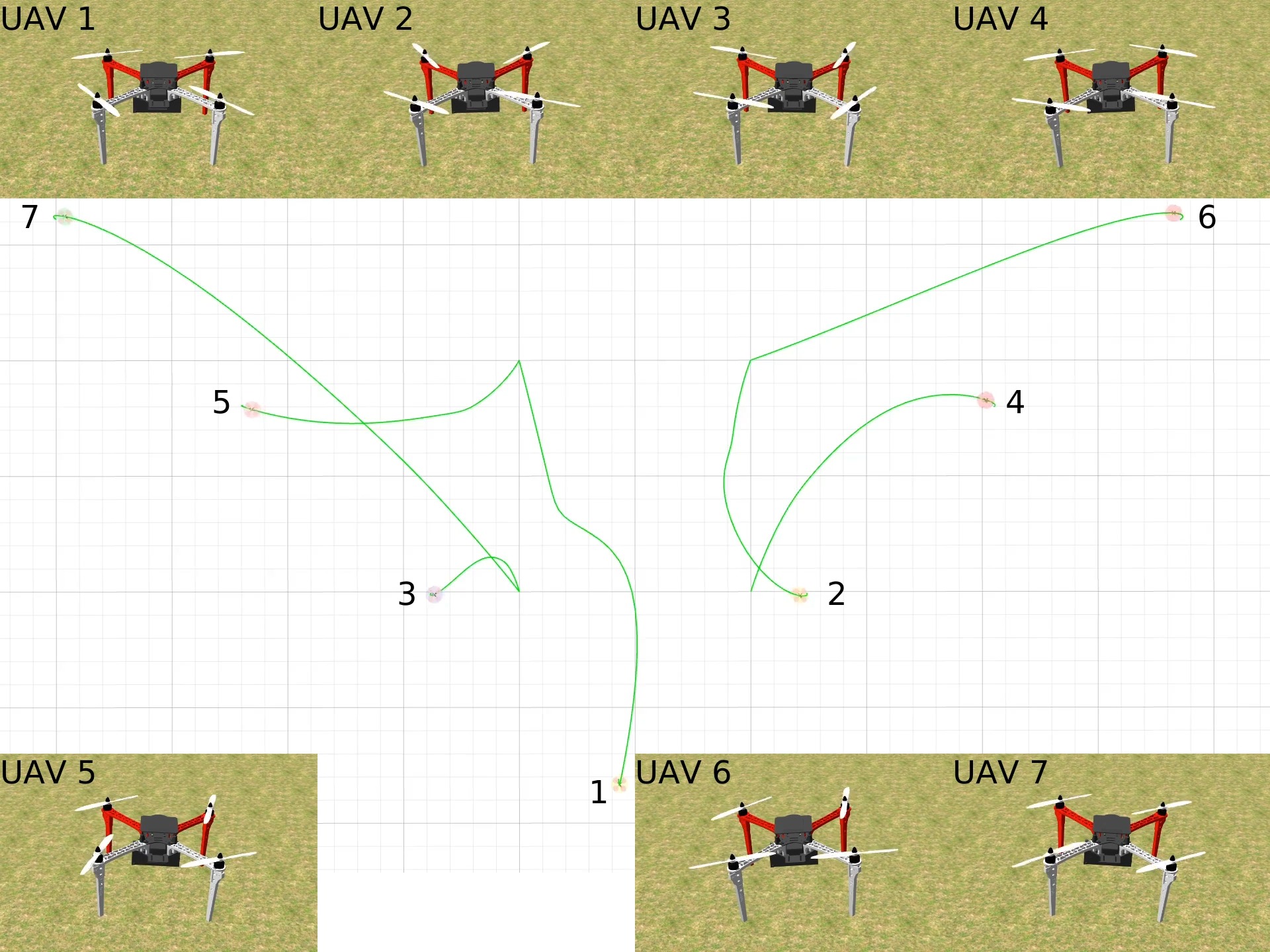}
		\caption{\(t=8\,\mathrm{s}\)}
	\end{subfigure}
     \hspace{0.025\linewidth}
    \begin{subfigure}{0.3\linewidth}
		\includegraphics[width=\linewidth]{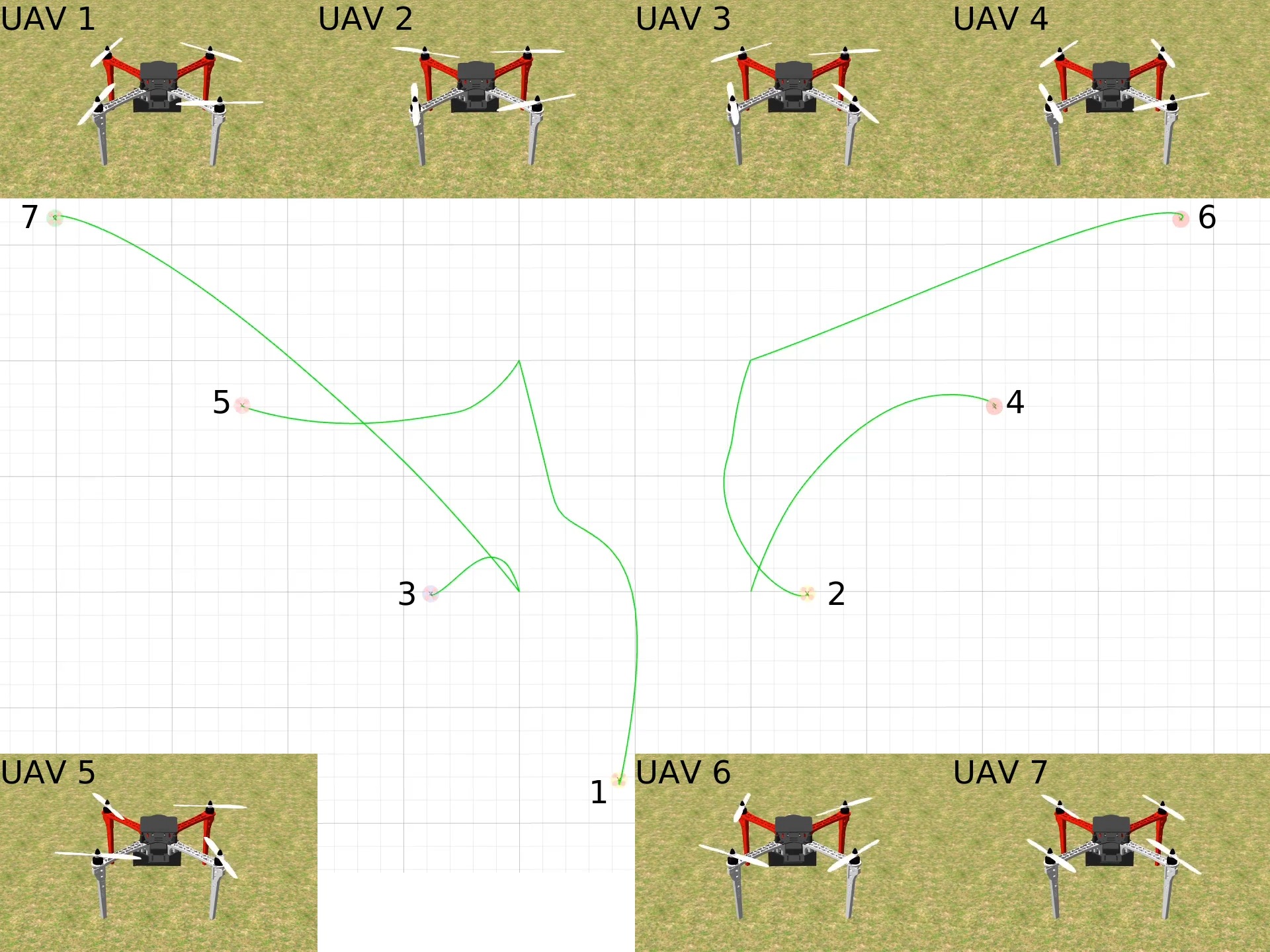}
		\caption{\(t=10\,\mathrm{s}\)}
	\end{subfigure}
\caption{Snapshots of a seven-UAV team tracking the formation trajectories using the directionally aware collision avoidance strategy with improved feasibility (green) in the high fidelity Gazebo simulator using MRS UAV System.}
\label{fig:aware7_slow}
\end{figure}

\begin{figure}[htbp]
	\centering
	\begin{subfigure}{0.24\linewidth}
		\includegraphics[width=\linewidth]{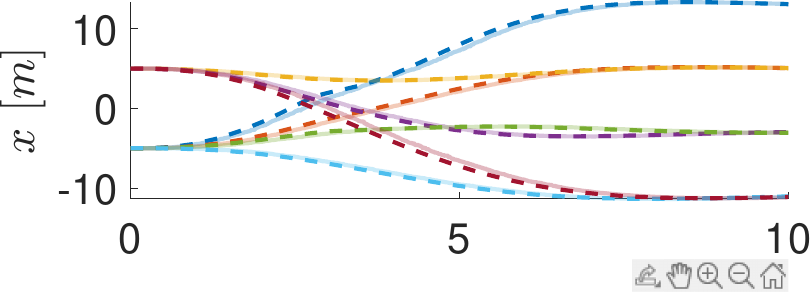}
        \includegraphics[width=\linewidth]{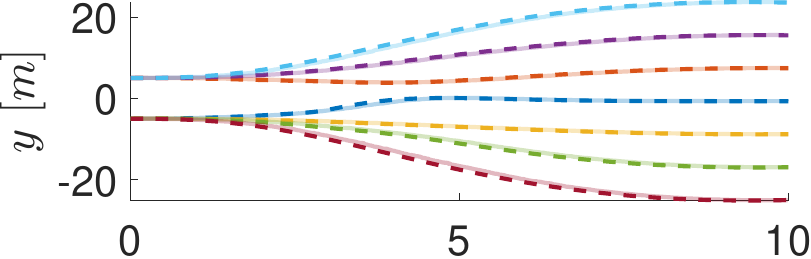}
        \includegraphics[width=\linewidth]{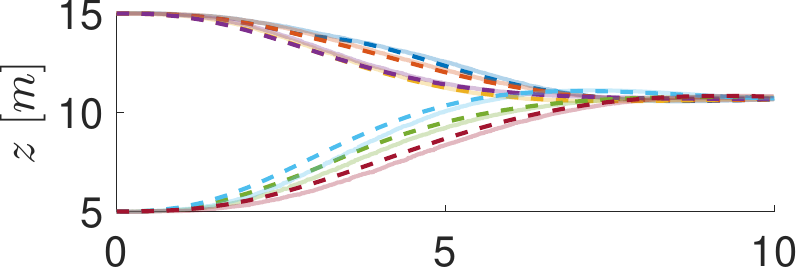}
        \includegraphics[width=\linewidth]{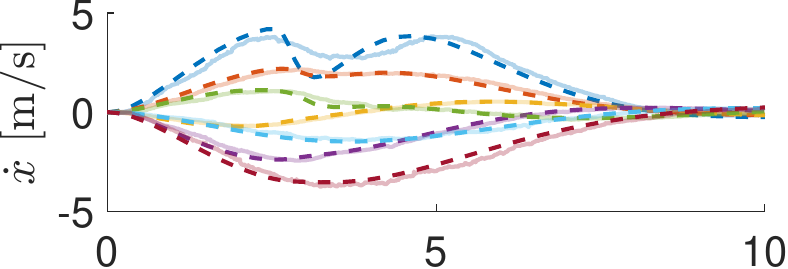}
        \includegraphics[width=\linewidth]{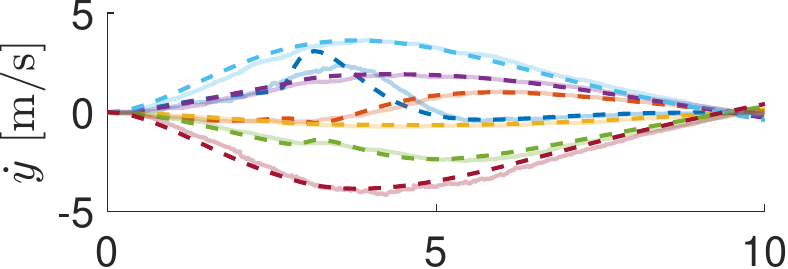}
        \includegraphics[width=\linewidth]{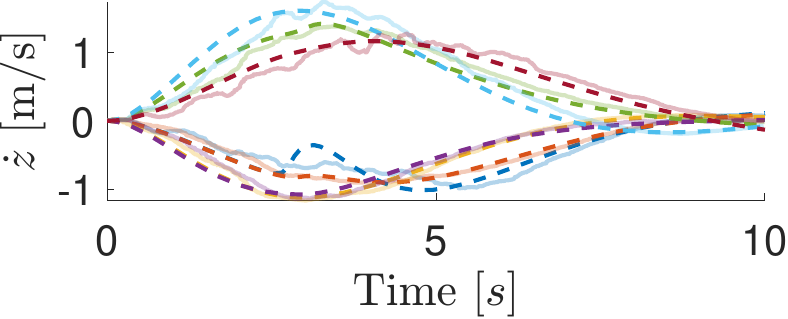}
		\caption{Positions and velocities}
	\end{subfigure}
    \begin{subfigure}{0.24\linewidth}
		\includegraphics[width=\linewidth]{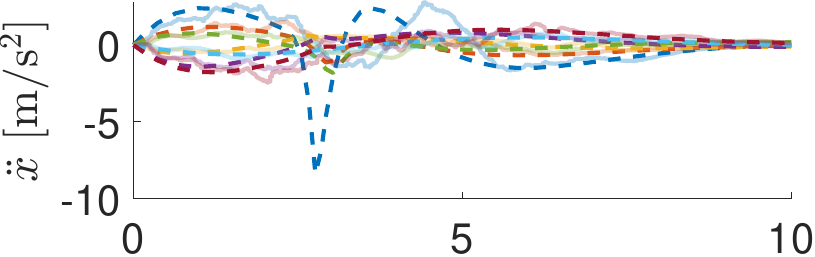}
        \includegraphics[width=\linewidth]{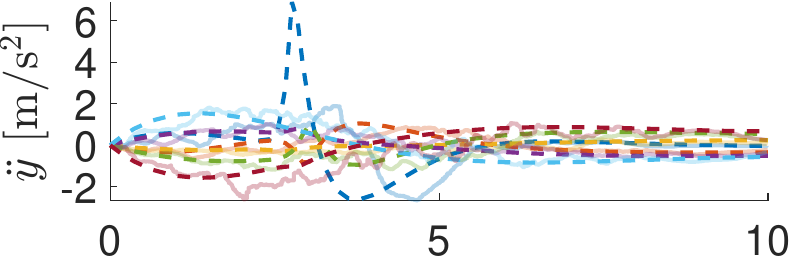}
        \includegraphics[width=\linewidth]{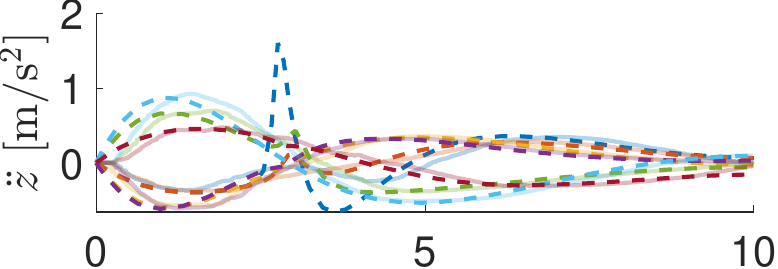}
        \includegraphics[width=\linewidth]{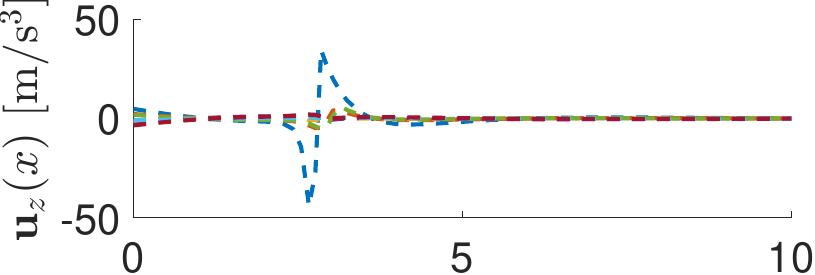}
        \includegraphics[width=\linewidth]{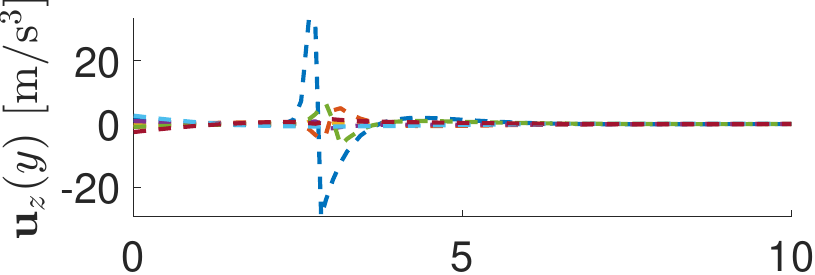}
        \includegraphics[width=\linewidth]{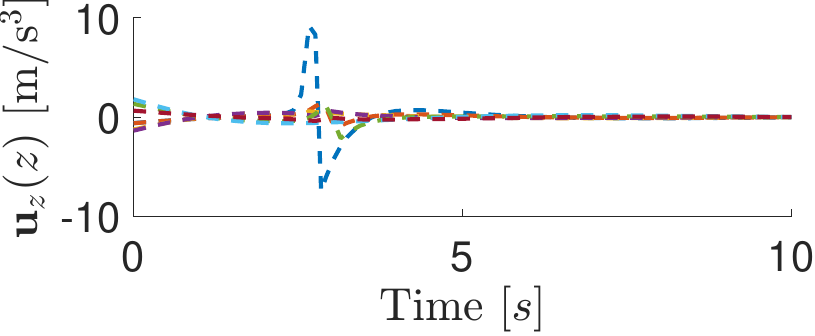}
		\caption{Accelerations and \(\mathbf{u}_{\mathbf{z}}\)}
	\end{subfigure}
    \begin{subfigure}{0.24\linewidth}
		\includegraphics[width=\linewidth]{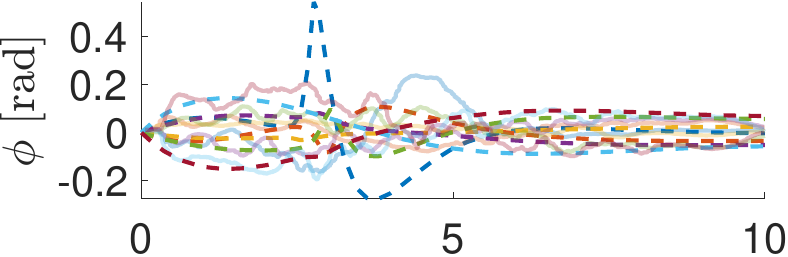}
        \includegraphics[width=\linewidth]{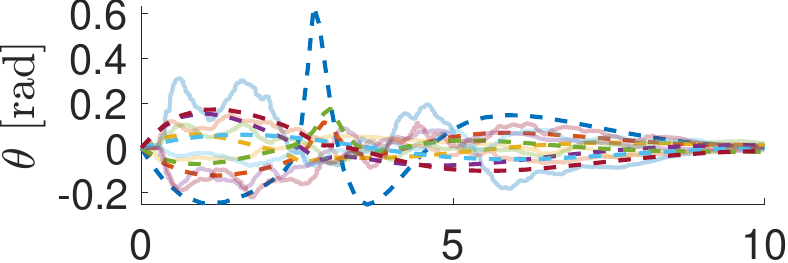}
        \includegraphics[width=\linewidth]{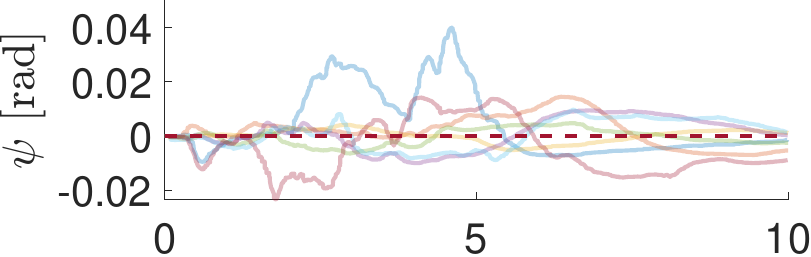}
        \includegraphics[width=\linewidth]{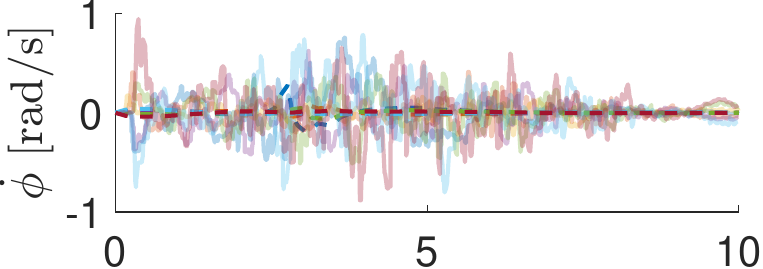}
        \includegraphics[width=\linewidth]{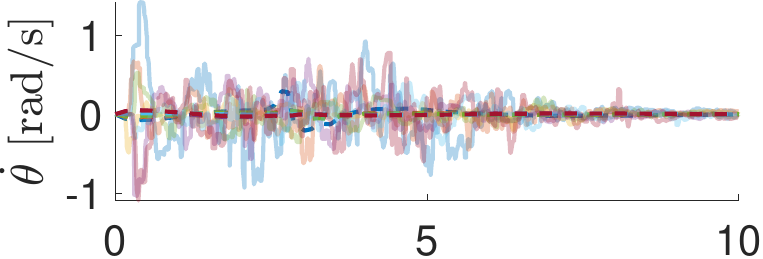}
        \includegraphics[width=\linewidth]{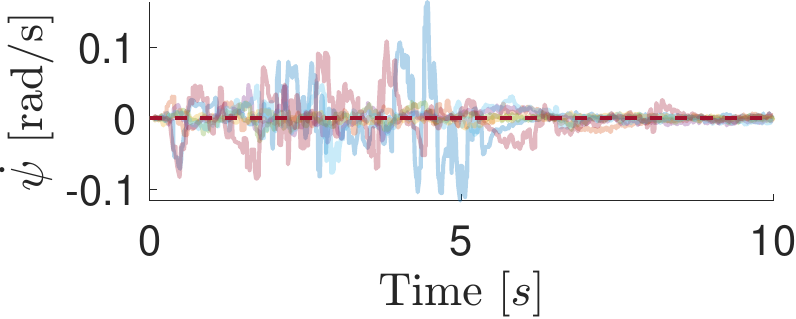}
		\caption{Attitudes with rates}
	\end{subfigure}
    \begin{subfigure}{0.45\linewidth}
	          \includegraphics[width=\linewidth]{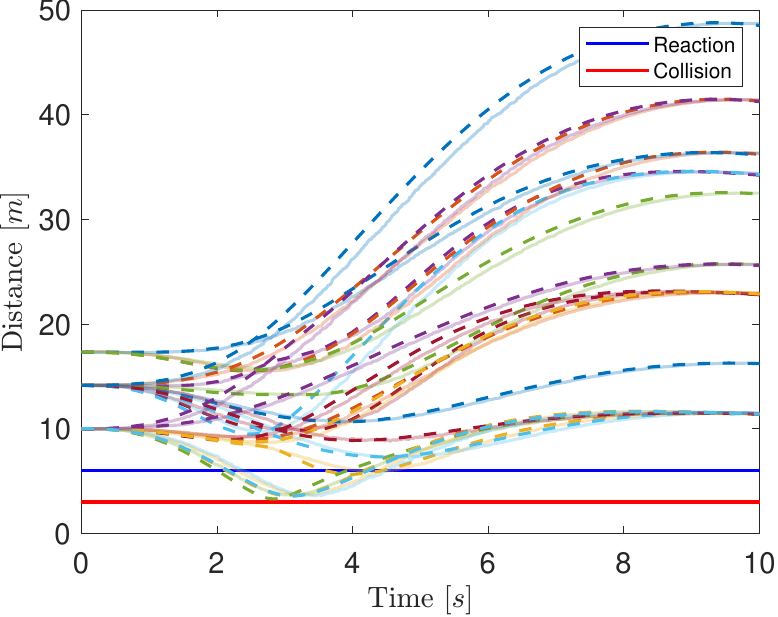}
		\caption{Euclidean distances}
    \end{subfigure}
    \caption{Time histories of positions, velocities, accelerations, jerks, attitudes with their rates, and Euclidean distances for the seven-UAV team formation trajectories utilizing the directionally aware collision avoidance strategy with improved feasibility. The faded lines correspond to the simulation results from the ROS-Gazebo-based simulator, each matching the color of its reference dashed line. }
\label{fig:aware-history7_slow}
\end{figure}

\section{Conclusion}\label{sec:con}
This paper presents a collision-free, finite-time formation control scheme for multiple UAVs based on the differential flatness property of UAV dynamics. Unlike traditional methods that rely on a leader, predefined tracking signals, or waypoints, our approach utilizes the initial states of the UAVs to plan optimal formation trajectories based on consensus. Additionally, we introduce a directionally aware collision avoidance strategy that prioritizes avoiding collisions both in the forward path and the relative approach for each UAV, promoting adaptive behavior and enhancing efficiency. By formulating and solving a finite-time optimal control problem, we derive trajectories that optimize a collective performance index, ensuring efficient formation and movement of the UAV team. To maintain collision-free tracking of the formation trajectory, we incorporate safety guarantees into the tracking controller when necessary, striking a balance between optimality and safety. High-fidelity simulations are performed using the ROS-Gazebo-based simulator. The results confirm the effectiveness of our control scheme, demonstrating successful formation trajectory planning and formation trajectory tracking while avoiding collisions. However, the assumption of a complete communication graph for global collision awareness may pose scalability challenges for larger UAV teams. Future work could address this limitation by developing distributed control strategies that enable real-time, directionally aware, collision-free trajectory planning and tracking, using onboard sensors such as UVDAR~\cite{8651535}. 

\section*{Acknowledgments} This work was partially funded by the Czech Science Foundation (GAČR) under research project no. \(\mathrm{23-07517S}\) and the European Union under the project Robotics and Advanced Industrial Production (reg. no. \(\mathrm{CZ.02.01.01/00/22\_008/0004590}\)), Czech Technical University in Prague (CTU) within the CTU Global Postdoc Fellowship program, and CTU grant no. \(\mathrm{SGS23/177/OHK3/3T/13}\).

 \appendix
 \section{proof of theorem~\ref{theorem:sol-open}} \label{app:proof}
\begin{proof}
Define the Hamiltonian 
\begin{align}
    H=\mathbf{r}^\top(t) \mathbf{Q} \mathbf{r}(t)&+\mathbf{u}_{\mathbf{r}}^\top(t) \mathbf{R} \mathbf{u}_{\mathbf{r}}(t)+\boldsymbol{\psi}^\top(t)(\mathbf{A}\mathbf{r}(t)+\mathbf{B}\mathbf{u}_{\mathbf{r}}(t)),
\end{align}
where $\boldsymbol{\psi}(t)$ is co-state. According to Pontryagin’s principle, the necessary conditions for optimality are $\frac{\partial H}{\partial \mathbf{u}_{\mathbf{r}}}=0$ and $\dot{\boldsymbol{\psi}}(t)=-\frac{\partial H}{\partial \mathbf{r}}$. Applying the necessary conditions yields
\begin{align}
     &\mathbf{u}_{\mathbf{r}}(t)=-\mathbf{R}^{-1}\mathbf{B}^\top\boldsymbol{\psi}(t), \label{eq:xi} \\
    &\dot{\boldsymbol{\psi}}(t)=-\mathbf{Q}\mathbf{r}(t)-\mathbf{A}^\top\boldsymbol{\psi}(t) ,\quad\boldsymbol{\psi}(t_f)=\mathbf{Q}_f \mathbf{r}(t_f). \label{eq:psi}
\end{align}

Note that (\ref{eq:PI}) is a strictly convex function of $\mathbf{u}_{\mathbf{r}}(t)$ for all admissible control functions, and the dynamics are linear. Therefore, the conditions obtained from Pontryagin’s principle are both necessary and sufficient~\cite{ENGWERDA1998729,ledzewicz2022pitfalls}. 

Substituting (\ref{eq:xi}) into the state dynamics (\ref{eq:dynamics}), we have
\begin{equation} \label{eq:dynamics-flat-aug2}
    \dot{\mathbf{r}}(t)=\mathbf{A}\mathbf{r}(t)-\mathbf{B}\mathbf{R}^{-1}\mathbf{B}^\top\boldsymbol{\psi}(t),\quad \mathbf{r}(0)=\mathbf{r}_0.
\end{equation}
Equations (\ref{eq:dynamics-flat-aug2}) and (\ref{eq:psi}) are unified into the following differential equation
\begin{equation}\label{eq:unified}
    \begin{bmatrix}\dot{\mathbf{r}}(t) \\ \dot{\boldsymbol{\psi}}(t)\end{bmatrix}=\mathbf{M}\begin{bmatrix}\mathbf{r}(t) \\ \boldsymbol{\psi}(t)\end{bmatrix}, \quad 
\end{equation}
where $\mathbf{M}$ is given by (\ref{eq:M}).
Moreover, the initial and terminal conditions are unified as
\begin{equation}\label{eq:condition}
      \begin{bmatrix}
     \mathbf{I} & \mathbf{0} \\ \mathbf{0} & \mathbf{0}
 \end{bmatrix}\begin{bmatrix}\mathbf{r}(0) \\ \boldsymbol{\psi}(0)\end{bmatrix}+\begin{bmatrix}
     \mathbf{0} & \mathbf{0} \\ -\mathbf{Q}_f & \mathbf{I}
 \end{bmatrix}\begin{bmatrix}\mathbf{r}(t_f) \\ \boldsymbol{\psi}(t_f)\end{bmatrix}=\begin{bmatrix}\mathbf{r}_0 \\ \mathbf{0}\end{bmatrix}.
\end{equation}

The solution of (\ref{eq:unified}) at $t_f$ is given by
\begin{align}\label{eq:soltf}
      \begin{bmatrix}\mathbf{r}(t_f) \\ \boldsymbol{\psi}(t_f)\end{bmatrix}&=\exp{(t_f\mathbf{M})}\begin{bmatrix}\mathbf{r}(0) \\ \boldsymbol{\psi}(0)\end{bmatrix}.
\end{align}

Substituting (\ref{eq:soltf}) into (\ref{eq:condition}), we obtain
\begin{align}\label{eq:condition2}
  \Bigg(\begin{bmatrix}
     \mathbf{I} & \mathbf{0} \\ \mathbf{0} & \mathbf{0}
 \end{bmatrix}+\begin{bmatrix}
     \mathbf{0} & \mathbf{0} \\ -\mathbf{Q}_f & \mathbf{I}
 \end{bmatrix}\exp{(t_f\mathbf{M})}\Bigg)\begin{bmatrix}\mathbf{r}(0) \\ \boldsymbol{\psi}(0)\end{bmatrix}=\begin{bmatrix}\mathbf{r}_0 \\ \mathbf{0}\end{bmatrix},
\end{align}
or equivalently, 
\begin{align}\label{eq:condition3}
  &\Bigg(\begin{bmatrix}
     \mathbf{I} & \mathbf{0} \\ \mathbf{0} & \mathbf{0}
 \end{bmatrix}\exp{(-t_f\mathbf{M})}+\begin{bmatrix}
     \mathbf{0} & \mathbf{0} \\ -\mathbf{Q}_f & \mathbf{I}
 \end{bmatrix}\Bigg)\exp{(t_f\mathbf{M})}\begin{bmatrix}\mathbf{r}(0) \\ \boldsymbol{\psi}(0)\end{bmatrix}=\begin{bmatrix}\mathbf{r}_0 \\ \mathbf{0}\end{bmatrix}.
\end{align} 

Let 
\begin{equation} \label{eq:expM}
    \exp{(-t_f\mathbf{M})}=\begin{bmatrix}
     \mathbf{C}_{11} & \mathbf{C}_{12} \\ \mathbf{C}_{21} & \mathbf{C}_{22}
 \end{bmatrix}.
\end{equation}
Multiplying both sides of (\ref{eq:condition3}) by $\begin{bmatrix}
        \mathbf{I} & -\mathbf{C}_{12} \\ \mathbf{0} & \mathbf{I} 
    \end{bmatrix}$, we get
\begin{align}\label{eq:unified5}
  &\begin{bmatrix}
     \mathbf{C}_{11}+\mathbf{C}_{12}\mathbf{Q}_f & \mathbf{0} \\ -\mathbf{Q}_f & \mathbf{I}
 \end{bmatrix}\exp{(t_f\mathbf{M})}\begin{bmatrix}\mathbf{r}(0) \\ \boldsymbol{\psi}(0)\end{bmatrix}=\begin{bmatrix}\mathbf{r}_0 \\ \mathbf{0}\end{bmatrix},
\end{align}
or equivalently, 
\begin{align}\label{eq:unifieds}
  &\begin{bmatrix}
     \mathbf{C}_{11}+\mathbf{C}_{12}\mathbf{Q}_f & \mathbf{0} \\ -\mathbf{Q}_f & \mathbf{I}
 \end{bmatrix}\begin{bmatrix}\mathbf{r}(t_f) \\ \boldsymbol{\psi}(t_f)\end{bmatrix}=\begin{bmatrix}\mathbf{r}_0 \\ \mathbf{0}\end{bmatrix}.
\end{align}

Substituting (\ref{eq:expM}) into (\ref{eq:Ht}) yields
\begin{align}\label{eq:Ht1}
    \mathbf{H}(t_f)=\mathbf{C}_{11}+\mathbf{C}_{12}\mathbf{Q}_f.
\end{align}
Thus, (\ref{eq:unified5}) is rewritten as 
\begin{align}\label{eq:unified6}
  &\begin{bmatrix}
     \mathbf{H}(t_f) & \mathbf{0} \\ -\mathbf{Q}_f & \mathbf{I}
 \end{bmatrix}\begin{bmatrix}\mathbf{r}(t_f) \\ \boldsymbol{\psi}(t_f)\end{bmatrix}=\begin{bmatrix}\mathbf{r}_0 \\ \mathbf{0}\end{bmatrix}.
\end{align}
From (\ref{eq:unified6}), we have
\begin{align}\label{eq:unifiedd}
  \begin{bmatrix}\mathbf{r}(t_f) \\ \boldsymbol{\psi}(t_f)\end{bmatrix}&=\begin{bmatrix}
     \mathbf{H}(t_f) & \mathbf{0} \\ -\mathbf{Q}_f & \mathbf{I}
 \end{bmatrix}^{-1}\begin{bmatrix}\mathbf{r}_0 \\ \mathbf{0}\end{bmatrix}=\begin{bmatrix}
     \mathbf{H}^{-1}(t_f) & \mathbf{0} \\ \mathbf{Q}_f\mathbf{H}^{-1}(t_f) & \mathbf{I}
 \end{bmatrix}\begin{bmatrix}\mathbf{r}_0 \\ \mathbf{0}\end{bmatrix}\nonumber\\&=\begin{bmatrix}
     \mathbf{H}^{-1}(t_f)\mathbf{r}_0  \\ \mathbf{Q}_f\mathbf{H}^{-1}(t_f)\mathbf{r}_0 
 \end{bmatrix}=\begin{bmatrix}
     \mathbf{I}  \\ \mathbf{Q}_f
 \end{bmatrix}\mathbf{H}^{-1}(t_f)\mathbf{r}_0,
\end{align}
Substituting (\ref{eq:unifiedd}) into (\ref{eq:soltf}) and rearranging it, we get 
\begin{align}\label{eq:unified7}
  &\begin{bmatrix}\mathbf{r}(0) \\ \boldsymbol{\psi}(0)\end{bmatrix}=\exp{(-t_f\mathbf{M})}\begin{bmatrix}
     \mathbf{I}  \\ \mathbf{Q}_f
 \end{bmatrix}\mathbf{H}^{-1}(t_f)\mathbf{r}_0.
\end{align}

Substituting (\ref{eq:unified7}) into the solution of (\ref{eq:unified}) at $t$ yields
\begin{align}\label{eq:unified8}
      \begin{bmatrix}\mathbf{r}(t) \\ \boldsymbol{\psi}(t)\end{bmatrix}&=\exp{(t\mathbf{M})}\begin{bmatrix}\mathbf{r}(0) \\ \boldsymbol{\psi}(0)\end{bmatrix}=\exp{((t-t_f)\mathbf{M})}\begin{bmatrix}
     \mathbf{I}  \\ \mathbf{Q}_f
 \end{bmatrix}\mathbf{H}^{-1}(t_f)\mathbf{r}_0.
\end{align}
From (\ref{eq:unified8}), we have
\begin{align}
     & \mathbf{r}(t) =\begin{bmatrix}
     \mathbf{I}  & \mathbf{0}
 \end{bmatrix}\exp{((t-t_f)\mathbf{M})}\begin{bmatrix}
     \mathbf{I}  \\ \mathbf{Q}_f
 \end{bmatrix}\mathbf{H}^{-1}(t_f)\mathbf{r}_0, \\
 & \boldsymbol{\psi}(t) =\begin{bmatrix}
     \mathbf{0}  & \mathbf{I}
 \end{bmatrix}\exp{((t-t_f)\mathbf{M})}\begin{bmatrix}
     \mathbf{I}  \\ \mathbf{Q}_f
 \end{bmatrix}\mathbf{H}^{-1}(t_f)\mathbf{r}_0.
\end{align}
Using the notation in (\ref{eq:Ht}) and (\ref{eq:Gt}), we get (\ref{eq:open-con}) and (\ref{eq:open-traj}). This concludes the proof. 
\end{proof}

\section{stability analysis of forward-path-aware collision avoidance control} \label{app:dir}

This appendix analyzes the stability of the collision avoidance control strategy under the forward-path-aware penalty function \( v_{ij}^F \), as introduced in Subsection~\ref{sec:forward}. Specifically, it demonstrates that the control inputs remain bounded without relying on Assumption~\ref{assumption:constancy}, ensuring the stability of the directionally aware collision avoidance framework. The analysis focuses on the gradient of the penalty function and its impact on the control law in~\eqref{eq:feedback}.

Recall the collision avoidance control strategy~\eqref{eq:feedback} under the forward-path-aware penalty function as follows,
\begin{align} \label{eq:con_alpha}
   \mathbf{u}_{\mathbf{z}}(t) =-\mathbf{R}_{\mathbf{z}}^{-1}\mathbf{B}^\top \mathbf{P}(t) \mathbf{z}(t)-\mathbf{R}_{\mathbf{z}}^{-1}\mathbf{B}^\top \frac{\partial V^F}{\partial \mathbf{z}}^\top,
\end{align}
where \( V_i^F = \sum_{j \in \mathcal{N}_i} v_{ij}^F = \sum_{j \in \mathcal{N}_i} \alpha_{ij} v_{ij} \) for UAV \( i \) and the total penalty is \( V^F = \sum_i V_i^F \).

Then, we have
\begin{equation}\label{eq:grad}
    \frac{\partial V_i^F}{\partial \mathbf{z}_i}=\sum_{j\in\mathcal{N}_i}\big(\alpha_{ij}\underbrace{\frac{\partial v_{ij}}{\partial \mathbf{z}_i}}_{\text{original gradient}}+v_{ij}\underbrace{\frac{\partial \alpha_{ij}}{\partial \mathbf{z}_i}}_{\text{directional adjustment}}\big).
\end{equation}
The gradient decomposition in~\eqref{eq:grad} shows how directional awareness modifies the collision response. The \textit{original gradient} term maintains baseline collision avoidance, while the \textit{directional adjustment} term modulates responses based on the forward-path awareness. 

To ensure stability and avoid unbounded control inputs, we must guarantee that the gradient of the total penalty \(\frac{\partial V^F}{\partial \mathbf{z}}\) does not grow indefinitely.

The total penalty $V$ is guaranteed to be non-increasing because of bounded components $\alpha_{ij}$ and $v_{ij}$ (where $\alpha_{ij} \in [0,1]$ and $v_{ij} \to 0$ as $\|\mathbf{p}_j-\mathbf{p}_i\| \to \infty$) and vanishing gradients
    \begin{equation*}
        \left\|\frac{\partial V_i^F}{\partial \mathbf{z}_i}\right\| \leq \sum_{j \in \mathcal{N}_i} \left(\alpha_{ij}\left\|\frac{\partial v_{ij}}{\partial \mathbf{z}_i}\right\| + v_{ij}\left\|\frac{\partial \alpha_{ij}}{\partial \mathbf{z}_i}\right\|\right),
    \end{equation*}
    where \(\frac{\partial v_{ij}}{\partial \mathbf{z}_i}\) is proven to be bounded in Lemma~\ref{lma:ICs}, and for \(\frac{\partial \alpha_{ij}}{\partial \mathbf{z}_i}\), its bounds are derived as follows.

For $\alpha_{ij} = \frac{(\mathbf{p}_j - \mathbf{p}_i)^\top \mathbf{v}_i}{\|\mathbf{p}_j - \mathbf{p}_i\| \|\mathbf{v}_i\|}$ (when $\|\mathbf{v}_i\| \neq 0$ and $\cos\eta_{ij} > 0$), the partial position and velocity derivatives are
\begin{align*}
    &\frac{\partial \alpha_{ij}}{\partial \mathbf{p}_i} = -\frac{\mathbf{v}_i}{\|\mathbf{p}_j - \mathbf{p}_i\| \|\mathbf{v}_i\|} + \frac{(\mathbf{p}_j - \mathbf{p}_i)^\top \mathbf{v}_i}{\|\mathbf{p}_j - \mathbf{p}_i\|^3 \|\mathbf{v}_i\|}(\mathbf{p}_j - \mathbf{p}_i),\\
    &\frac{\partial \alpha_{ij}}{\partial \mathbf{v}_i} = \frac{\mathbf{p}_j - \mathbf{p}_i}{\|\mathbf{p}_j - \mathbf{p}_i\| \|\mathbf{v}_i\|} - \frac{(\mathbf{p}_j - \mathbf{p}_i)^\top \mathbf{v}_i}{\|\mathbf{p}_j - \mathbf{p}_i\| \|\mathbf{v}_i\|^3}\mathbf{v}_i.
\end{align*}
Taking their Euclidean norm, we have
\begin{align*}
\left\|\frac{\partial \alpha_{ij}}{\partial \mathbf{p}_i}\right\| &\leq \left\|-\frac{\mathbf{v}_i}{\|\mathbf{p}_j - \mathbf{p}_i\| \|\mathbf{v}_i\|}\right\| + \left\|\frac{(\mathbf{p}_j - \mathbf{p}_i)^\top \mathbf{v}_i}{\|\mathbf{p}_j - \mathbf{p}_i\|^3 \|\mathbf{v}_i\|}(\mathbf{p}_j - \mathbf{p}_i)\right\| \\
&= \frac{1}{\|\mathbf{p}_j - \mathbf{p}_i\|} + \frac{|\cos\eta_{ij}|}{\|\mathbf{p}_j - \mathbf{p}_i\|} \leq \frac{2}{\|\mathbf{p}_j - \mathbf{p}_i\|},
\end{align*}
and
\begin{align*}
\left\|\frac{\partial \alpha_{ij}}{\partial \mathbf{v}_i}\right\| &\leq \left\|\frac{\mathbf{p}_j - \mathbf{p}_i}{\|\mathbf{p}_j - \mathbf{p}_i\| \|\mathbf{v}_i\|}\right\| + \left\|\frac{(\mathbf{p}_j - \mathbf{p}_i)^\top \mathbf{v}_i}{\|\mathbf{p}_j - \mathbf{p}_i\| \|\mathbf{v}_i\|^3}\mathbf{v}_i\right\| \\
&= \frac{1}{\|\mathbf{v}_i\|} + \frac{|\cos\eta_{ij}|}{\|\mathbf{v}_i\|} \leq \frac{2}{\|\mathbf{v}_i\|}.
\end{align*}
Thus, the total derivative norm is bounded by
\begin{equation*}
\left\|\frac{\partial \alpha_{ij}}{\partial \mathbf{z}_i}\right\| \leq \left\|\frac{\partial \alpha_{ij}}{\partial \mathbf{p}_i}\right\| + \left\|\frac{\partial \alpha_{ij}}{\partial \mathbf{v}_i}\right\| \leq \frac{2}{\|\mathbf{p}_j - \mathbf{p}_i\|} + \frac{2}{\|\mathbf{v}_i\|}.
\end{equation*}

Since \( v_{ij} \to 0 \) as \(\|\mathbf{p}_j - \mathbf{p}_i\| \to \infty\), the term \( v_{ij} \frac{\partial \alpha_{ij}}{\partial \mathbf{z}_i} \) in~\eqref{eq:grad} diminishes at large separations, further ensuring that the directional adjustment term does not contribute to unbounded growth in the total gradient.

The boundedness of \(\frac{\partial \alpha_{ij}}{\partial \mathbf{z}_i}\), combined with the boundedness of \(\frac{\partial v_{ij}}{\partial \mathbf{z}_i}\) (as proven in Lemma~\ref{lma:ICs}) and the decay of \( v_{ij} \) as \(\|\mathbf{p}_j - \mathbf{p}_i\| \to \infty\), ensures that the total gradient \(\frac{\partial V_i^F}{\partial \mathbf{z}_i}\) in~\eqref{eq:grad} remains bounded. Consequently, the control input \(\mathbf{u}_{\mathbf{z}}(t)\) in~\eqref{eq:con_alpha} avoids unbounded growth, preserving the stability of the HJB equation and the collision avoidance strategy without Assumption~\ref{assumption:constancy}.

Similar derivations can be applied to the approach-aware penalty function \( v_{ij}^A = \beta_{ij} v_{ij} \) and the unified penalty function \( v_{ij}^U = \xi_{ij} v_{ij} \), as their weighting factors (\(\beta_{ij}\), \(\xi_{ij}\)) are also bounded in \([0, 1]\) and their gradients can be shown to be bounded using analogous techniques, ensuring stability across the directionally aware framework.

\bibliographystyle{IEEEtran}
\bibliography{mybibliography}


\end{document}